%% file: main.tex
\documentclass{article}

\input{icml_header.tex}

\begin{document}

\input{icml_begin_doc.tex}

\begin{abstract}
In this paper, we introduce a conformal prediction method to construct prediction sets in a one-shot federated learning setting. More specifically, we define a quantile-of-quantiles estimator and prove that for any distribution, it is possible to output prediction sets with desired coverage in only one round of communication. To mitigate privacy issues, we also describe a locally differentially private version of our estimator. Finally, over a wide range of experiments, we show that our method returns prediction sets with coverage and length very similar to those obtained in a centralized setting. Overall, these results demonstrate that our method is particularly well-suited to perform conformal predictions in a one-shot federated learning setting.
\end{abstract}

\section{Introduction}

\input{intro.tex}

\section{Background and Related Work}
\label{sec:background}

\input{background.tex}

\section{Quantile-of-Quantiles for Federated CP}

\input{CP-QQ.tex}

\section{Differentially Private \method}
\label{sec:privacy}

\input{privacy.tex}

\section{Experiments}
\label{sec:xps}

\input{xps.tex}

\section{Discussion}

\input{discussion.tex}

\section*{Acknowledgements}

This work was supported in part by the French Agence Nationale de la Recherche under grants ANR-20-CE23-0015 (Project PRIDE) and ANR-17-CE23-0011 ({\sc Fast-Big}). Batiste Le Bars is supported by an Inria-EPFL fellowship. 
Sylvain Arlot is also supported by Institut Universitaire de France (IUF). 

\vskip 0.2in
\bibliography{biblio.bib}
\bibliographystyle{icml2023}

\appendix
\onecolumn
\begin{center}
	{\Large \textbf{Appendix}}
\end{center}

\input{appendix}


\end{document}

%% file: icml_header.tex
\usepackage{microtype}
\usepackage{graphicx}
\usepackage{subfigure}
\usepackage{booktabs} 

\usepackage{hyperref}

\usepackage{algorithm, algpseudocode}

\usepackage[accepted]{icml2023}

\usepackage{amsmath}
\usepackage{amssymb}
\usepackage{mathtools}
\usepackage{amsthm}
\usepackage{stmaryrd}

\theoremstyle{plain}
\newtheorem{theorem}{Theorem}[section]
\newtheorem{proposition}[theorem]{Proposition}
\newtheorem{lemma}[theorem]{Lemma}

\newtheorem{definition}[theorem]{Definition}

\newtheorem{remark}[theorem]{Remark}

\usepackage[textsize=tiny]{todonotes}

\icmltitlerunning{One-Shot Federated Conformal Prediction}

\usepackage[capitalize,noabbrev]{cleveref}
\usepackage{dsfont}
\usepackage{enumitem}

\usepackage{comment}
\allowdisplaybreaks 

\newcommand{\calD}{{\cal D}}
\newcommand\calX{{\cal X}}
\newcommand\calY{{\cal Y}}

\newcommand\calA{{\cal A}}
\newcommand\calS{{\cal S}}
\newcommand\calO{{\cal O}}

\newcommand\fh{{\widehat{f}}}
\newcommand\IE{{\mathbb{E}}}
\newcommand\IR{{\mathbb{R}}}
\newcommand\IN{{\mathbb{N}}}
\newcommand\IP{{\mathbb{P}}}
\newcommand\Chat{{\widehat{\mathcal{C}}}}
\newcommand\Qh{{\widehat{Q}}}

\newcommand{\Ind}[1]{\mathds{1}\{#1\}}

\newcommand{\One}[1]{{\mathds{1}}\left\{{#1}\right\}}
\newcommand{\OneSm}[1]{{\mathds{1}}\{{#1}\}}
\newcommand{\Bin}{\mathrm{Bin}}
\newcommand{\PoisBin}{\mathrm{PoisBin}}
\newcommand{\dTV}{d_{\mathrm{TV}}}

\usepackage{color}
\usepackage{xcolor} 

\newcommand{\method}{\text{FedCP-QQ}}
\newcommand{\methodDP}{\text{FedCP$^2$-QQ}}
\newcommand{\methodAvg}{\text{FedCP-Avg}}

\renewcommand{\geq}{\geqslant}
\renewcommand{\leq}{\leqslant}

%% file: icml_begin_doc.tex
\twocolumn[
\icmltitle{One-Shot Federated Conformal Prediction}



\icmlsetsymbol{equal}{*}

\begin{icmlauthorlist}
	\icmlauthor{Pierre Humbert}{yyy}
	\icmlauthor{Batiste Le Bars}{zzz}
	\icmlauthor{Aurélien Bellet}{zzz}
	\icmlauthor{Sylvain Arlot}{yyy}
\end{icmlauthorlist}

\icmlaffiliation{yyy}{Université Paris-Saclay, CNRS, Inria, Laboratoire de mathématiques d’Orsay, 91405, Orsay, France.}
\icmlaffiliation{zzz}{Université Lille, Inria, CNRS, Centrale Lille, UMR 9189, CRIStAL, F-59000 Lille}

\icmlcorrespondingauthor{Pierre Humbert}{pierre.humbert@universite-paris-saclay.fr}

\icmlkeywords{Machine Learning, ICML}

\vskip 0.3in
]



\printAffiliationsAndNotice{}  

%% file: intro.tex

Federated Learning (FL) is a recent paradigm that allows to learn from decentralized data sets stored locally by multiple agents \citep{kairouz2021advances}. FL is particularly appealing when data are highly sensitive and cannot be centralized for privacy or security reasons. So far, the design of FL algorithms has mainly focused on the training phase of machine learning: the goal is to fit models on decentralized data sets while minimizing the amount of communication or optimizing the privacy-utility trade-off \citep[see e.g.][]{mcmahan2017communication, dp_fed_sgd_user_level, li2020federated,scaffold,Noble2022a}. However, FL poses further challenges regarding model evaluation, as this step must also be done without access to centralized data. In particular, with the increasing popularity of black-box methods, deploying machine learning models in real-world applications often requires to appropriately \emph{quantify the uncertainty} of their predictions. Unfortunately, models trained with the above supervised FL algorithms only provide point predictions (e.g., class labels or regression targets). This is not sufficient in high-stakes applications like medicine \citep{begoli2019need}, where decisions may impact human lives.

In this work, we investigate the task of outputting a prediction set rather than a single point prediction in a FL setting. Formally, given some data stored by multiple agents and an additional test point $(X, Y)$, we want to construct a \emph{marginally valid} set which is likely to contain the unknown response $Y$. In other words, we want a set $\Chat(X)$ such that
\begin{equation}
	\label{eq:marg_cp}
	\IP\big(Y \in \Chat(X) \big) \geq 1-\alpha \; ,
\end{equation}
where $\alpha \in (0, 1)$ is a desired miscoverage rate. Although there exist several methods to construct such a set \citep{papadopoulos2002inductive, vovk2005algorithmic, romano2019conformalized}, they require access to a \emph{centralized} data set. They are thus incompatible with the constraints of FL, in which agents process their data locally and only interact with a central server by sharing some aggregate statistics. Constructing a valid prediction set is even more challenging in the \emph{one-shot FL} \citep{zhang2012communication,guha2019one,Bayesian-oneshot,practical-oneshot,clustering-oneshot,one-shot-jmlr} that we consider in this work, where the communication between the agents and the server is further restricted to a \emph{single round}. One-shot FL is motivated by the fact that the number of communication rounds is often the main bottleneck in FL \citep{kairouz2021advances}.

\paragraph{Contributions.}
In this paper, we present an intuitive one-shot FL method based on Conformal Prediction (CP) \citep{vovk2005algorithmic, papadopoulos2002inductive} to construct distribution-free prediction sets satisfying \eqref{eq:marg_cp}. The key step of CP methods is the ordering of \emph{scores} computed for each calibration data point. In the FL setting, this ordering step is not possible without exchanging the local data sets or performing many agent-server communication rounds. To circumvent this problem, we define a \emph{quantile-of-quantiles} estimator: each agent sends to the server a local empirical quantile and the server aggregates them by computing a quantile of these quantiles. We describe how to choose the order of the quantiles (depending on the number of agents and the size of their local data sets) to obtain a prediction set that satisfies \eqref{eq:marg_cp}. We also prove that property \eqref{eq:marg_cp} can be verified \emph{conditionally to the observed data} with a modification of the selected quantiles. While the previous results rely on certain data homogeneity assumptions, we further quantify the impact of heterogeneous (non-identically distributed) data on the performance of our algorithm. To address use cases with strong privacy constraints, we derive a version of our approach that satisfies differential privacy \cite{dwork2014algorithmic}, in which agents run the exponential mechanism to privately select their local quantile. Finally, we empirically evaluate the performance of our method on standard CP benchmarks and show that it produces prediction sets that are very close to the ones obtained when data are centralized.

%% file: background.tex

\subsection{Split Conformal Prediction}
\label{subsec:spiltCP}
Conformal Prediction (CP) is a framework to construct distribution-free prediction sets satisfying \eqref{eq:marg_cp} \citep{vovk2005algorithmic}. One of the most popular methods to perform CP in a centralized setting is the \emph{split conformal} \cite{papadopoulos2002inductive} which is at the core of our main contribution described in Section \ref{sec:main}.

To use the split conformal method (split CP), we first need to choose a score function $s:
\calX \times \calY \rightarrow \IR$, which measures the magnitude of a predictor error for a given point. Whether we are in the regression or classification setting, many different score functions exist in the literature \citep[see e.g.][]{angelopoulos2023conformal}. In regression, for instance, a common choice is the fitted absolute residual
where $\fh$ is some predictor learned on a training data set. Note that our approach does not assume a particular choice of score function, so throughout the paper, we keep the function $s$ abstract.
Then, we split the data $\calD_{n+n_{tr}} = \{(X_{1}, Y_{1}),\ldots,(X_{n+n_{tr}}, Y_{n+n_{tr}}) \}$  into
a \emph{calibration} set  $\calD_{n}^{cal}=\{(X_{1}, Y_{{1}}),\ldots,(X_{n}, Y_{n}) \}$ and a \emph{training} set $\calD_{n_{tr}}^{tr}=\{(X_{n+1}, Y_{n+1}),\ldots,(X_{n+n_{tr}}, Y_{n+n_{tr}}) \}$ with $n, n_{tr} \geq 1$. The predictor $\widehat{f}$ is fitted on $\calD_{n_{tr}}^{tr}$ and conformity scores $\calS^{cal}_{n} \triangleq\{S_{1},\ldots, S_{n}\}$ are calculated on $\calD_{n}^{cal}$ via the previously chosen score function $s$. 
Finally, given a test point $X$ and $\alpha \in (0, 1)$, we construct the conformal set
\begin{align*}
	\label{set_conf}
	\Chat(X) = \Big\{y \in \IR : s(X, y) \leq \widehat{Q}_{(\lceil (n+1)(1-\alpha)\rceil)}(\calS^{cal}_{n})\Big\} \; ,
\end{align*}
where $\widehat{Q}_{(\cdot)}(\cdot)$ is defined by
\begin{equation}
\label{def:Qk}
\widehat{Q}_{(k)}(\calS') = \widehat{Q}_{(k)} \triangleq \begin{cases}
S'_{(k)} & \text{if 
	$k \leq |\calS'| $} \\
\infty & \text{otherwise} \; ,
\end{cases}
\end{equation}
with $|\calS'|$ the size of the sample $\calS'$, and $S'_{(1)} \leq \ldots \leq S'_{(|\calS'|)}$ the order statistics of the scores $S'_
{1}, \ldots, S'_{|\calS'|}$ in $\calS'$. In other words, $\widehat{Q}_{(k)}$ outputs the $k$-th smallest value in a given set of scores. The following theorem proves that the set returned by the split CP method satisfies \eqref{eq:marg_cp} under mild assumptions.
\begin{theorem}[\citealp{vovk2005algorithmic, lei2018distribution}]
	\label{thm:cp_base}
For any $n, n_{tr} \geq 1$, 
let us consider $n+n_{tr}$ i.i.d. (or only exchangeable) random variables $(X_1, Y_1), \ldots, (X_{n+n_{tr}}, Y_{n+n_{tr}})$ from $\calX \times \calY$ and an additional test point $(X, Y)$. For any score function $s$ and any $\alpha \in (0, 1)$, the set returned by the split CP method satisfies
\begin{equation*}
    \IP\left(Y \in \Chat(X) \right) \geq 1-\alpha \; .
\end{equation*}
Furthermore, if $S_{1}, \ldots S_{n}$ are almost surely distinct, this probability is upper bounded by $1-\alpha + 1/(n + 1)$.
\end{theorem}

Although the first CP methods were the \emph{split} and the related \emph{full} methods \cite{papadopoulos2002inductive, vovk2005algorithmic}, many extensions based upon them have been proposed recently. In regression, \citet{lei2018distribution} present a method called locally weighted CP and provide theoretical insights for conformal inference. More recently, \citet{romano2019conformalized} have developed a variant of the split CP called Conformal Quantile Regression (CQR). Other recent alternatives have been proposed \citep{kivaranovic2020adaptive, sesia2021conformal, gupta2022nested, ndiaye2022stable}. We refer to \citet{vovk2005algorithmic}, \citet{angelopoulos2023conformal} and \citet{fontana2023conformal} for in-depth presentations of CP.

\subsection{Related Work in Federated Learning}
\label{sec:back-fl}

As already mentioned, FL methods are today mostly focused on the training part of the learning process (i.e., fitting $\fh$ to the data). Nevertheless, a few recent works have considered other types of FL problems that can be related to our work. The closest related work is the one of  \citet{lu2021distribution} which, to the best of our knowledge, is the only paper claiming to perform conformal prediction in the FL setting. Their idea is to locally calculate the quantiles $\widehat{Q}_{(\lceil (n+1) (1-\alpha)\rceil)}$ for all agents and to average them in the central server. Unfortunately, they do not prove that their prediction set has valid coverage. Furthermore, their method is non-robust, especially when the size of local data sets is small, and their experiments (and ours, in Section 
\ref{sec:xps}) suggest that this set is generally too large. We show in the next sections that by considering a quantile-of-quantiles instead of an average of quantiles, the method we propose addresses these limitations. \citet{gauraha2021synergy} propose an ensemble-based CP approach that can be performed in a distributed setting. However, they assume that a shared calibration set is available on the central server, which is unrealistic in FL. Finally, we can also mention recent works on federated evaluation of classifiers \citep{cormode2022federated}, federated quantile computation \cite{andrew2021differentially, pillutla2022differentially}, and on uncertainty quantification with Bayesian FL \cite{el2021federated,kotelevskii2022fedpop} which, although related to our work, do not study CP and do not allow to obtain coverage guarantees.

%% file: CP-QQ.tex

\label{sec:main}

In this section, we present a method to perform conformal prediction in a \emph{one-shot} FL setting  \citep{guha2019one, zhang2012communication}, where only one round of communication from the agents to the central server is allowed.

\subsection{Setup and Objective}

Consider a set of $m \in \IN^*$ agents, with their own local data, that seek to collaborate in order to compute a valid prediction set. For simplicity, we suppose that each agent has exactly $n$ calibration data points, and refer to Appendix~\ref{sec:diff_n} for the case where agents have calibration sets of different sizes. We also assume that the predictor $\fh$ is given in advance: for instance, it could be learned on a separate set of data points using standard FL algorithms such as FedAvg \citep{mcmahan2017communication}. We therefore only focus on the calibration of the prediction set and not on the training step.  As a consequence, in the following, all theoretical statements are made conditionally on $\fh$ (often implicitly).

Formally, each agent $j\in\{1,\dots,m\}$ holds a local calibration data set $\mathcal{S}^{(j)} \triangleq (S^{(j)}_1, \ldots, S^{(j)}_n)$ composed of $n$ scores, where $S^{(j)}_i = s(X^{(j)}_i, Y^{(j)}_i)$ is the score associated to the $i$-th calibration data point of agent $j$ and we want to find a particular value $\widehat{q}$ such that for a test point $(X, Y)$, the set $\Chat(X) = \left\{y \in \IR : s(X, y) \leq \widehat{q} \right\}$ contains the unknown response $Y$ with probability at least $1-\alpha$. In the centralized case, the split CP method presented in Section \ref{subsec:spiltCP} requires to order all the scores and to choose $\widehat{q}$ as the $\lceil(mn+1)(1-\alpha)\rceil$ smallest score. In one-shot FL, this global ordering step is only possible if the agents send their whole list of local scores to the server. This naive implementation of the split CP method is impractical, due to both privacy concerns and unacceptable communication costs, requiring us to design another strategy. As a single round of communication is allowed, the main difficulty is to choose what should be sent from the agents to the server, and what kind of aggregation should be done by the server to yield the desired coverage.

\subsection{Main Contribution: \method}
\looseness=-1 Our method is based on the idea that each agent $j$ should return a quantile of its local scores $\mathcal{S}^{(j)}$, in the same way as for the split CP method described in Section~\ref{subsec:spiltCP}. The main questions that then arise are (i) which quantile of the scores the agents should send, and (ii) how to aggregate them at the central server level. \citet{lu2021distribution} propose to use an empirical average, but this aggregation strategy is not satisfactory. This is obvious in the extreme case where $n=1$ (a single data point per agent): it amounts to calculating the average of the local scores, which typically fails to provide the desired coverage \eqref{eq:marg_cp}. Instead, we propose to select a quantile of the locally computed quantiles. This \emph{quantile-of-quantiles} estimator is defined below.

\begin{definition}[Quantile-of-quantiles]
For any $(\ell, k)$ in $\llbracket 1, n \rrbracket \times \llbracket 1, m \rrbracket$, the Quantile-of-Quantiles (QQ) estimator is defined by
\begin{equation}
    \label{eq:CP-QQ}
    \Qh_{(\ell, k)} \triangleq \Qh_{(k)}\left(\Qh_{(\ell)}(\mathcal{S}^{(1)}), \ldots, \Qh_{(\ell)}(\mathcal{S}^{(m)})\right)\; ,
\end{equation}
where $\Qh_{(\cdot)}(\cdot)$ is defined by Equation~\eqref{def:Qk}.
\end{definition}
In words, QQ takes for each agent the $\ell$-th smallest local score and then
takes the $k$-th smallest value of these scores. This requires a single round
of communication and thus fits the constraints of one-shot FL. The associated
plug-in prediction set is 
\begin{equation}
	\label{CPQQ_set}
    \Chat_{\ell, k}(X) = \left\{y \in \IR : s(X, y) \leq \Qh_{(\ell, k)} \right\} \, .
\end{equation}
Our objective is now to find $(\ell, k)$ such that $\IP(Y \in \Chat_{\ell, k}(X))$ is closest possible to $1-\alpha$ while being guaranteed to be above. To this aim, we derive the following result.

\begin{theorem}
\label{them:main}
Let $\{(X^{(j)}_i, Y^{(j)}_i)\}_{i,j=1}^{m,n}$ and $(X, Y)$ be i.i.d. random variables (given $\fh$). For any $(\ell,k) \in \llbracket 1, n \rrbracket \times \llbracket 1, m \rrbracket$ we have:
\begin{equation}
\label{eq:main_equa}
\begin{split}
   &\IP\left(Y \in \Chat_{\ell, k}(X)\right) \geq M_{\ell, k} \triangleq \\
   & 1 - \displaystyle \dfrac{1}{m n + 1} \sum^{m}_{j=k} \binom{m}{j} \sum^{n}_{I_{1,j}=\ell} \sum^{\ell-1}_{I^c_{1,j}=0}  \dfrac{\binom{n}{i_1} \cdots \binom{n}{i_m}}{\binom{m n}{i_1 + \cdots + i_m}}\; ,
\end{split}
\end{equation}
where $I_{1,j} = \{i_1, \ldots, i_j\}$ and $I^c_{1,j} =\{i_{j+1},\ldots,i_m\}$. Moreover, when the associated scores $\{S_i^{(j)}\}_{i,j=1}^{n,m}$ and $S\triangleq s(X,Y)$ have continuous c.d.f, \eqref{eq:main_equa} is an equality.
\end{theorem}

The proof is given in Appendix~\ref{app:proof-them:main}. This theorem shows that we can lower bound the probability of coverage of our quantile-of-quantiles prediction set by a quantity $M_{\ell, k}$ that does not depend on the data distribution but only on $m$, $n$, $\ell$ and~$k$. Furthermore, the lower bound becomes an \emph{equality} when scores have a continuous c.d.f. This is the case, for instance, with the fitted absolute residual when the conditional distribution of $Y$ given $X$ has a continuous c.d.f., i.e., when the noise distribution is atomless. Note that although the theorem requires the data points to be i.i.d., in fact only the scores need to satisfy this hypothesis (conditionally to $\fh$). This is interesting since there are situations where the scores are i.i.d. even though data distributions are different across agents. In Section \ref{sec:hetero}, we further discuss the impact of data heterogeneity across agents, an important aspect of many FL applications.

Based on Theorem~\ref{them:main}, our algorithm returns $\Qh_{(\ell^*, k^*)}$ with
\begin{align}
  \label{eq:lk_star}
  (\ell^*, k^*) = { \arg\min_{\ell, k}} \left\{M_{\ell, k} : M_{\ell, k} \geq 1-\alpha\right\} \; .
\end{align}
By construction, the associated set \eqref{CPQQ_set} is \emph{marginally valid}, in the sense that it satisfies the desired coverage \eqref{eq:marg_cp}. The full procedure, called \emph{Federated Conformal Prediction with Quantile-of-Quantiles} (\method), is summarized in Algorithm~\ref{alg:CP-QQ}.

\begin{algorithm}[t]
\small
\caption{\method}
\begin{algorithmic}
    \State\textbf{Input:} Local scores $\{\mathcal{S}^{(j)}\}_{j=1}^m$, $\alpha, M$ (see Equation~\eqref{eq:main_equa}) 
    
    \vspace{.5em}
    \State $(\ell^*, k^*) \longleftarrow {\displaystyle \arg\min_{\ell, k}} \left\{M_{\ell, k} : M_{\ell, k} \geq 1-\alpha\right\} $ 
    \vspace{.5em}

    \For{$j=1, \ldots, m$}
        \State Agent $j$ sends $\Qh_{(\ell^*)}(\mathcal{S}^{(j)})$ to the central server
    \EndFor
    
    \State Central server returns $\Qh_{(k^*)}\left(\Qh_{(\ell^*)}(\mathcal{S}^{(1)}), \ldots, \Qh_{(\ell^*)}(\mathcal{S}^{(m)})\right)$
\end{algorithmic}
 \label{alg:CP-QQ}
\end{algorithm}

\paragraph{Particular cases.}
To gain more intuition on our \method~procedure, let us consider the two extreme cases $n=1$ and $n\rightarrow \infty$. When $n=1$, each agent sends its unique score to the server. Thus, by Theorem~\ref{thm:cp_base}, it suffices for the server to compute the $k$-th smallest score with $k = \lceil(m+1)(1-\alpha)\rceil$ to obtain a valid set. In the other extreme case where $n \rightarrow \infty$, if the agents send their $\ell$-th smallest score with $\ell=\lceil(n+1)(1-\alpha)\rceil$, each agent has in fact sent the true quantile of order $(1-\alpha)$ of the distribution of~$S$. The server can therefore choose any of these values and obtains a valid set. We see that in both cases, if both the agents and the server compute appropriate quantiles, we can obtain a valid set. Our method extends this idea to any values of $m$ and $n$ using Theorem \ref{them:main} and Equation \eqref{eq:lk_star}. In Appendix \ref{sec:max_machine}, we study another interesting specific case where each machine sends its maximum value, i.e., $\ell = n$.

\paragraph{Computational optimizations.}
The brute-force computation of $M_{\ell, k}$ in Equation \eqref{eq:main_equa} for all $(\ell,k)$ can be quite costly in practice. To accelerate this step, we describe in Appendix \ref{sec:fast_comput} an efficient way to compute $M_{\ell, k}$, based on the calculation of rectangular probabilities of a multivariate hypergeometric distribution.

We also note that $M=(M_{\ell, k})_{(\ell,k) \in \llbracket 1, n \rrbracket \times \llbracket 1, m \rrbracket}$ or $(\ell^*,k^*)$ can be precomputed and reused across multiple executions of \method. Indeed, as $M$ and $(\ell^*,k^*)$ are independent from the distribution of the data (Theorem~\ref{them:main}), they do not change as long as $m$ (the number of agents) and $n$ (the size of local data sets) remain fixed. This is the case for instance when computing prediction sets for multiple score functions $s$, predictors $\fh$, and miscoverage rates $\alpha$ on the same data.

\subsection{Upper Bound on the Probability of Coverage}
\label{sec:upper}
While by construction our probability of coverage is necessarily lower bounded by $1-\alpha$, it is also interesting to have an upper bound, guaranteeing that the coverage of our prediction set is not too large. In the centralized case, if the scores have a continuous c.d.f., the split CP method with a calibration set of size $mn$ gives $\IP(Y \in \Chat(X)) \leq 1-\alpha + 1/(mn+1)$ (Theorem \ref{thm:cp_base}). This means that when there is only one agent (or when agents do not collaborate), this probability is upper bounded by $1-\alpha + 1/(n+1)$.

Assuming that the scores have a continuous c.d.f., in Figure 
\ref{fig:compare_upper} we compare the two upper bounds with the value of $M_{\ell^*, k^*} = \IP(Y \in \Chat_{\ell^*, k^*}(X))$ returned by \method. Recall that, by Theorem~\ref{them:main}, $M_{\ell^*, k^*}$ is equal to the exact coverage of $\Chat_{\ell^*, k^*}(X)$. Figure~\ref{fig:compare_upper} shows that \method~returns prediction sets with coverage (in blue) comparable to the (tight) upper bound of the centralized case with $mn$ calibration points (in orange). We also see that the coverage is much larger if we consider the data of a single agent (in red), which illustrates the advantage of our method and the need for collaboration between the agents.

The form of our quantile-of-quantiles estimator does not allow us to extend the proof techniques of the centralized framework and obtain a theoretical upper bound similar to the one of Theorem~\ref{thm:cp_base}. Nevertheless, the results obtained in Figure \ref{fig:compare_upper} make us conjecture that an upper bound could be of the same order as in the centralized framework, i.e., in $1-\alpha + \calO(1/(mn+1))$.
\begin{figure}[t]
	\centering
	\includegraphics[width=1.\linewidth]{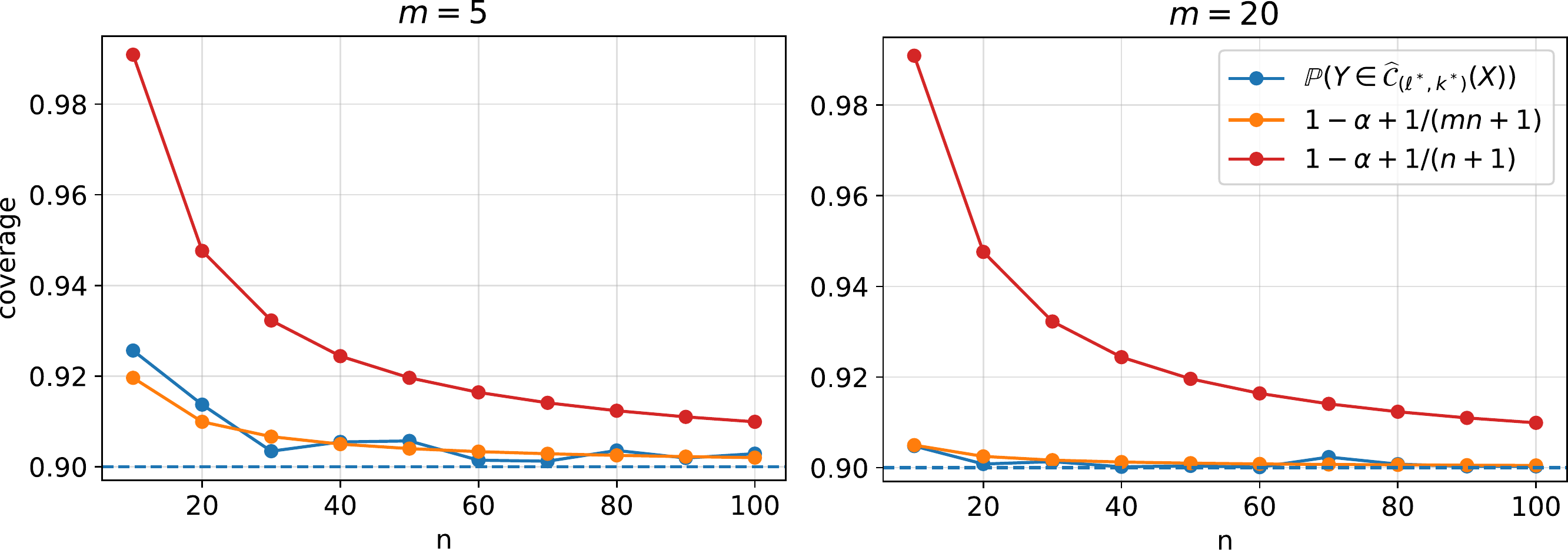}
	\vspace{-2em}
	\caption{Comparison of the exact value of $\IP(Y \in \Chat_{\ell^*, k^*}(X)) = M_{\ell^*, k^*}$ (blue) with the upper bound either when data are centralized (orange) or when there is only one agent (red). Parameters are $\alpha=0.1, m=\{5, 20\}$, and $n=\{10, \ldots, 100\}$.}
	\label{fig:compare_upper}
\end{figure}
\subsection{Conditional Coverage Guarantee}
In practice, we are interested in the coverage rate for test points when the data set is fixed. However, the guarantee in \eqref{eq:marg_cp} does not address this as the probability is also taken over the (calibration) data. In other words, it bounds the miscoverage rate \emph{on average} over all possible calibration data points (and over a training set if $\fh$ is learned). Instead, we can define the \emph{conditional} miscoverage rate as a function of the calibration data: 
\begin{equation} \label{def.alphaP(Dmn)}
    \alpha_P(\mathcal{D}_{mn}) = \IP\left(Y \notin \Chat_{\ell, k}(X) \mid \fh , \mathcal{D}_{mn}\right) \; ,
\end{equation}
with $\mathcal{D}_{mn}$ the full calibration set without the test point $(Y, X)$. While, by construction of $\Chat_{\ell, k}(X)$, the expectation of $\alpha_P(\mathcal{D}_{mn})$ is smaller than $\alpha$, the random variable $\alpha_P(\mathcal{D}_{mn})$ may have a high variance. In particular, it is possible to construct a scenario where $\IP\left(\alpha_P(\mathcal{D}_{mn}) = 1 \right) = \alpha  \text{ and } \IP\left(\alpha_P(\mathcal{D}_{mn}) = 0 \right) = 1-\alpha$ \citep{bian2022training}.

Here, we have $\IE[\alpha_P(\mathcal{D}_{mn})] = \alpha$ but a non-negligible proportion of calibration data sets might result in a poor conditional coverage even though the average coverage is still $1 - \alpha$. In practice, we want to have $\alpha_P(\mathcal{D}_{mn}) \approx \alpha$ with a probability close to~$1$ to avoid this unfavorable scenario.

In the following theorem, we show that it is possible to control the conditional miscoverage of \method.
\begin{theorem}
	\label{them:cond}
In the framework of Theorem~\ref{them:main}, if $\delta \in (0, 0.5]$ and $\ell\cdot k \geq (1-\alpha)\cdot mn$, then the conditional miscoverage rate---defined by Eq.~\eqref{def.alphaP(Dmn)}---is controlled as follows:
	\begin{equation}
	\IP\left( \alpha_P(\mathcal{D}_{mn}) \leq \alpha + \sqrt{\dfrac{\log(1/\delta)}{2mn}} \right) \geq 1-\delta \; .
	\end{equation}
\end{theorem}
Theorem~\ref{them:cond} is proved in Appendix~\ref{app.pr.them:cond}. 
It states that the probability that a particular data set results in a conditional miscoverage rate much higher than $\alpha$ vanishes with the number of data points used for calibration. A similar bound is obtained in the centralized setting \citep{vovk2012conditional, bian2022training} for the split method. However, note that Theorem~\ref{them:cond} holds only for couples $(\ell, k)$ verifying a condition not necessarily verified by the couple $(\ell^*, k^*)$ used by \method. Nevertheless, our experiments suggest that this could still be true for $(\ell^*, k^*)$, up to a slight modification of the bound. However, similarly to the upper bound on the probability of coverage (see Section~\ref{sec:upper}), the proof of this statement is difficult because it requires to study the rank of $\Qh_{(\ell, k)}$ in the full data set which, contrary to the centralized case, is a random variable. In the proof of Theorem~\ref{them:cond}, we rely on an almost sure lower bound for this rank, which is conservative and negatively impacts the final result. In the centralized case, the rank is almost surely fixed and this greatly simplifies the theoretical analysis.
 
\subsection{Impact of Heterogeneous Data}
\label{sec:hetero}
An important challenge in FL is to deal with data heterogeneity across agents \citep{li2020federated,kairouz2021advances,bars2022refined}. This heterogeneity can yield different distributions of scores across agents and thus affects the coverage of the set returned by CP methods. To better understand these effects, we no longer assume that all the variables are drawn from the same distribution. Instead, we only suppose that the local data points of agent $j$ are drawn i.i.d. from an \emph{agent-specific distribution} with a test point also drawn from a potentially different distribution.

As we do not have any information on the underlying distributions of the scores, we study how data heterogeneity affects the coverage of the set returned by \method, i.e., we quantify how much we lose in coverage if we apply the same strategy as in the i.i.d. case. Intuitively, the more the distributions of the scores $\{S^{(j)}\}_j$ are similar and close to the one of $S$, the less we lose in coverage. This is made precise in the following result.

\begin{proposition}
\label{prop:hetero}
Assume that the calibration data $\{(X^{(j)}_i, Y^{(j)}_i)\}_{i,j=1}^{m,n}$ and the test point $(X, Y)$ are such that, given $\fh$, the corresponding scores $\{ S_i^{(j)}\}_{i,j=1}^{n,m} , S$ are independent, and that for every $j \in \llbracket 1, m \rrbracket$, $\{ S_i^{(j)}\}_{i=1}^n$ are i.i.d. 
Let $\{\widetilde{S}^{(j)}_i\}^{n, m}_{i, j=1}, \widetilde{S}$ be i.i.d. random variables (given $\fh$). Define, for every $j \in \llbracket 1, m \rrbracket$, $p_j^*(S) = \IP(S^{(j)}_{(\ell^\star)} \leq S \vert S )$ and $\tilde{p}^*(\widetilde{S}) = \IP(\widetilde{S}^{(1)}_{(\ell^\star)} \leq \widetilde{S} \vert \widetilde{S})$. 
Then, we have 
\begin{align*}
&\IP\big(Y \in \Chat_{\ell^\star, k^\star}(X) \big) 
\geq 1-\alpha 
\\
&\qquad - \IE \biggl[ \dTV \Bigl( 
\PoisBin \bigl( p^*(S) \bigr) 
\, , \, 
\Bin\bigl( m , \tilde{p}^*(\widetilde{S}) \bigr) 
\Bigr)  \biggr] 
\, ,
\end{align*}
where $\dTV (\cdot, \cdot)$ is the total-variation (TV) distance, 
$\PoisBin$ the Poisson-Binomial distribution and $\Bin$ the binomial distribution.
\end{proposition}
Proposition~\ref{prop:hetero} is proved in Appendix~\ref{app.pr.prop:hetero}. The general idea of this result is that when variables are i.i.d., probabilities on order statistics only depend on the c.d.f. of a certain binomial distribution, whereas when the variables are independent but with different distributions, the binomial needs to be replaced by a Poisson-Binomial distribution. The inequality indicates that, in the heterogeneous case, the coverage is reduced by the TV distance between the two distributions. We note that this distance can be upper bounded in specific cases (see Appendix~\ref{app.pr.prop:hetero}) and that it is equal to $0$ when all the data are i.i.d and $S = \widetilde{S}$. We leave to future work the precise characterization of cases where the TV distance is negligible in front of $1-\alpha$.

%% file: privacy.tex

While FL methods are often informally claimed to mitigate privacy issues, they still leak information about the local data sets during the execution of the algorithm. In the case of \method, it is easy to see how revealing a particular quantile of the local score distribution may leak sensitive information. In this section, we propose a privacy-preserving version of \method~based on Differential Privacy (DP) \cite{dwork2014algorithmic}, a mathematical notion of privacy that essentially requires that the output distribution of a randomized algorithm is not too sensitive to a small modification of the input data set. In particular, we consider the strong \emph{Local DP} (LDP) model where agents do not trust the central server and must locally privatize the messages they send.

Formally, for any $\varepsilon>0$, a randomized algorithm $\calA$ is said to be $\varepsilon$-LDP if for any two local data sets $\calS$ and $\calS'$ that differ in a single data point (we call them \emph{neighboring}), and any set of possible outputs $O$, we have:
\begin{equation}
\label{eq:dp}
\IP \bigl( \calA(\calS)\in O \bigr) \leq \exp{(\varepsilon)}\IP \bigl( \calA(\calS')\in O \bigr) \;.
\end{equation}
A smaller $\varepsilon$ therefore yields a better privacy. In our specific framework, $\calS$ and $\calS'$ correspond to two neighboring calibration data sets of an agent $j$ and $\calA(\calS)$ to the information sent by $j$ to the central server.

Our approach builds upon the (centralized) differentially private quantile mechanism recently introduced by \citet{Angelopoulos_2022} and summarized in Algorithm \ref{alg:DP-ordered}. The main idea is to apply the exponential mechanism \cite{mcsherry2007mechanism} to a discretization of the scores into bins and with an appropriate choice of utility function. It requires to fix a number of bins $B \in \mathbb{N}$, an upper bound on the scores $S_{\max}$ and a set of points $0 = e_0 < e_1 < \cdots < e_{B-1} < e_B = S_{\max}$ defining the bins $I_b=(e_{b-1},e_b]$. Algorithm \ref{alg:DP-ordered} is $\varepsilon$-DP by a direct application of the exponential mechanism with utility function $w_b$ and sensitivity $\Delta_q$.
\begin{algorithm}[t]
    \small
        \caption{Differentially Private Quantile}
        \begin{algorithmic}
            \State\textbf{Input:} Scores $(S_1,\ldots,S_n) \in \IR^n$, quantile $q \in (0,1)$, privacy level $\varepsilon > 0$, bins $\{I_1, \ldots, I_B\}$
            
                \For{$i=1, \ldots, n$}
                    \State Compute the discretized score $\bar{S}_i = e_b$
                    such that $S_i \in I_b$
                \EndFor
                \For{$b=1, \ldots, B$}
                    \State Compute the weight $w_b=\max\Big\{\frac{|\{i : 
                    \bar{S}_i < e_b\}|}{q}, \frac{|\{i : \bar{S}_i > e_b\}|}{1-q}\Big\}$
                \EndFor
                    \State $\Delta_q \longleftarrow \max\{\frac{1}{q}, \frac{1}{1-q}\}$
                    \State\textbf{Output:} Bin $e_b$ with probability
                    $e^{-\frac{\varepsilon w_b}{2\Delta_q}}/\sum_{b'=1}^Be^{-
                    \frac{\varepsilon w_{b'}}{2\Delta_q}}$ 
            \end{algorithmic}
     \label{alg:DP-ordered}
\end{algorithm}
\begin{algorithm}[t]
    \small
        \caption{\methodDP}
        \begin{algorithmic}
            \State\textbf{Input:} Local scores $\{\mathcal{S}^{(j)}\}_{j=1}^m$, miscoverage level $\alpha$, $M$ (see Equation~\eqref{eq:main_equa}), privacy level $\varepsilon > 0$, bins $\{I_1, \ldots, I_B\}$, $\gamma \in (0,1)$ \vspace{.5em}
            \State The server finds $(\ell_\gamma,k_\gamma)$ as in
            \method~(Algorithm \ref{alg:CP-QQ}) with coverage level $\frac{1-\alpha}{1 - \gamma\alpha}$
            
            \State $q \longleftarrow \max\Big\{\frac{\ell_\gamma + \ell_{\text{cor}}}{n},\frac{1}{2}\Big\}$ with $\ell_{\text{cor}}$ from Eq. \eqref{eq:l-comp}
            \For{$j=1, \ldots, m$}
            \State Agent $j$ sends $\widehat{Q}_{j}^{\varepsilon}$, the output of Alg. \ref{alg:DP-ordered} with $\mathcal{S}^{(j)}$, to the server. \looseness=-1
            \EndFor
            \State\textbf{Output:} The server orders $\widehat{Q}_{1}^{\varepsilon}, \ldots, \widehat{Q}_{m}^{\varepsilon}$ and outputs the $k_\gamma$-th value denoted $\widehat{Q}_{(k_\gamma)}^{\varepsilon}$.
            \end{algorithmic}
     \label{alg:CPP-QQ}
\end{algorithm}
\paragraph{\methodDP~ method.}
Our private algorithm, called \textit{Federated Conformal Private Prediction} (FedCP$^2$)-QQ, is an extension of \method~(Algorithm \ref{alg:CP-QQ}) with two key modifications: (i) exact local quantile computations are replaced by calls to DP Quantile (Algorithm~\ref{alg:DP-ordered}), and (ii) the orders of client and server-level quantiles are adjusted to guarantee the desired coverage. More precisely, if the central server asks for the $\ell$-th smallest score of each agent, then the agents use Algorithm~\ref{alg:DP-ordered} to return a randomized bin around the true quantile $\Qh_{(\ell)}(\calS^{(j)})$. To achieve the desired coverage $1-\alpha$ despite the randomness due to privacy, the server computes $(\ell_\gamma, k_\gamma)$ such that $\IP(S \leq \Qh_{(\ell_\gamma, k_\gamma)})$ is above but close to $\frac{1-\alpha}{1-\gamma \alpha}$, where $\gamma \in (0,1)$ is a free parameter. Because the agents might return bins smaller than the one of the requested $\ell_\gamma$-th score, the central server further compensates by asking agents for their $(\ell_\gamma + \ell_{\text{cor}})$-th smallest score with
\begin{equation}
    \label{eq:l-comp}
    \ell_{\text{cor}} = \left\lceil \frac{2}{\varepsilon}\log{ \left( \frac{B}{1-(1-\gamma \alpha)^\frac{1}{m}} \right) } \right\rceil \; .
\end{equation}
Note that the smaller the privacy parameter $\varepsilon$ (more privacy), the bigger the correction $ \ell_{\text{cor}}$. At first sight, one could think that $B$ should be taken small to reduce the correction. In practice, it should also be taken sufficiently large to avoid aggressive rounding that could lead to a large final prediction set. We refer to \citet[Section~4.2]{Angelopoulos_2022} for an in-depth discussion on the selection of the number of bins $B$. The following theorem ensures that Algorithm~\ref{alg:CPP-QQ} preserves privacy and allows to construct prediction sets that satisfy the desired coverage. The proof is given in Appendix~\ref{app:privacy-thme}.

\begin{theorem}
\label{thme:private}
For any $\varepsilon>0$, Algorithm \ref{alg:CPP-QQ} satisfies $\varepsilon$-LDP. Moreover, denoting $\Chat_\varepsilon(X) = \{y \in \IR : s(X, y) \leq \widehat{Q}^\varepsilon \} $ with $\widehat{Q}^\varepsilon$ the output of the algorithm, we have 
\[ 
\IP \bigl( Y \in \Chat_\varepsilon(X) \bigr) \geq 1- \alpha 
\; . 
\]
\end{theorem}
\looseness=-1
\textbf{Choosing $\gamma$.}
Intuitively, in order to be equivalent to the non-private \method, $\gamma$ should tend to $0$ and the privacy parameter $\varepsilon$ should tend to infinity. To select $\gamma$ automatically for any given $\varepsilon>0$, we propose a grid-search strategy. We look for the $\gamma$ that brings the smallest amount of correction, which we evaluate using the pre-computed table $M$. More precisely, for a given $\gamma$, we evaluate $M_{\ell_\gamma + \ell_{\text{cor}}, k_\gamma}$ which is the coverage obtained by the non-private \method~estimator $\Qh_{(\ell_\gamma + \ell_{\text{cor}}, k_\gamma)}$. Note that this coverage is not the one of our private estimator since each agent might return a score smaller than the $(\ell_\gamma + \ell_{\text{cor}})$-th smallest. To find the best $\gamma$, we look at the one that brings the smaller coverage $M_{\ell_\gamma + \ell_{\text{cor}}, k_\gamma}$ over the grid. To gain more intuition on the degree of correction brought by the additional randomness of the private setting, we represent in Figure \ref{fig:private_l} the quantity $M_{\ell_\gamma + \ell_{\text{cor}}, k_\gamma}$ found for the best $\gamma$ and for different values of $n$ and $\varepsilon$. This plot shows how fast the correction decreases as $n$ and $\varepsilon$ increase.

\begin{figure}[t]
    \centering
    \includegraphics[width=1.\linewidth]{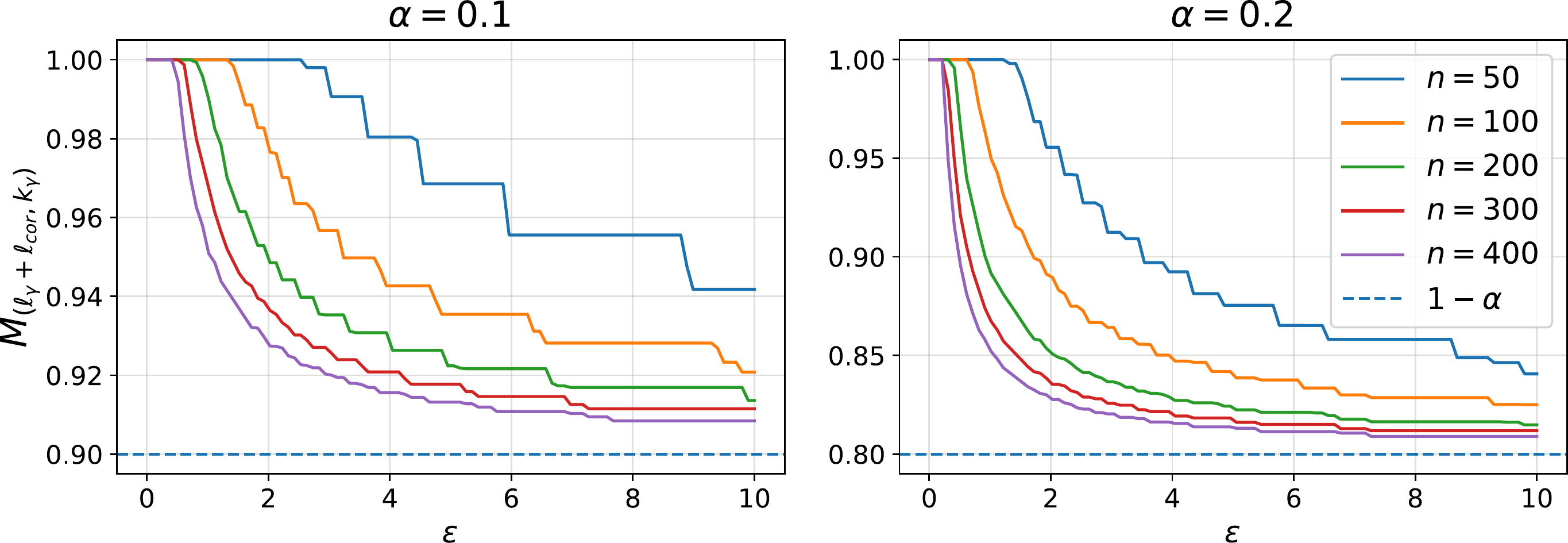}
    \vspace{-2em}
    \caption{Degree of compensation $M_{\ell_\gamma + \ell_{\text{cor}}, k_{\gamma}}$ for different values of $\alpha$, $n$ and $\varepsilon$ when $m=10$. We clearly observe that $M_{\ell_\gamma + \ell_{\text{cor}}, k_{\gamma}}$ tends to the desired coverage $1-\alpha$ (dashed lines) as $n$ and $\varepsilon$ tends to $+\infty$, which means that the compensation vanishes. }
    \label{fig:private_l}
    \vspace{-1em}
\end{figure}

\textbf{Privacy amplification by shuffling or aggregation.} To achieve better privacy-utility trade-offs, it is common in FL to relax the LDP model and instead assume that the agents' messages are sent to a secure computation function whose output is received by the central server. This is sometimes referred to as \emph{Distributed DP} \citep{kairouz2021advances}. Two standard secure computation primitives are compatible with \methodDP: secure shuffling \citep{shuffling} and secure aggregation \cite{bonawitz2017practical}. Secure shuffling outputs a random permutation of the messages, which still allows the server to compute the desired quantile. For secure aggregation (which outputs the sum of the messages), each agent can encode its private quantile as a one-hot vector of size $B$ indicating the corresponding bin. The sum of these vectors is then sufficient for the server to find the bin corresponding to the $k_{\gamma}$-th smallest score. In both cases, $\varepsilon$ is reduced by a factor of $\calO(1/\sqrt{m})$. In other words, if one of the previous privacy amplification schemes is used, we can replace $\varepsilon$ by $\varepsilon \sqrt{m}$ (up to a constant) and therefore reduce the correction $\ell_{\mathrm{cor}}$ by a factor $\calO(\sqrt{m})$, while still satisfying the same privacy guarantees. Detailed privacy amplification formulas are provided by \citet{shuffling} and \citet{amp_by_aggreg}.

\begin{remark}
\label{rem:privacy-for-all}
\methodDP~provides privacy guarantees with respect to the calibration data. To provide privacy guarantees with respect to the data used to train the model, one should train the model using locally differentially private algorithms \citep[see e.g.][]{dp_fed_sgd_user_level,dp_fedavg_user_level,Noble2022a}. Note that the training and calibration data sets are disjoint, and that \methodDP~only post-processes the private model to compute the calibration scores. Therefore, if model training satisfies $\varepsilon_1$-LDP and \methodDP~satisfies $\varepsilon_2$-LDP, the full pipeline satisfies $\max (\varepsilon_1,\varepsilon_2)$-LDP thanks to parallel composition.
\end{remark}

%% file: xps.tex

In this section, we evaluate \method~on synthetic and real regression data sets. Additional experiments on unbalanced data sets and on \methodDP~are presented in Appendices \ref{sec:diff_n} and~\ref{sec:private_xp}. The code of our two methods is available at \url{https://github.com/pierreHmbt/FedCP-QQ}.

Depending on the experiments, we use the split CP method presented in Section \ref{sec:background} or its popular variant Conformalized Quantile Regression (CQR) \citep{romano2019conformalized}, which is directly compatible with our approach. For split CP, $\fh$ is a standard regressor, the score function $s$ is $s(X, Y) = \lvert Y - \fh(X) \rvert$, and the resulting prediction set is an interval of constant length $[\fh(X) \pm \hat{q}]$. In CQR, $\fh$ is replaced by a couple $(\fh_{\alpha/2}, \fh_{1-\alpha/2})$ where $\fh_{\beta}$ is a quantile regressor of order $\beta$ \citep{koenker1978regression} and $s(X, Y)= \max(\fh_{\alpha/2}(X)-Y, Y - \fh_{1-\alpha/2}(X))$. In contrast to split CP, CQR returns sets of the form $[\fh_{\alpha/2}(X) - \hat{q}, \fh_{1-\alpha/2}(X) + \hat{q}]$ which have a size adaptive to heteroscedasticity.

For both split CP and CQR, we use \method~to find the value of $\widehat{q}$ (calibration step). We compare it with the centralized baseline (Equation \ref{def:Qk}) and \methodAvg, the federated approach proposed by~\citet{lu2021distribution}. Recall that the latter simply averages the $m$ quantiles of order $\lceil (n+1)(1-\alpha) \rceil/n$ sent by the agents  (see Section \ref{sec:back-fl}).
\begin{figure}[t]
    \centering
    \includegraphics[width=.73\linewidth]{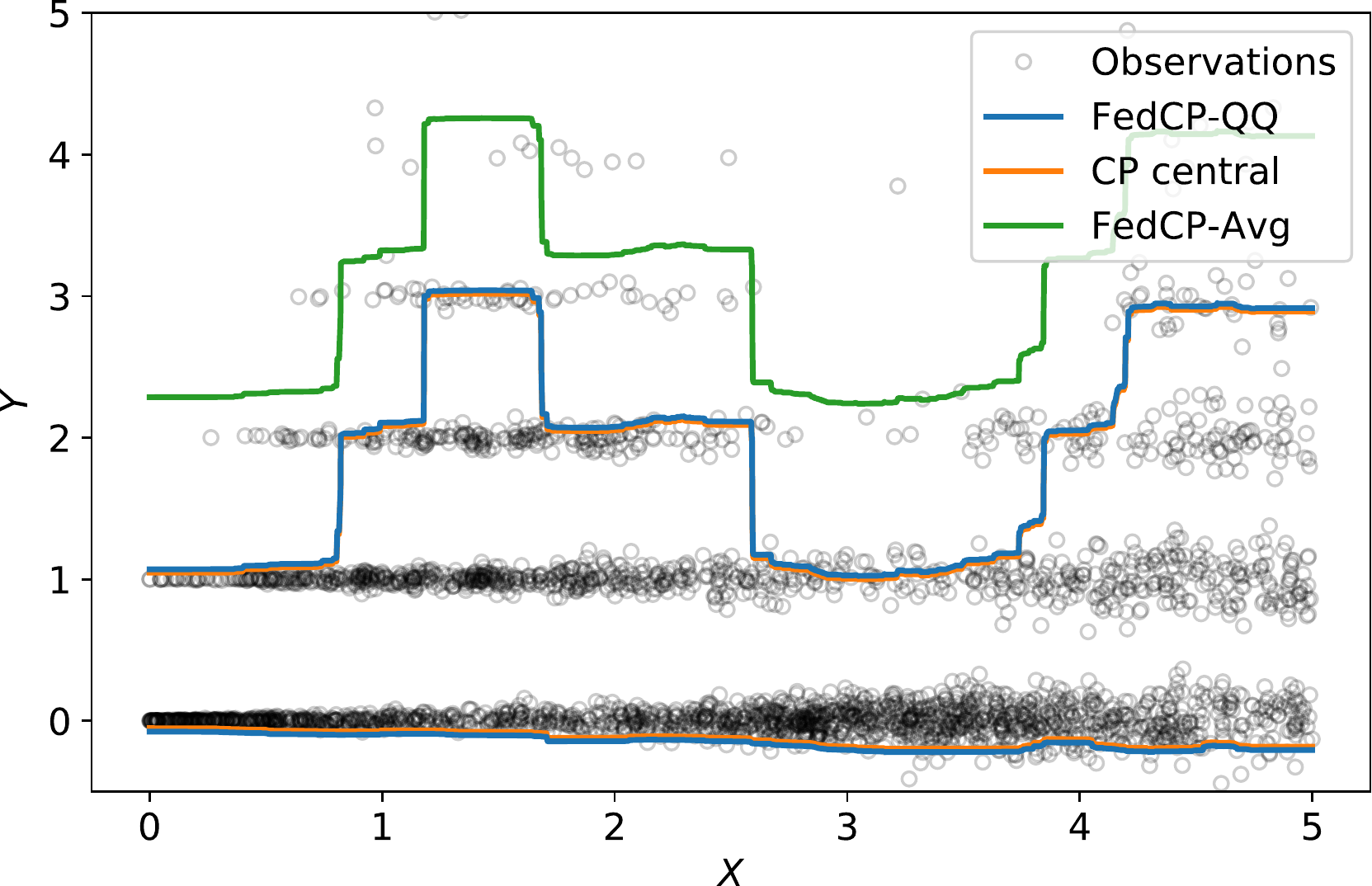}
    \vspace{-1em}
    \caption{Prediction intervals on simulated data with \method~(ours), centralized, and \methodAvg~calibrations. The lower bound of the set returned by \methodAvg~is beyond the figure.}
    \label{fig:synthetic_xp_CQR_QQ}
    \vspace{-1em}
\end{figure}
\begin{figure*}[t]
    \centering
    \includegraphics[width=.495\linewidth]{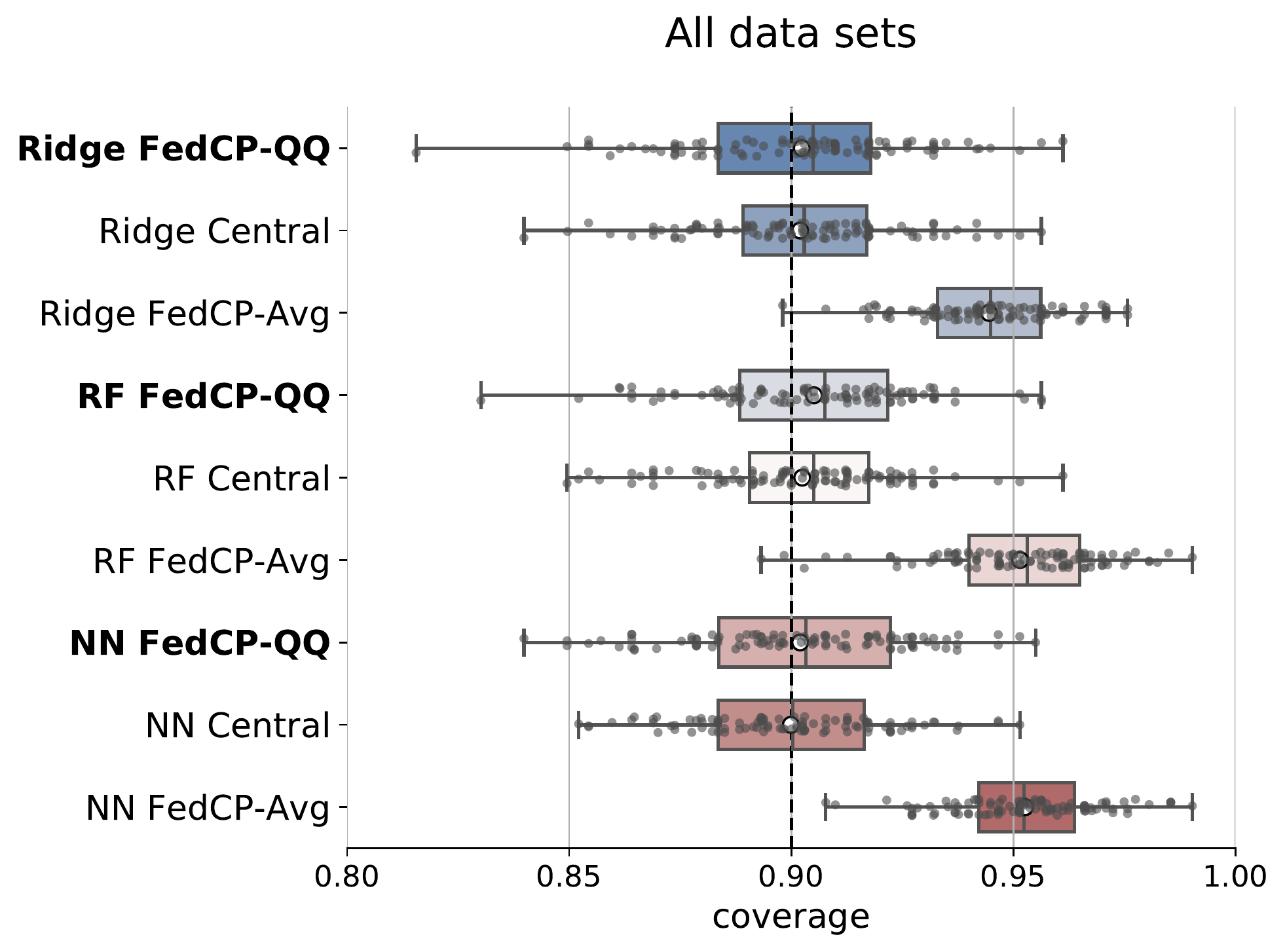}
    \includegraphics[width=.495\linewidth]{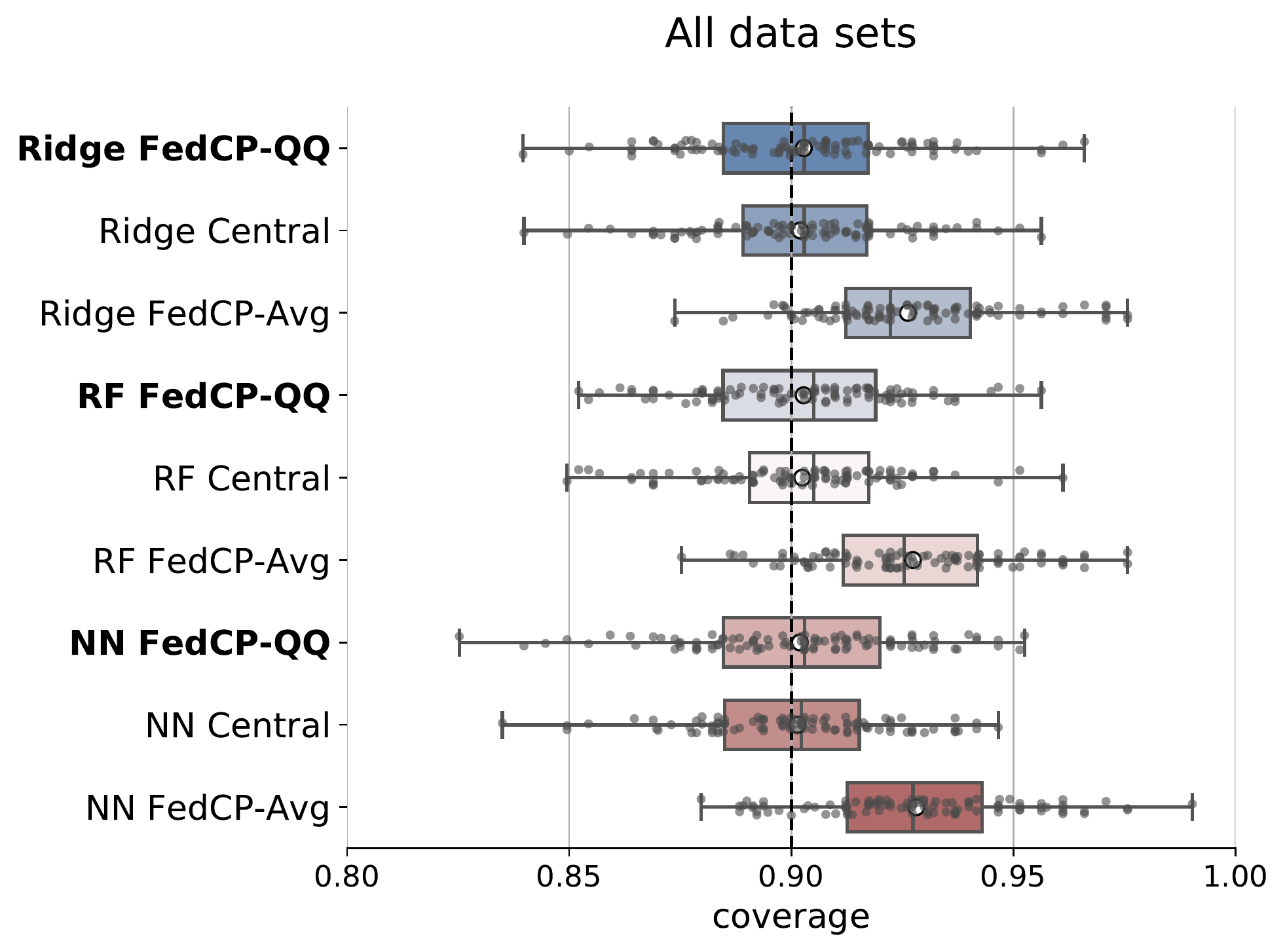}
    \vspace{-1em}
    \caption{Empirical coverages of prediction intervals ($\alpha= 0.1$) constructed by various methods. On the left, when $m \gg n$. On the right, when $m \ll n$. Our method \method~is shown in bold font. The white circle represents the mean.}
    \label{fig:real_xp_all}
\end{figure*}
\subsection{Synthetic Data} \label{sec.simus.synthetic}
\textbf{Data set.}
We draw $2000$ independent, univariate random variables $X_i$ from a uniform distribution on  $[1, 5]$. Following \citet{romano2019conformalized}, the response variable is sampled as
\begin{align*}
    Y_i \, \vert \, X_i \sim \text{Pois}(\sin^2(X_i) + 0.1) & + 0.03 \cdot X_i \varepsilon_{1, i} \\
    & + 25 \cdot\One{U_i < 0.01}\varepsilon_{2, i} \; ,
\end{align*}
where $\text{Pois}(\lambda)$ is the Poisson distribution with mean $\lambda$, $\varepsilon_{1,i}$ and $\varepsilon_{2,i}$ are i.i.d. standard Gaussian variables, and $U_i$ is uniform on the interval $[0, 1]$. Note that the last term of the equation can generate outliers. Then, we split the data set into two disjoint subsets: one for training and one for calibration. To simulate a FL scenario, the calibration set is divided into $m=50$ disjoint subsets of size $n=20$. Finally, we generate a test set of size $5000$ with the same properties.

\looseness=-1 We construct the prediction sets using the CQR approach where the estimation of the (quantile) regression function is made with quantile regression forests \citep{meinshausen2006quantile}. The number of trees in the forest is set to $1000$, the two parameters controlling the coverage rate on the training data are tuned using cross-validation and the remaining hyperparameters are set as done by \citet{romano2019conformalized}.

\textbf{Results.} Figure \ref{fig:synthetic_xp_CQR_QQ} illustrates the performance of the different methods when $\alpha=0.1$. We see that the set returned by \method~when the data are decentralized is almost identical to the one obtained when the data are centralized. This is not the case for \methodAvg~which outputs a larger set. This may be due to the presence of outliers in the data and because the mean (the server aggregation strategy for \methodAvg) is not robust. On the contrary, by using a quantile function to aggregate the agents' quantiles, \method~is robust to outliers and produces smaller yet valid sets. In the next subsection, we show that the same behavior is observed on real data sets.

\subsection{Real Data}
\label{sec:xps:real}

\textbf{Data sets.} We evaluate our method on five public-domain regression data sets also considered by \citet{romano2019conformalized} and \citet{sesia2021conformal}: physicochemical properties of protein tertiary structure (bio) \citep{rana2013physicochemical}; bike sharing (bike) \cite{bikeshare}; communities and crimes (community) \citep{2011crime}; Tennessee’s student teacher achievement ratio (star) \citep{achilles2008tennessee}; and concrete compressive strength (concrete) \citep{yeh1998modeling}.

In this section, we use (i) split-CP with ridge regression---the regularization parameter is tuned by cross-validation; (ii) CQR with quantile Regression Forests (RF)---the hyper-parameters are the ones used in Section~\ref{sec.simus.synthetic}; and (iii) CQR with Neural Networks (NN) for quantile regression \citep{taylor2000quantile}---the architecture and the parameters are those used by \citet{romano2019conformalized}.

The prediction sets, with a miscoverage rate fixed to $\alpha=0.1$, are either calibrated with CP in the centralized setting or in a FL setting using \method~and \methodAvg. For each experiment, we split the full data set into three parts: a training set ($40 \%$), a calibration set ($40 \%$), and a test set ($20\%$). To simulate a FL scenario, we also split the calibration set in $m$ disjoint subsets of equal size $n$. We consider scenarios where $m \gg n$, and $m \ll n$. Their exact values for each data set are given in Appendix \ref{app:add-xp}. All features are then standardized to have zero mean and unit variance. For each method, we compute the empirical coverage obtained on the test set and the average length of the conformal set. These two metrics are collected over $20$ different training-calibration-test random splits.

\textbf{Results.}
Figure \ref{fig:real_xp_all} displays the boxplots of the empirical coverages obtained by each method over all the data sets and all the $20$ different random splits (one point represents the empirical coverage obtained on one random split of one data set). Results on individual data sets are presented in Appendix~\ref{app:add-xp}, as well as boxplots of the lengths of the intervals obtained. The first observation we can make is that, on average (white circle), \method~does return intervals whose coverage is greater than $0.90$ (the desired coverage), without being too far from it. More importantly, our method returns prediction sets with coverage and length very similar to those returned by centralized calibration. In Figure~\ref{fig:real_xp_all} for instance, we see that the mean (white circle) and standard-deviation (size of the box) of the coverages obtained with \method~and the centralized baseline have comparable values, with a slightly larger standard-deviation for \method. The same kind of observation can be made concerning the length of the prediction sets (see figures in Appendix~\ref{app:add-xp}). Finally, it is interesting to note that, with \method, we obtain similar results for $m \gg n$ and $m \ll n$. This is in contrast to \methodAvg, which yields sets with higher coverages and lengths on all data sets and is therefore strictly inferior to our method. Note that Appendix~\ref{sec:private_xp} provides additional results about our DP algorithm \methodDP, showing how the coverage varies with the privacy parameter $\varepsilon$.
Overall, these experiments support the fact that \method~is a well-suited method to perform the calibration step of CP in a decentralized setting, placing it as the only one adapted to the context of (one-shot) FL.

%% file: discussion.tex

This paper introduces the method \textit{Federated Conformal Prediction with Quantile-of-Quantiles} (\method) to output valid distribution-free prediction sets in a one-shot Federated Learning context. In addition to the analysis and discussion about the different properties of our method, we also introduce \methodDP, a private version of \method~based on Local Differential Privacy. Multiple experiments highlight that our method returns prediction sets with coverage and length close to those returned in a centralized setting, supporting the fact that \method~is a well-suited method for (one-shot) FL scenarios.

This work brings many important future research directions. Among them, we expect that new proof techniques could lead to better theoretical guarantees, notably regarding conditional coverage and the private estimator. Our paper focuses on the calibration step, making it particularly suited for \emph{split}-based conformal methods. However, it would be interesting to study how our FL approach could be extended to the full conformal or the nested conformal methods \cite{gupta2022nested}. Finally, an interesting line of research is the derivation of specific estimators for cases where local data sets are not identically distributed.

%% file: appendix.tex

\section{Supplementary Discussions}

\subsection{\method~with Different $n_i$}
\label{sec:diff_n}

In the main article, we assumed for simplicity that all agents had the same amount of data $n$. Our method is in fact generalizable to the case where the agents have data sets of different sizes $n_1, \ldots, n_m$. In this case, the random variables $(\Qh_{(\ell)}(\mathcal{S}^{(1)}), \ldots, \Qh_{(\ell)}(\mathcal{S}^{(m)}))$ are no longer identically distributed as they are computed on data sets of different sizes. Hence, we need to use the cdf of INID data---see \citep[Equation~(16)]{balakrishnan2007permanents}. The right-hand side of Equation \eqref{eq:main_equa} becomes
\begin{align*}
M_{\ell_1, \ldots, \ell_m, k} &= 1 - \dfrac{1}{n_1 + \cdots + n_{m} + 1} \sum^{m}_{j=k} \sum_{A \in \mathcal{P}_j} \sum^{n_{a_1}}_{i_1=\ell_{a_1}} \cdots \sum^{n_{a_{j}}}_{i_{j}=\ell_{a_{j}}} \sum^{\ell_{a_{j+1}} - 1}_{i_{j+1}=0} \cdots \sum^{\ell_{a_m}-1}_{i_m=0} \dfrac{\binom{n_{a_1}}{i_1} \cdots \binom{n_{a_m}}{i_m}}{\binom{n_{1} + \cdots + n_{m}}{i_1 + \cdots + i_m}} \; ,
\end{align*}
where $\mathcal{P}_{j}$ is the set of  subsets of $\{1, \ldots ,m \}$ of size $j$, $A = \{a_1, \ldots, a_{j}\} \in \mathcal{P}_{j}$, and $A^c = \{a_{j+1}, \ldots, a_{m}\}$ such that $a_1 < a_2 < \ldots < a_{j}$ and $a_{j + 1 } < a_{j + 2 } < \ldots < a_m$. It can be computed in the same way as for the case where $n_1=\cdots = n_m =n$ (see Appendix \ref{sec:fast_comput}). An important difference that appears if we want to apply the methodology of \method~presented in the paper is that we now have to find different values for $\ell_1,\ldots,\ell_m$ since the local sample sizes are different. Although possible, computing $M_{\ell_1, \ldots, \ell_m, k}$ for all possible values of $(\ell_1, \ldots, \ell_m, k)$ to find the smallest one above $1-\alpha$ can be very time-consuming.

In practice, we propose to directly fix $\ell_j=\lceil (1-\alpha)(n_j+1)\rceil$ as it would be similarly done in the classical (centralized) split methodology. Hence, the previous probability function only needs to be computed for the different values of $k=1,\ldots,m$, thereby reducing significantly the computation at the cost of being slightly less close to $1-\alpha$. Note that this strategy can also be used in the context of the main paper, i.e., when $n_j=n$ and $\ell_j = \ell$.

We made an additional experiment with such unbalanced data sets using the setting of the synthetic experiments but with different sizes for each local data set. We set $\ell_j=\lceil (1-\alpha)(n_j+1)\rceil$ and find the value of $k$ such that the coverage is greater than $0.9$. Results are displayed in Figure~\ref{fig_unbalance} and, as expected, the coverage is respected.

\begin{figure}[H]
	\centering
	\includegraphics[width=.45\linewidth]{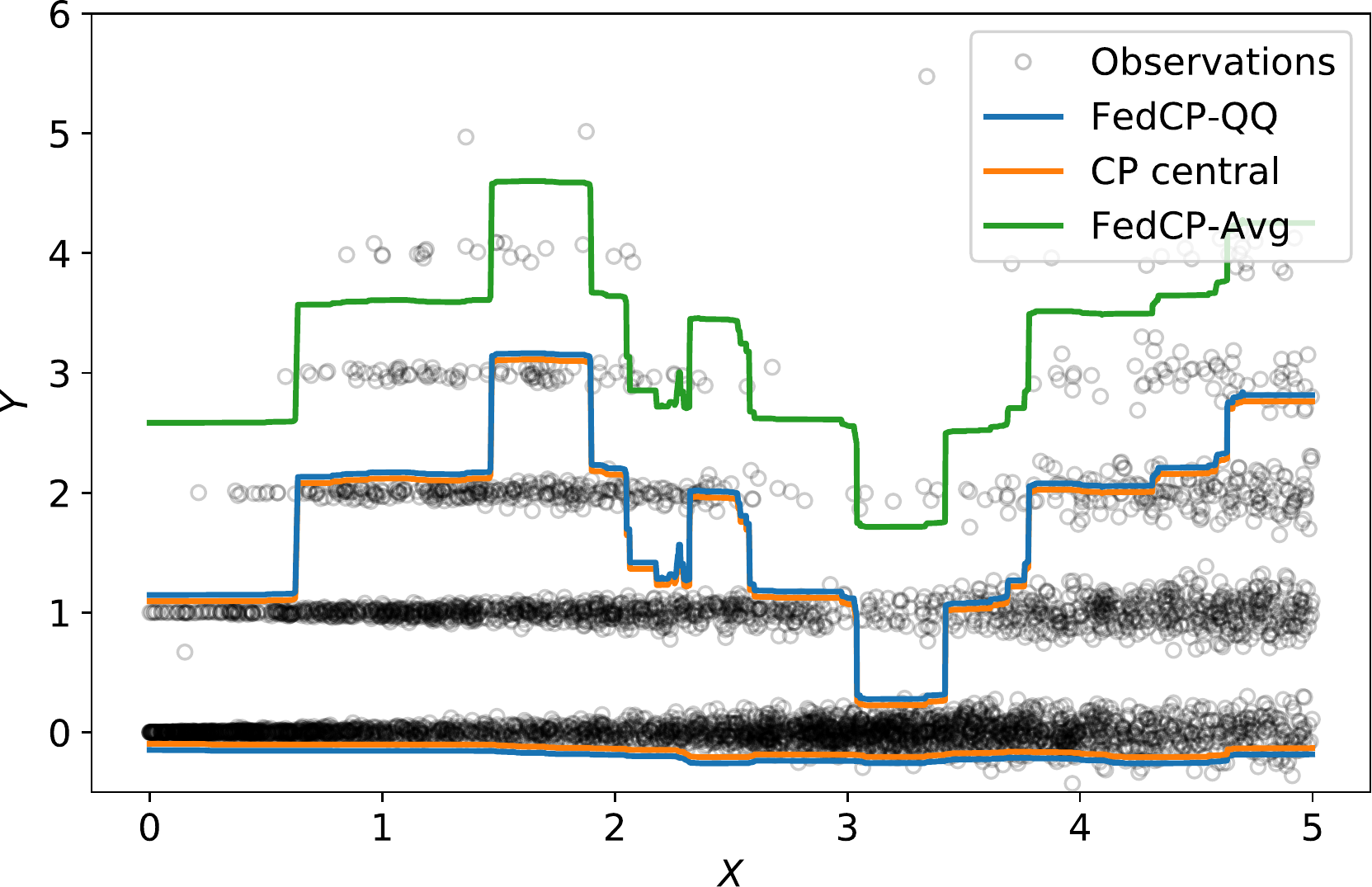}
	\vspace{-1em}
	\caption{Prediction intervals on simulated data (unbalanced case) with \method~(ours), centralized, and \methodAvg~calibrations. The lower bound of the set returned by \methodAvg~is beyond the figure.}
	\label{fig_unbalance}
\end{figure}

\subsection{Computation of Equation (\ref{eq:main_equa})}
\label{sec:fast_comput}
\input{fast_comput.tex}

\subsection{Reporting the Maximum of Each Agent}
\label{sec:max_machine}
Another particular case of interest is when $\ell = n$, i.e., agents send their maximum value. Using the fact that the c.d.f. of the maximal value of $n$ i.i.d random variables with common c.d.f. $F$ is $F^n$, it is possible to give a simpler formula for $M_{n, k}$.

\begin{proposition}
\label{prop:maxQQ}
 For every $m, n \geq 1$ and $k \in \llbracket 1, m \rrbracket$, we have
\begin{align*}
    M_{n, k} &= \dfrac{ \Gamma(k + 1/n)}{\Gamma(k)} \cdot \dfrac{\Gamma(m+1)}{\Gamma(m+1/n+1)} \; ,
\end{align*}
where $M_{n,k}$ is defined by Eq.~\eqref{eq:main_equa} and 
$\Gamma$ is the Gamma function: 
for any complex number $z$ such that $\Re(z) > 0$, 
$\Gamma(z) = \int_0^{+\infty} t^{z-1} \mathrm{e}^{-t}\mathrm{d}t$. 
	
Furthermore, when $k = k_m \triangleq \lceil m (1-\alpha)^n \rceil$, we have
\begin{align*}
    \lim_{m\to\infty} \IP\left(Y \in \Chat_{n, k_m}(X)\right) \geq 1-\alpha \; . 
\end{align*}
\end{proposition}
Proposition~\ref{prop:maxQQ} is proved in Appendix~\ref{app.pr.prop:maxQQ}. It shows that when each agent sends the maximum to the central server, by taking the $k_m$-th smallest value of these maximums with $k_m \geq m(1-\alpha)^n$, the server obtains a valid coverage of $(1-\alpha)$. Note that for a fixed $m$, $k_m$ decreases to $0$ when $n$ grows to infinity. This is expected since, intuitively, if the number of points per agent increases, the maximums also increase, and the server must compensate by taking a very small quantile of these values to obtain a coverage close to $(1-\alpha)$.

\section{Additional Experimental Results}
\subsection{Results on Individual Data Sets}
\label{app:add-xp}
We present in Figures \ref{fig:real_xp_bio} to~\ref{fig:real_xp_concrete_low_m} the results of the experiments of Section~\ref{sec:xps:real} on individual data sets.

\begin{figure}[H]
	\centering
	\includegraphics[width=.495\linewidth]{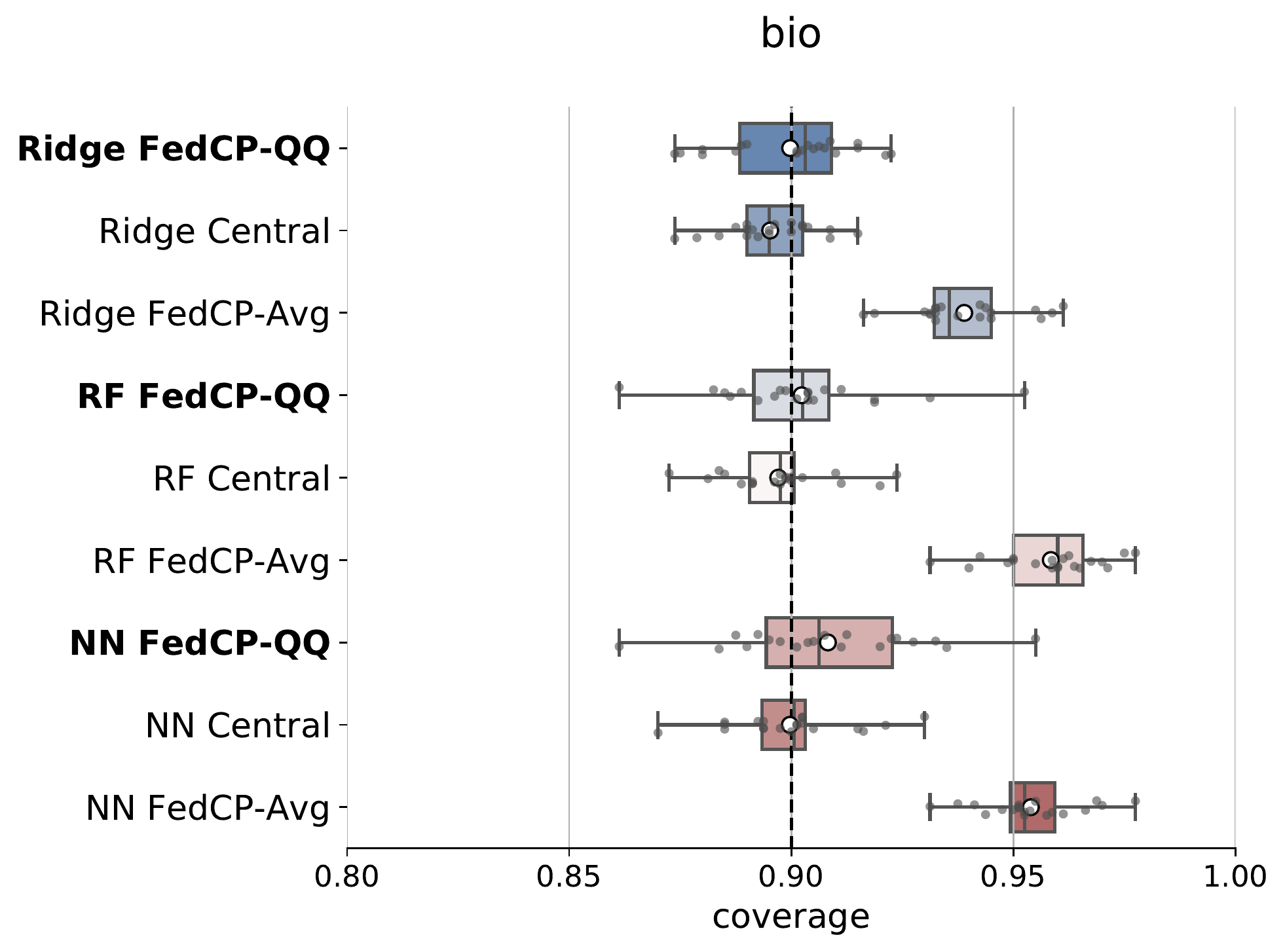}
	\includegraphics[width=.495\linewidth]{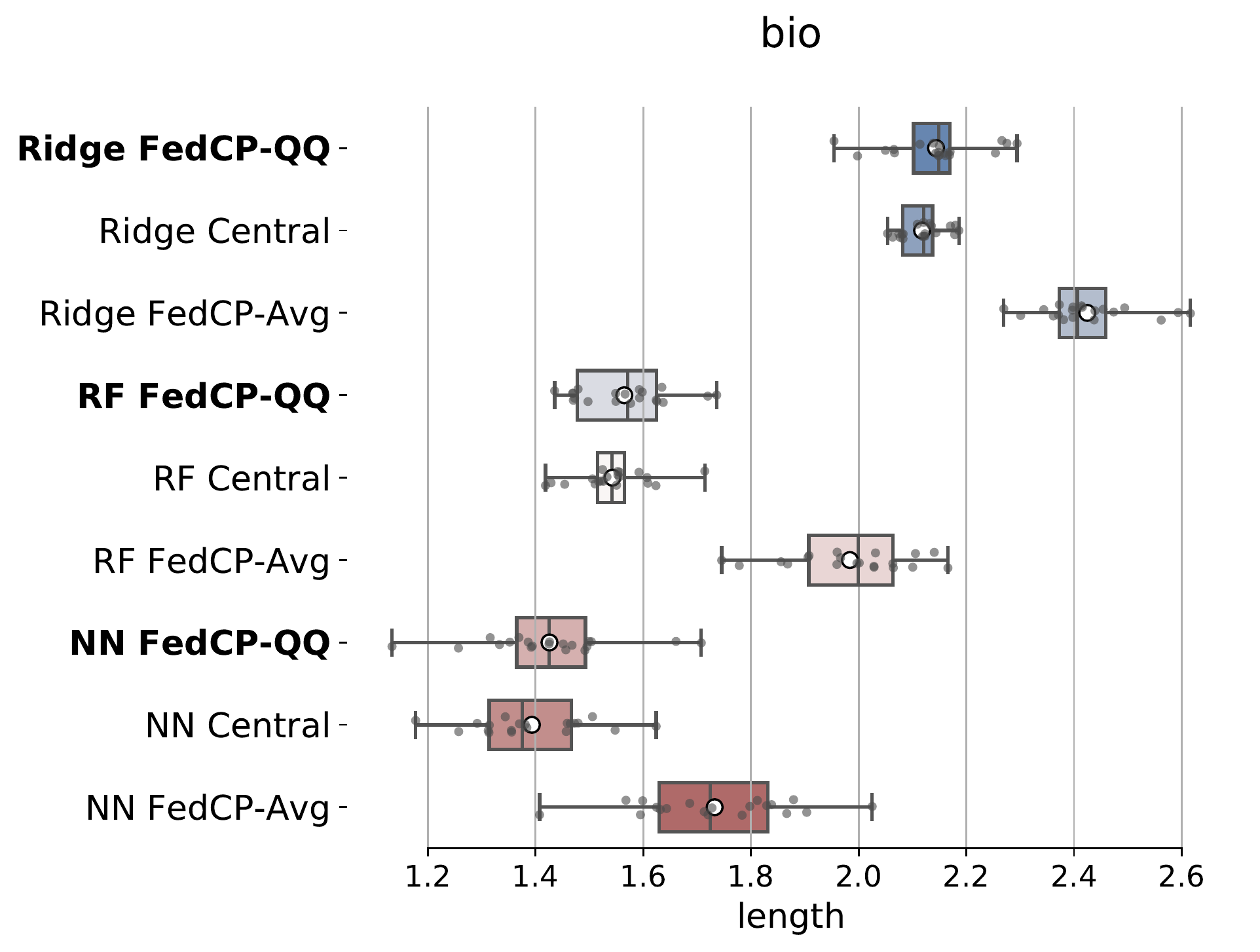}
	\vspace{-1em}
	\caption{Coverage (left) and average length (right) of prediction intervals for $20$ random training-calibration-test splits. The miscoverage is $\alpha=0.1$, and the calibration set is split into $m=100$ disjoint subsets of equal size $n=10$. The white circle represents the mean and the name of the data set is located at the top of each plot.}
	\label{fig:real_xp_bio}
\end{figure}
\begin{figure}[H]
	\centering
	\includegraphics[width=.495\linewidth]{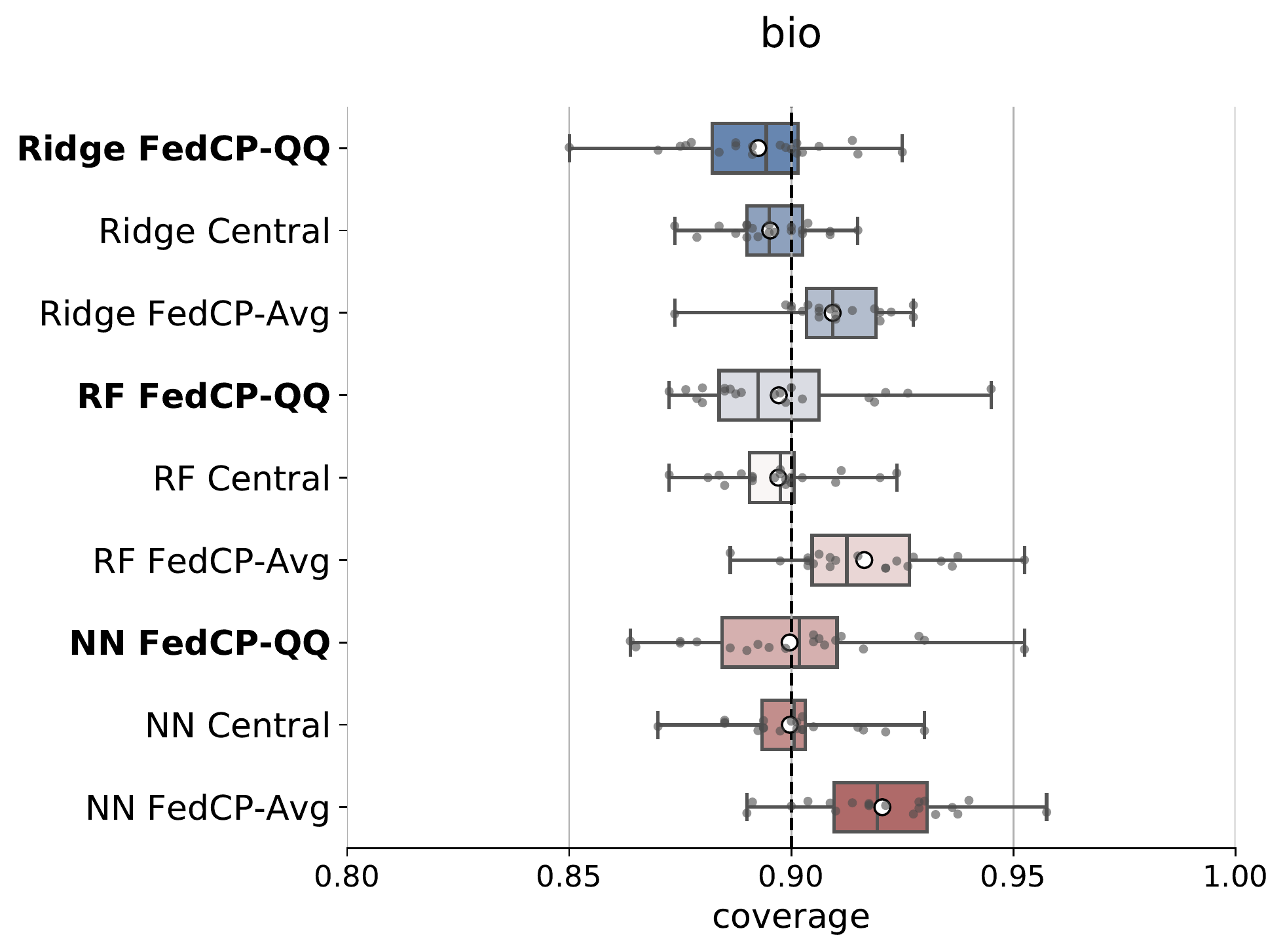}
	\includegraphics[width=.495\linewidth]{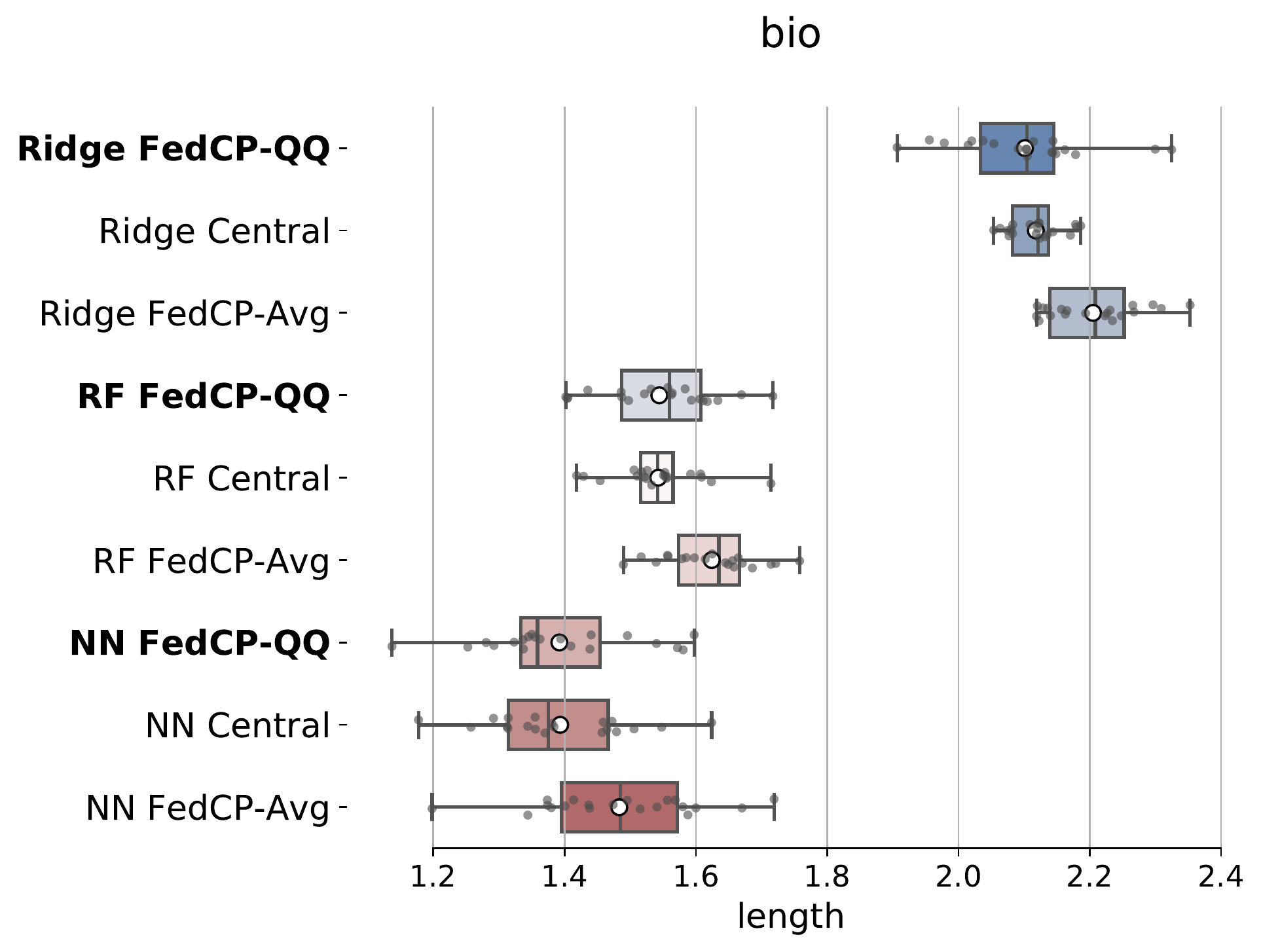}
	\vspace{-1em}
	\caption{Coverage (left) and average length (right) of prediction intervals for $20$ random training-calibration-test splits. The miscoverage is $\alpha=0.1$, and the calibration set is split into $m=10$ disjoint subsets of equal size $n=100$. The white circle represents the mean and the name of the data set is located at the top of each plot.}
	\label{fig:real_xp_bio_low_m}
\end{figure}

\begin{figure}[H]
	\centering
	\includegraphics[width=.495\linewidth]{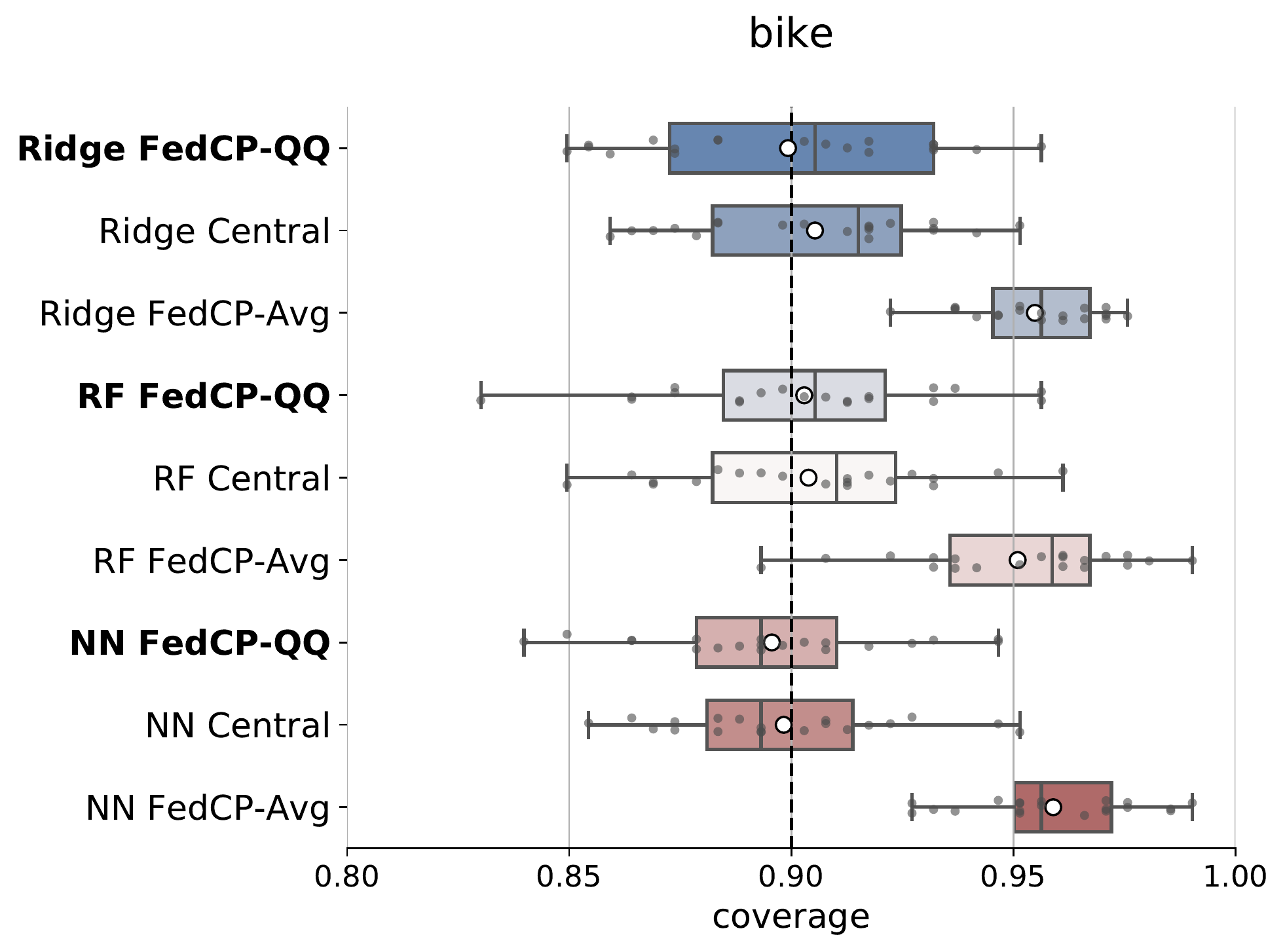}
	\includegraphics[width=.495\linewidth]{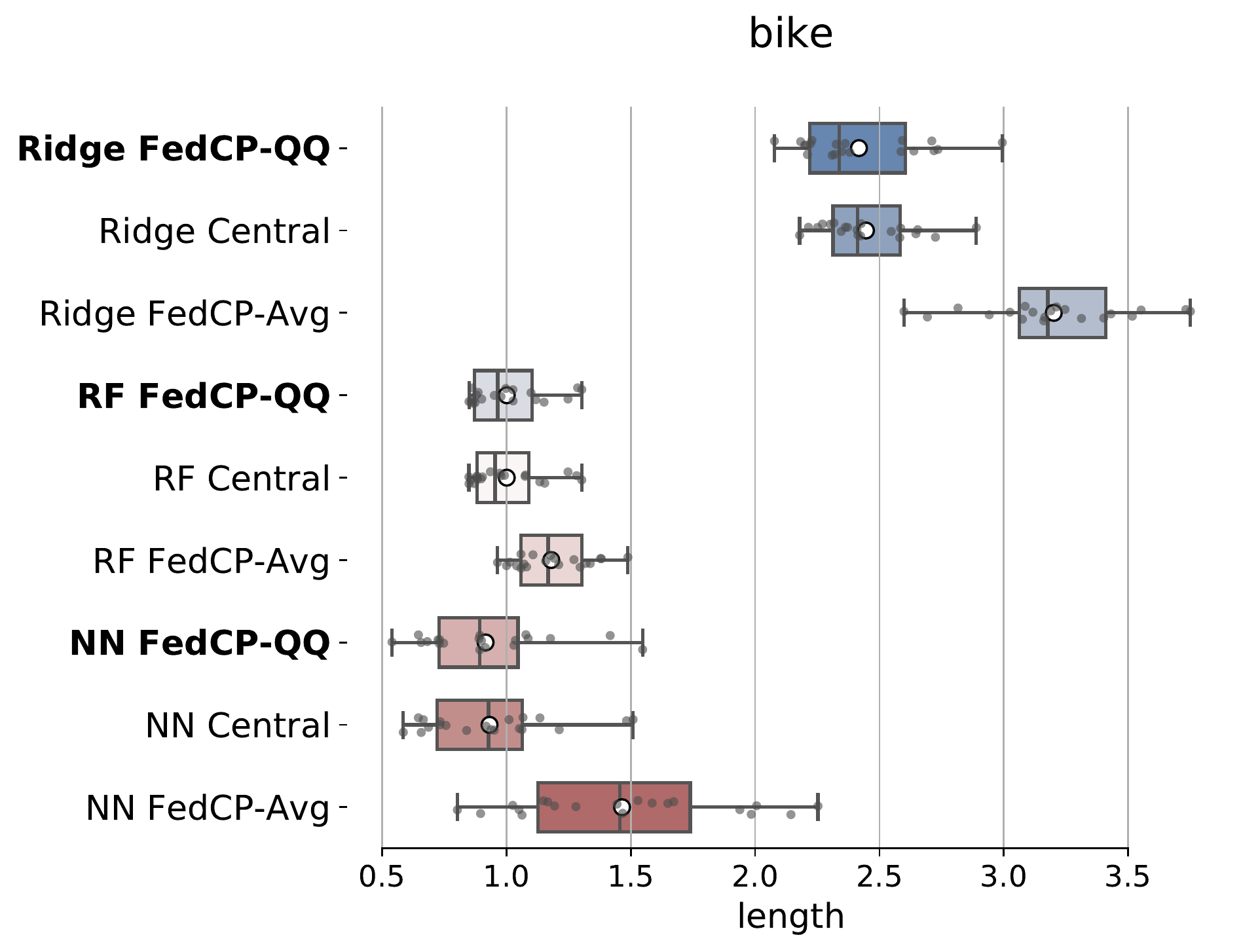}
	\vspace{-1em}
	\caption{Same as Figure \ref{fig:real_xp_bio} (see its caption) 
 with $m=100$ and $n=10$.}
	\label{fig:real_xp_bike}
\end{figure}

\begin{figure}[H]
	\centering
	\includegraphics[width=.495\linewidth]{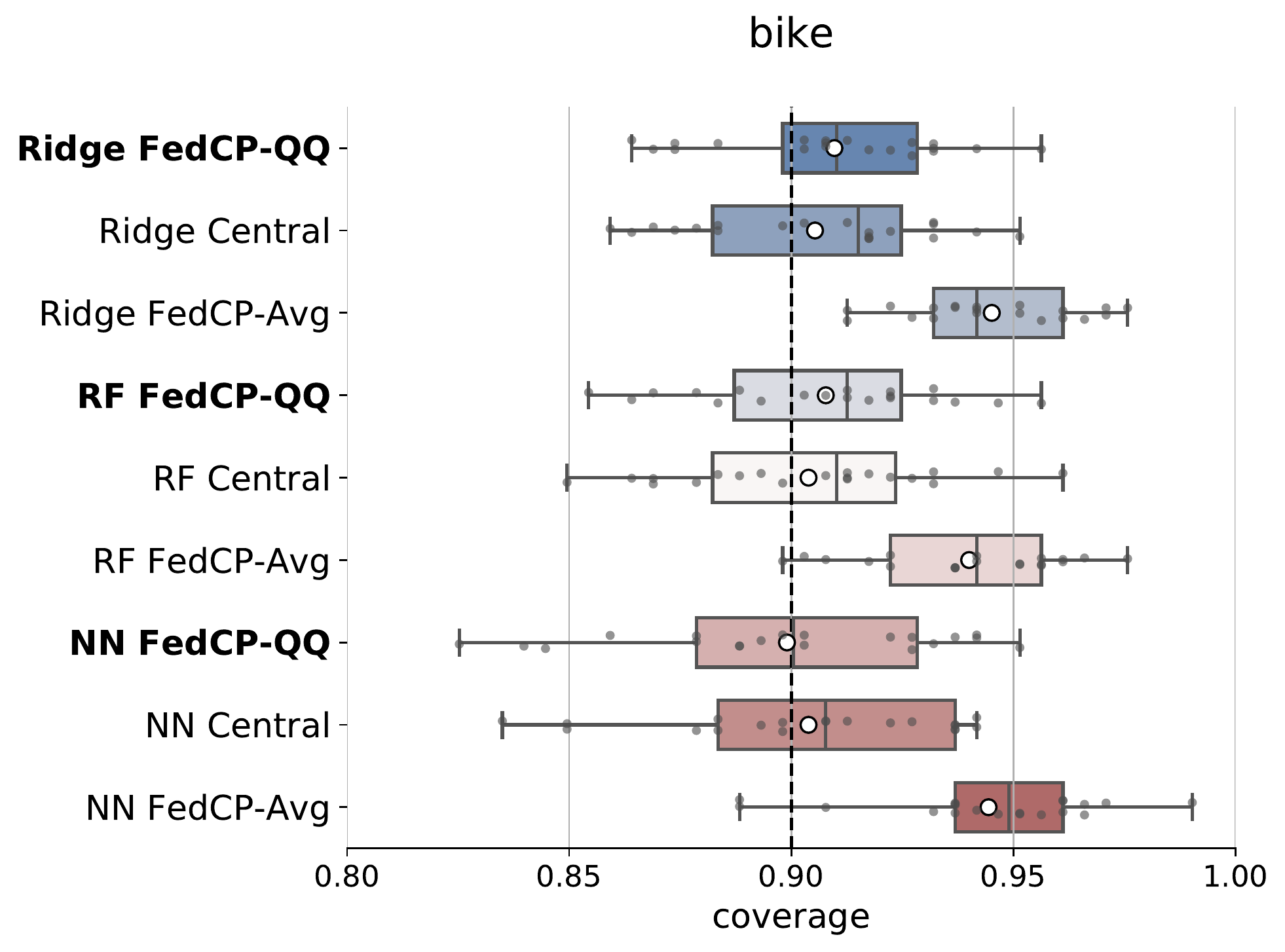}
	\includegraphics[width=.495\linewidth]{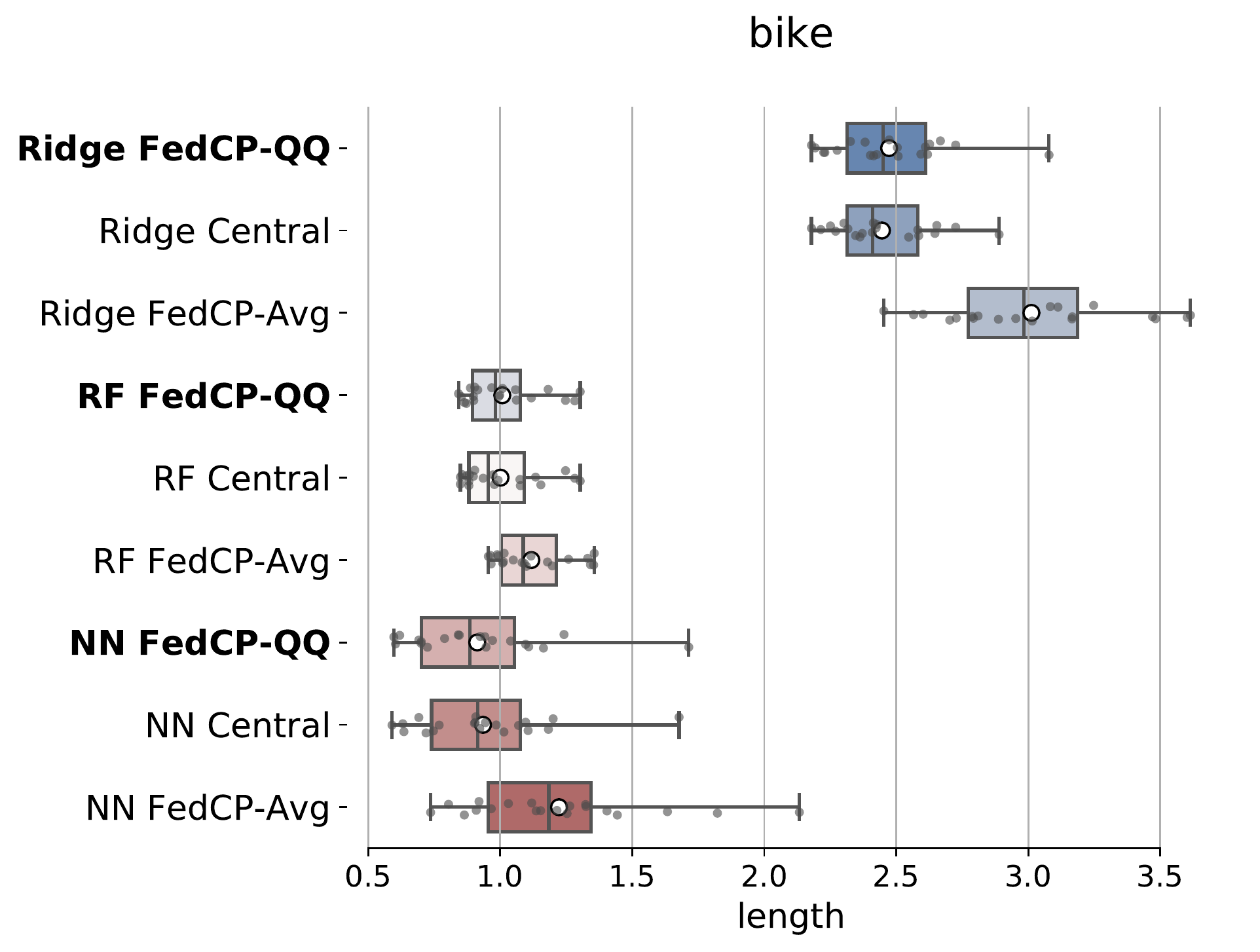}
	\vspace{-1em}
	\caption{Same as Figure \ref{fig:real_xp_bio_low_m} (see its caption) with $m=10$ and $n=100$.}
	\label{fig:real_xp_bike_low_m}
\end{figure}

\begin{figure}[H]
	\centering
	\includegraphics[width=.495\linewidth]{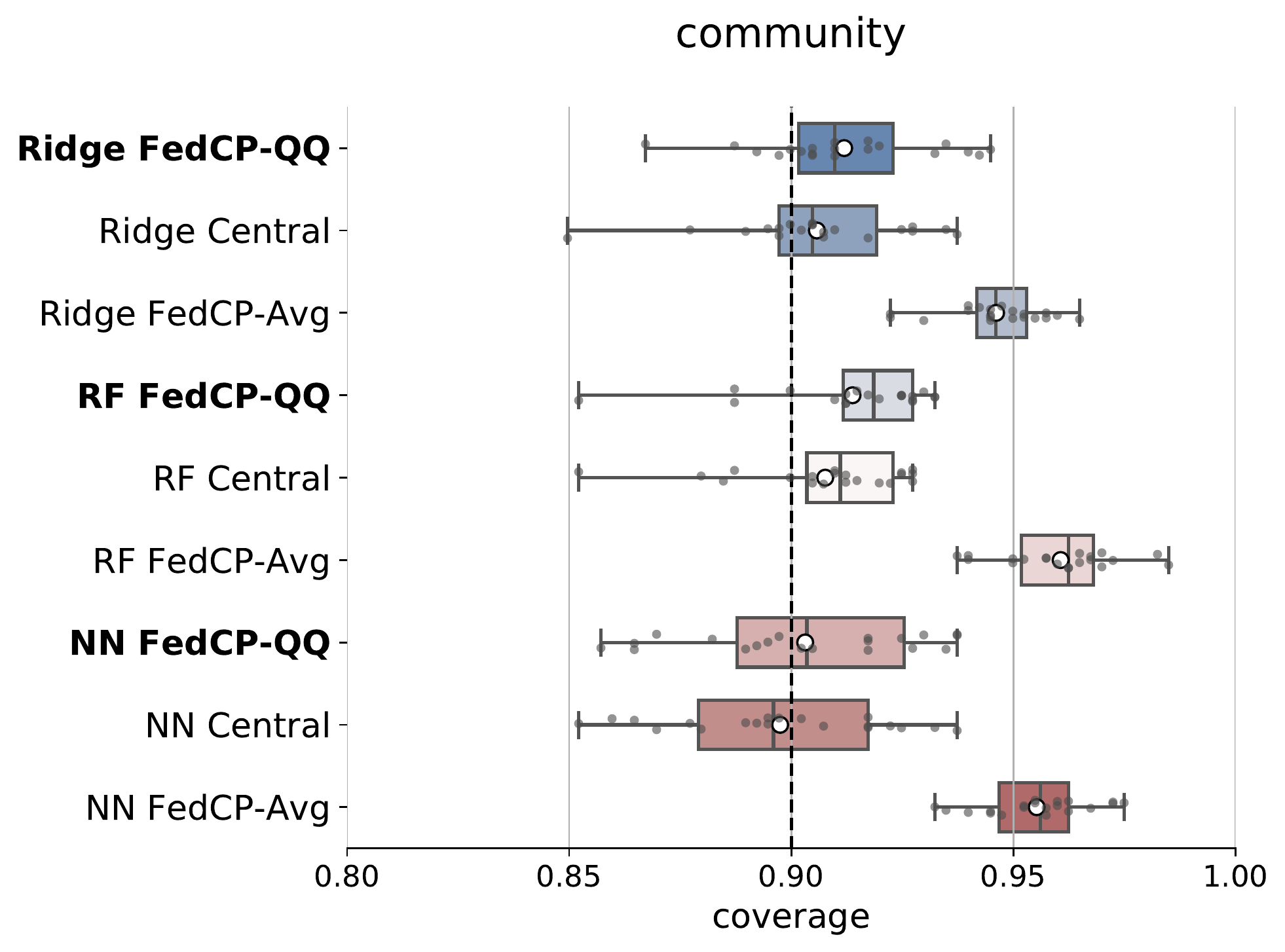}
	\includegraphics[width=.495\linewidth]{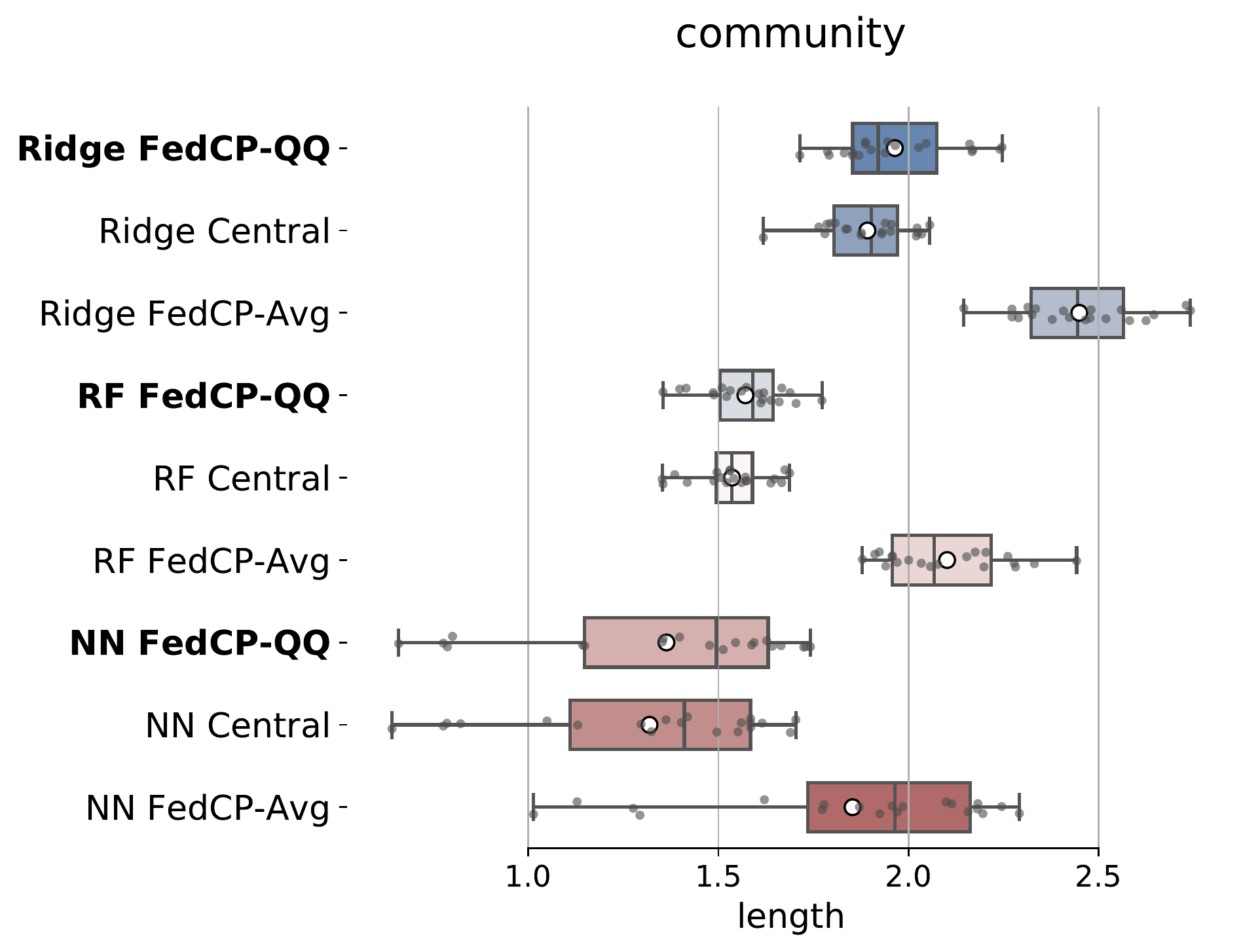}
	\vspace{-1em}
	\caption{Same as Figure \ref{fig:real_xp_bio} (see its caption) 
 with $m=80$ and $n=10$.}
	\label{fig:real_xp_community}
\end{figure}

\begin{figure}[H]
	\centering
	\includegraphics[width=.495\linewidth]{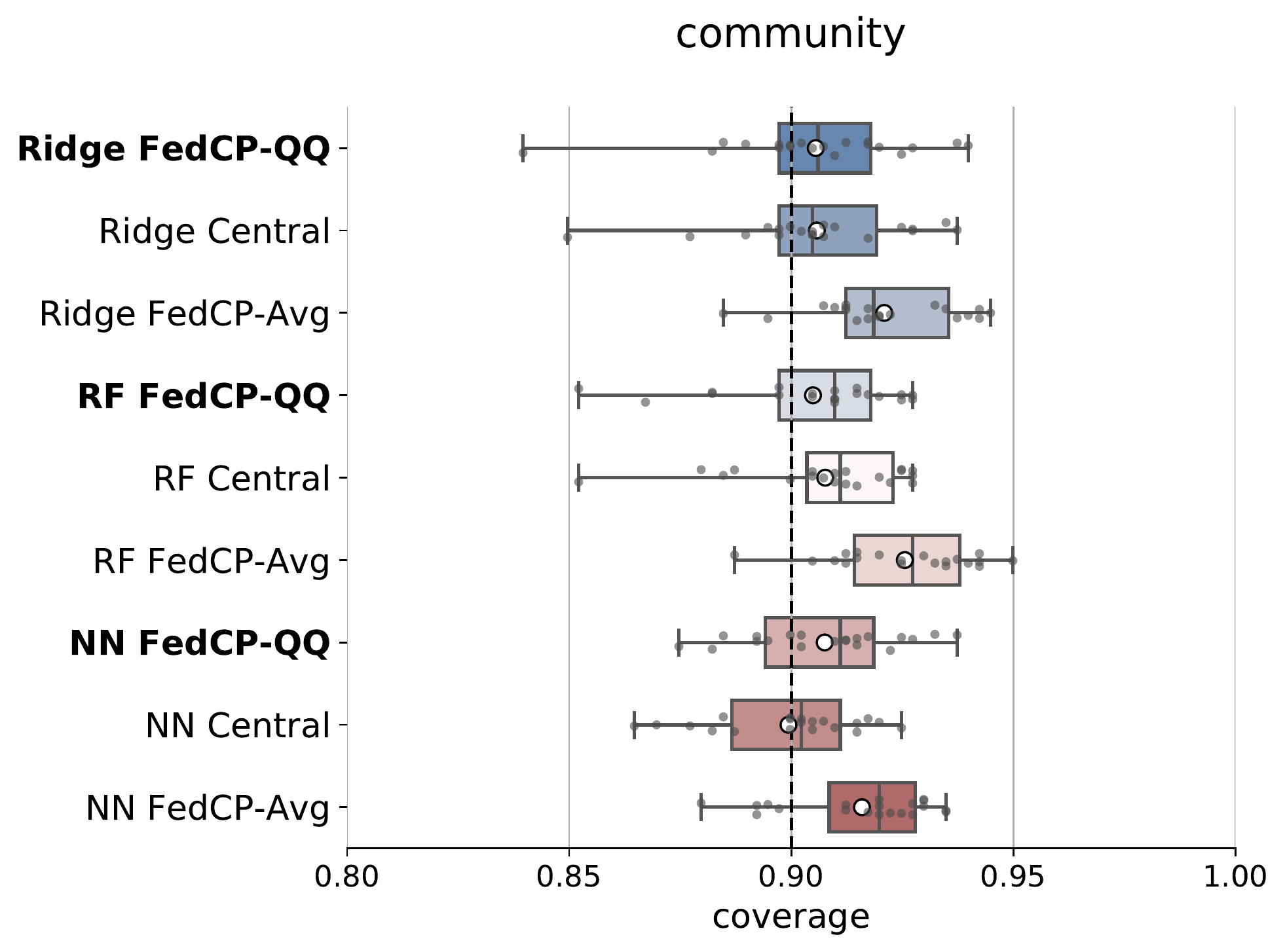}
	\includegraphics[width=.495\linewidth]{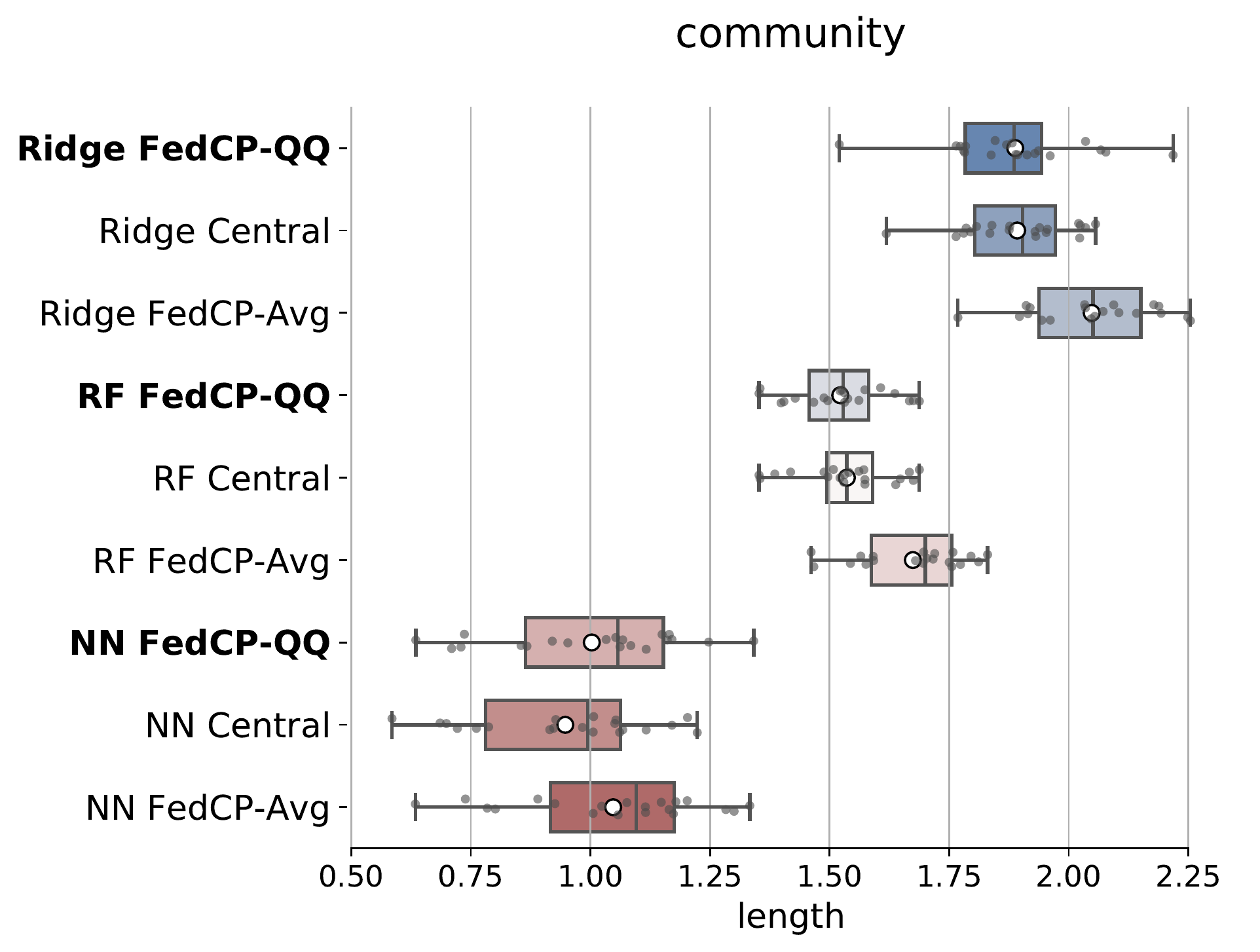}
	\vspace{-1em}
	\caption{Same as Figure \ref{fig:real_xp_bio_low_m} (see its caption) with $m=10$ and $n=80$.}
	\label{fig:real_xp_community_low_m}
\end{figure}

\begin{figure}[H]
	\centering
	\includegraphics[width=.495\linewidth]{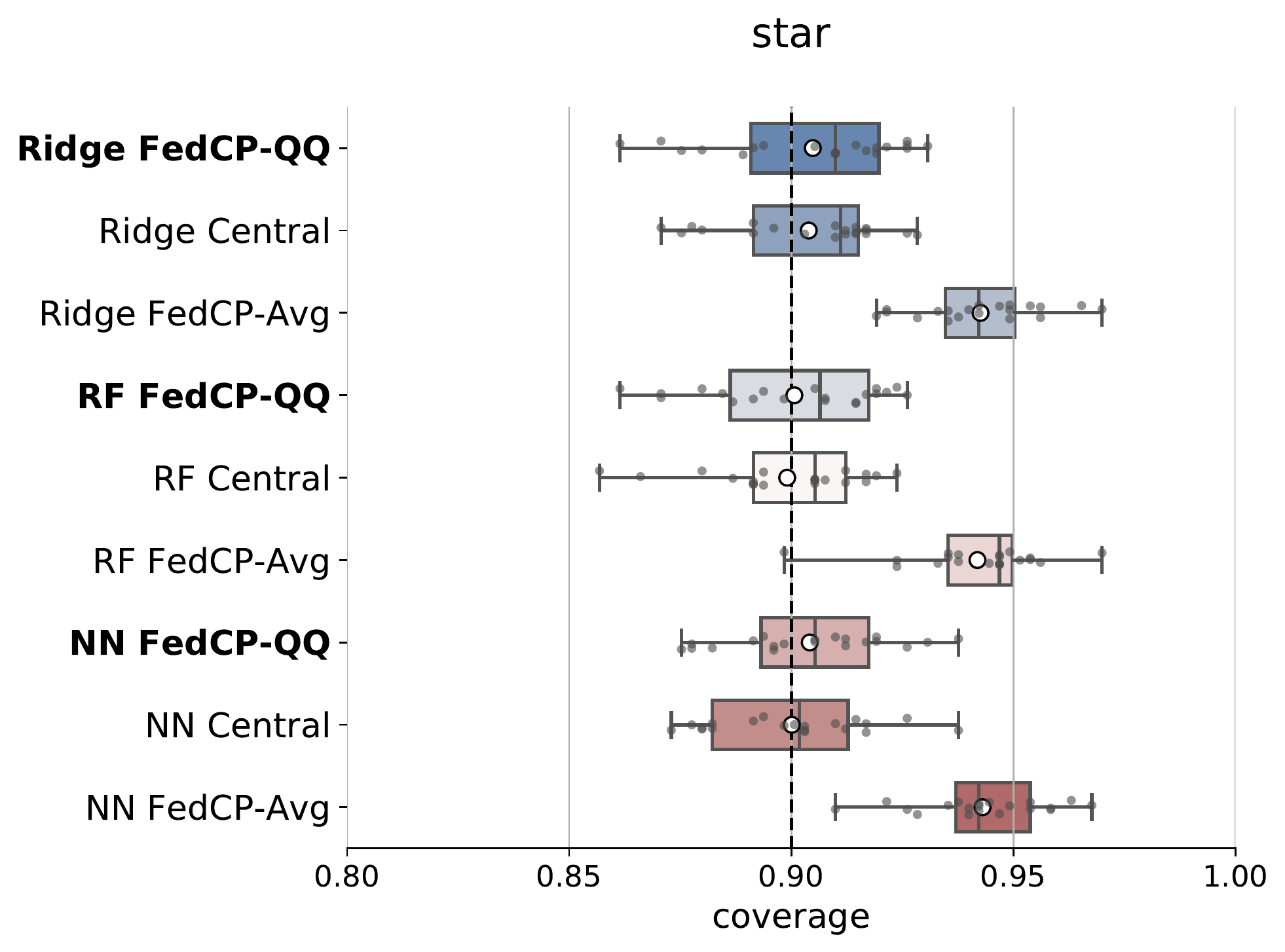}
	\includegraphics[width=.495\linewidth]{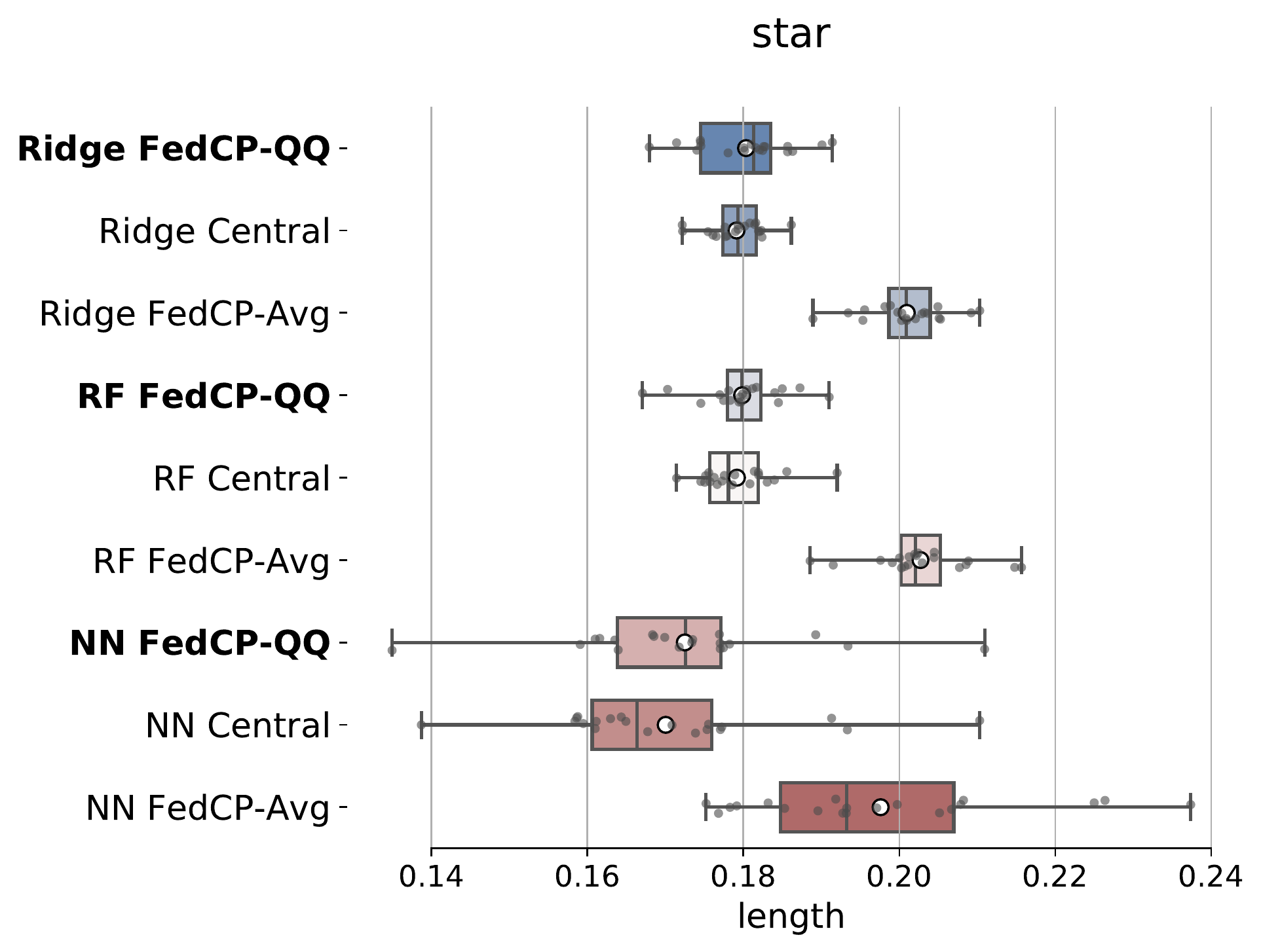}
	\vspace{-1em}
	\caption{Same as Figure \ref{fig:real_xp_bio} (see its caption) 
 with $m=80$ and $n=10$.}
	\label{fig:real_xp_star}
\end{figure}

\begin{figure}[H]
	\centering
	\includegraphics[width=.495\linewidth]{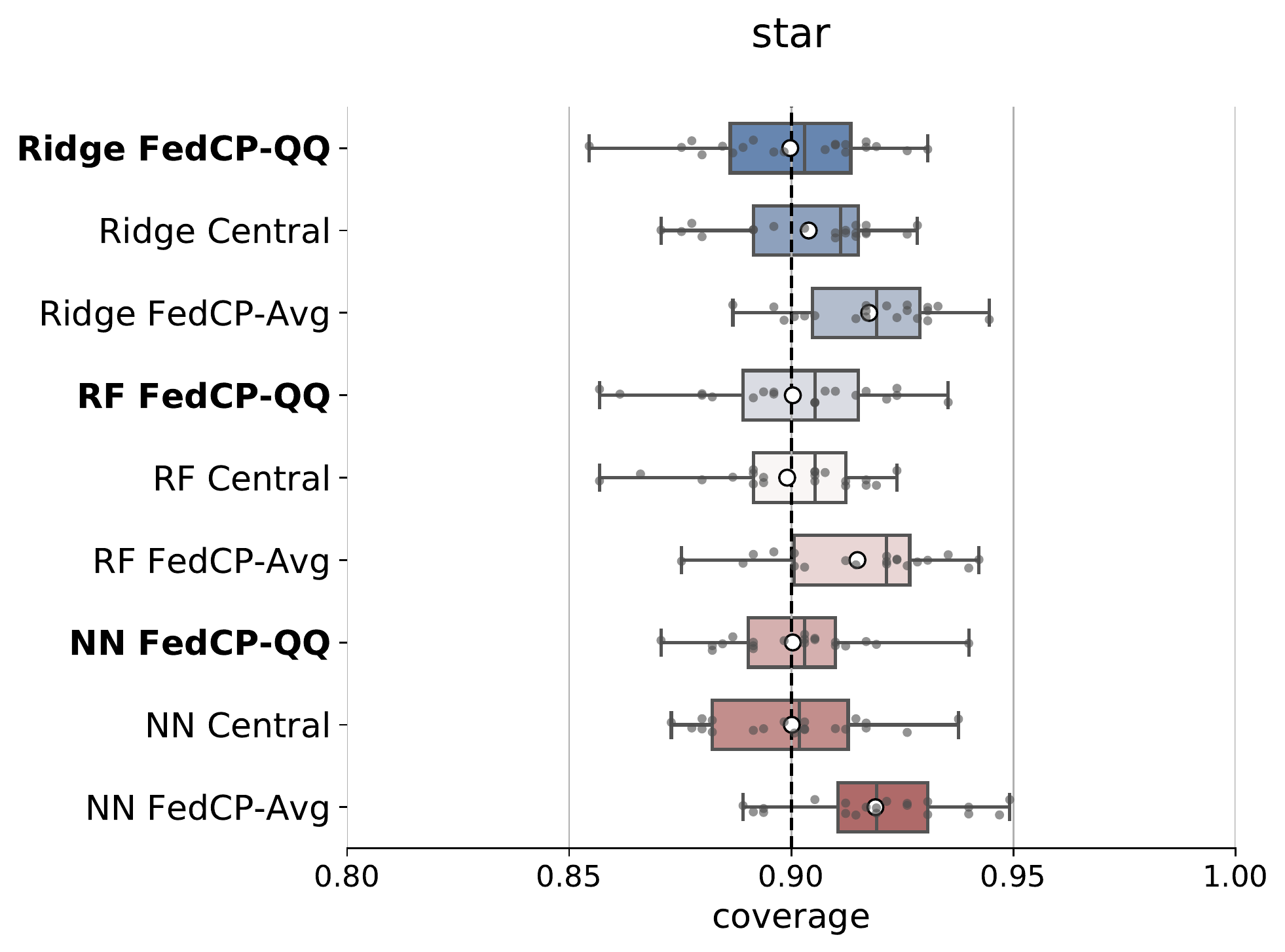}
	\includegraphics[width=.495\linewidth]{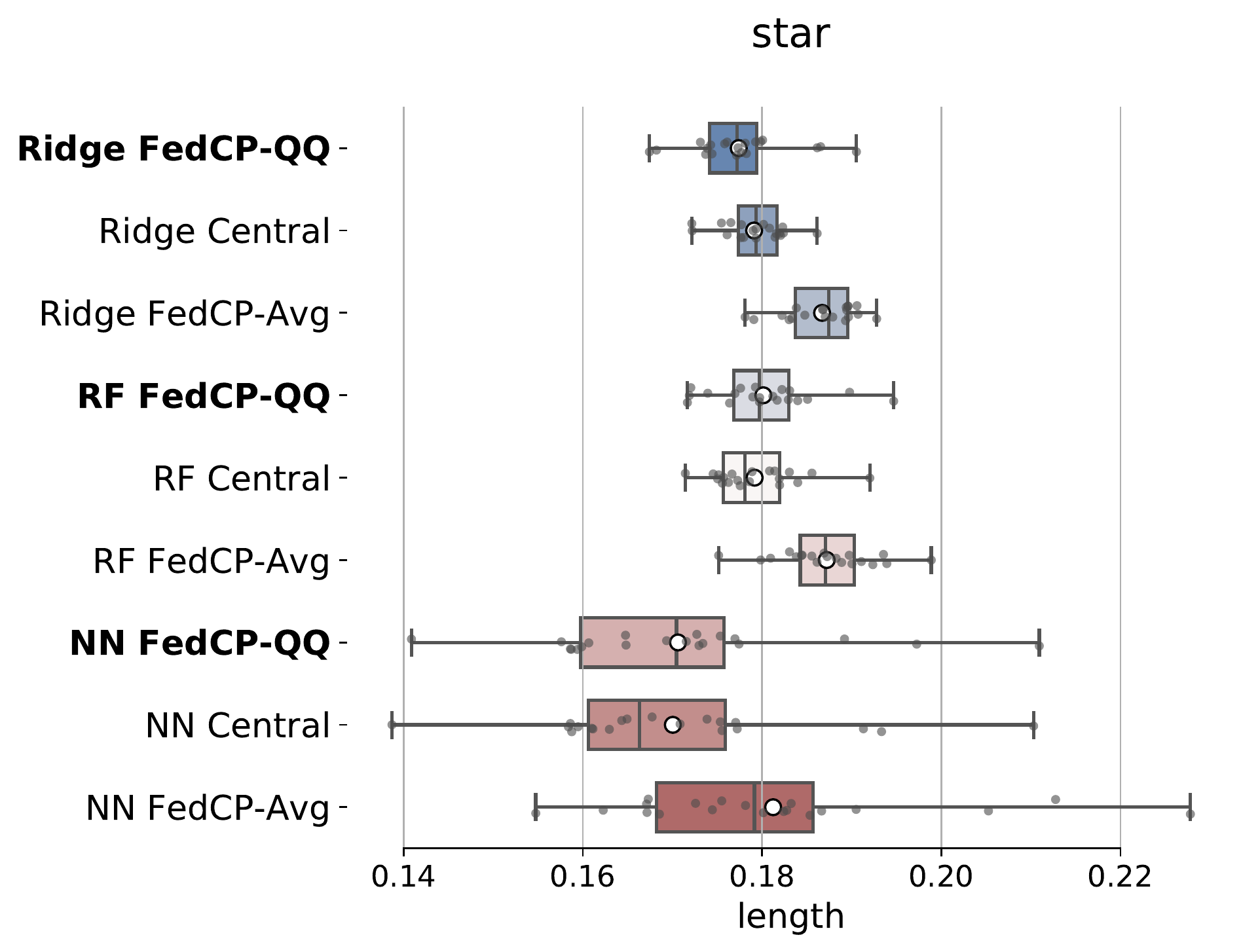}
	\vspace{-1em}
	\caption{Same as Figure \ref{fig:real_xp_bio_low_m} (see its caption) with $m=10$ and $n=80$.}
	\label{fig:real_xp_star_low_m}
\end{figure}

\begin{figure}[H]
	\centering
	\includegraphics[width=.495\linewidth]{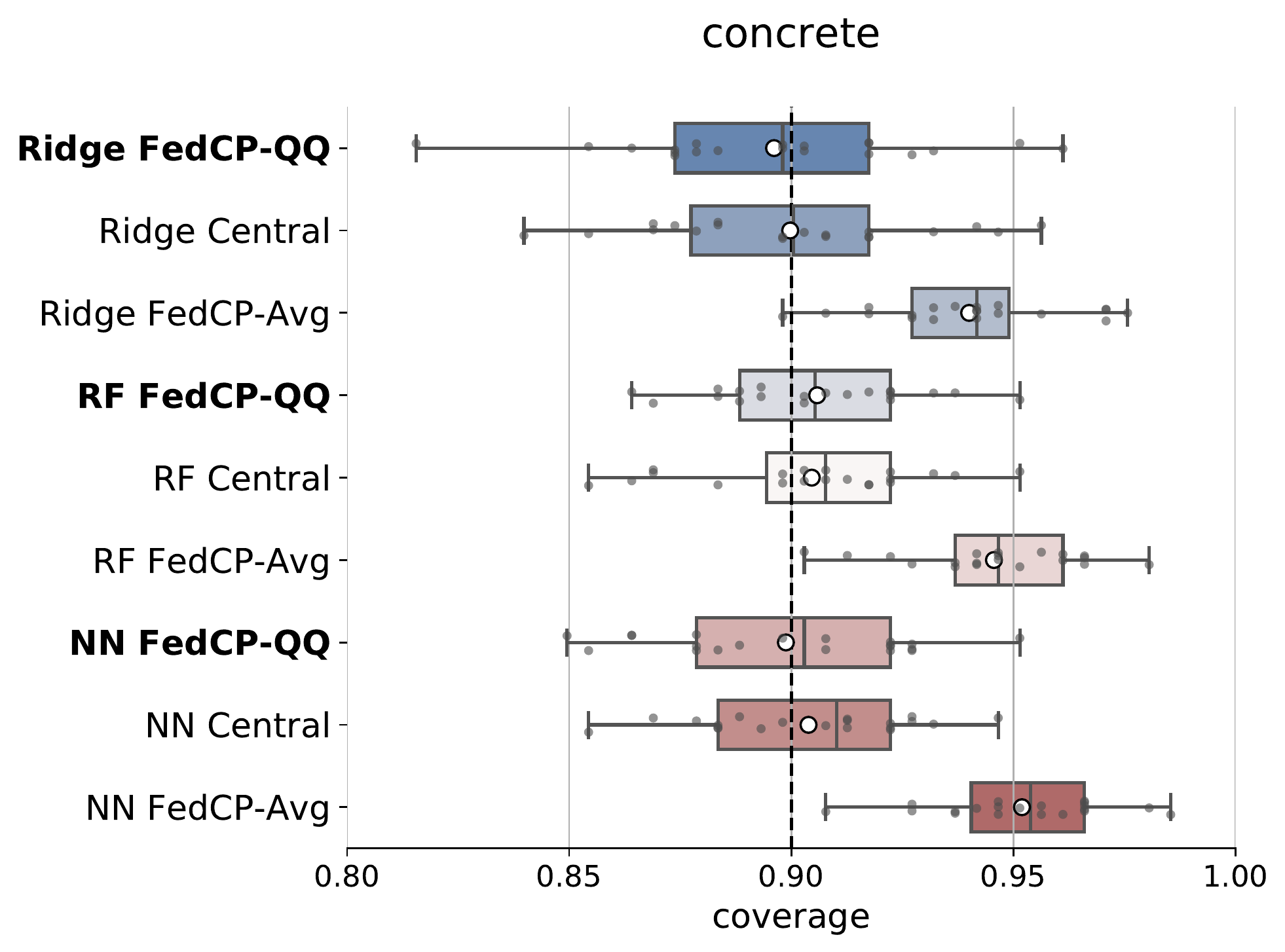}
	\includegraphics[width=.495\linewidth]{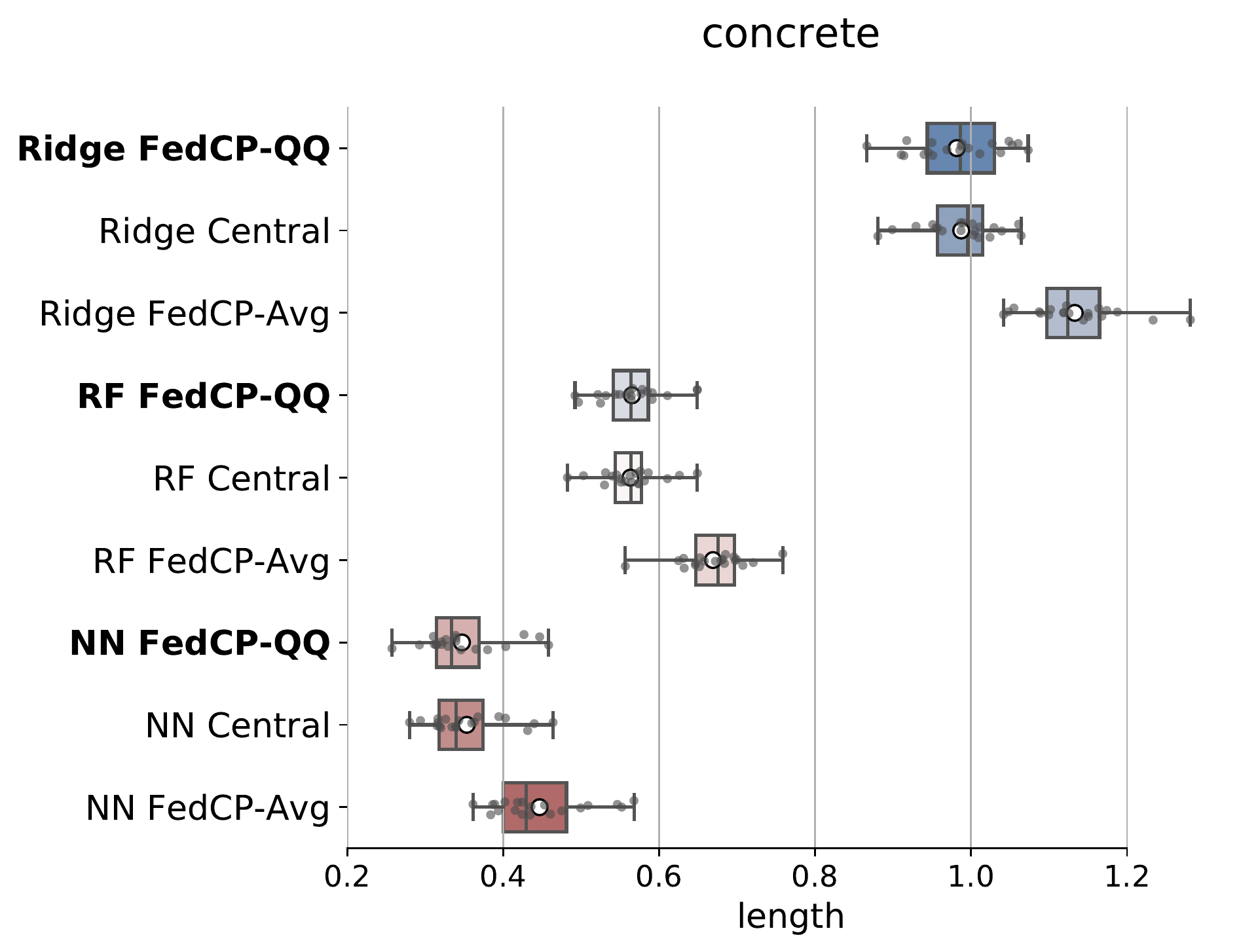}
	\vspace{-1em}
	\caption{Same as Figure \ref{fig:real_xp_bio} (see its caption) 
 with $m=40$ and $n=10$.}
	\label{fig:real_xp_concrete}
\end{figure}

\begin{figure}[H]
	\centering
	\includegraphics[width=.495\linewidth]{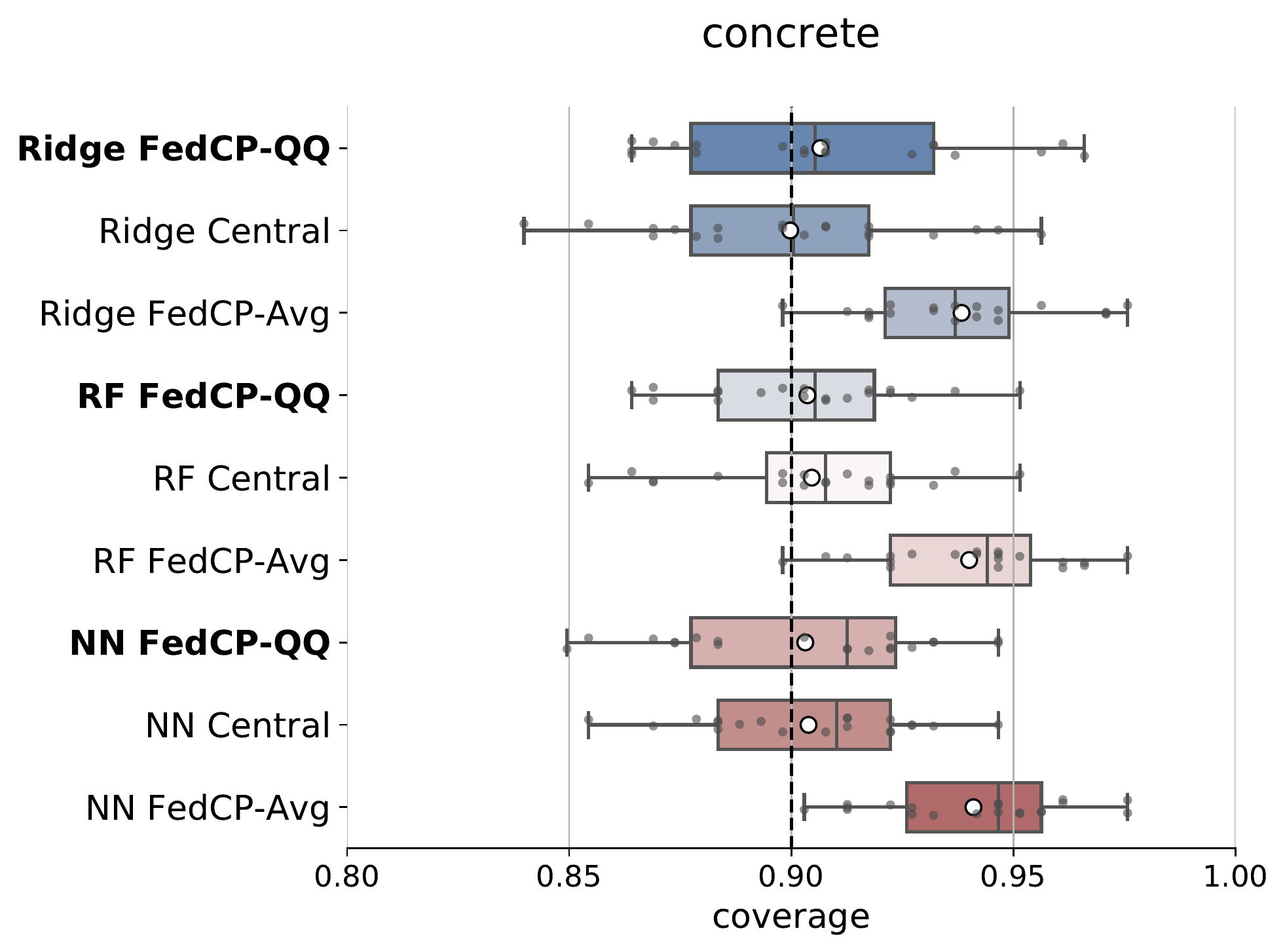}
	\includegraphics[width=.495\linewidth]{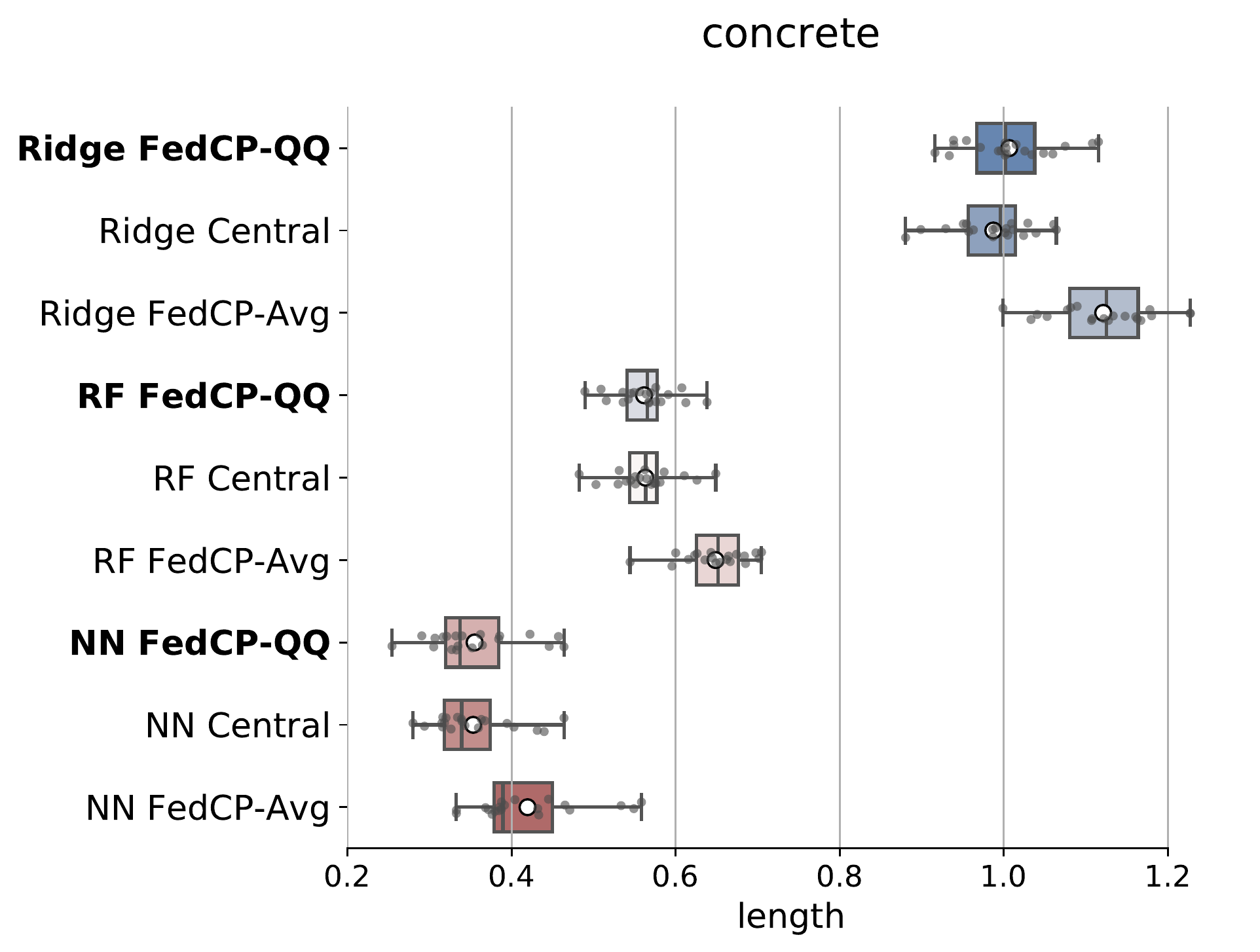}
	\vspace{-1em}
	\caption{Same as Figure \ref{fig:real_xp_bio_low_m} (see its caption) with $m=10$ and $n=40$.}
	\label{fig:real_xp_concrete_low_m}
\end{figure}

\subsection{Experiments with Differential Privacy}
\label{sec:private_xp}
For the sake of completeness, we also evaluate the quality of our private algorithm \methodDP~described in Section~\ref{sec:privacy} on the bio and bike data sets with $m=5$ and $n=200$. The predictor is a quantile RF, the number of bins is set to $B=100$, $S_{max}$ is fixed to the largest score (no clipping), and $\varepsilon = 10, 5, 1$.

Figure~\ref{fig:real_xp_bike_private} displays the empirical coverages obtained over $20$ different random splits. As expected from Theorem~\ref{thme:private}, we observe that on average the desired coverage at 0.90 is well satisfied. However, we also see that the coverages become quickly conservative as the privacy parameter $\varepsilon$ decreases. This suggests that the different corrections introduced to compensate for the extra randomness due to privacy may be overly strong. Finally, we note that these results would be significantly improved with the privacy amplification strategies discussed in Section \ref{sec:privacy}.

\begin{figure}[H]
	\centering
	\includegraphics[width=.495\linewidth]{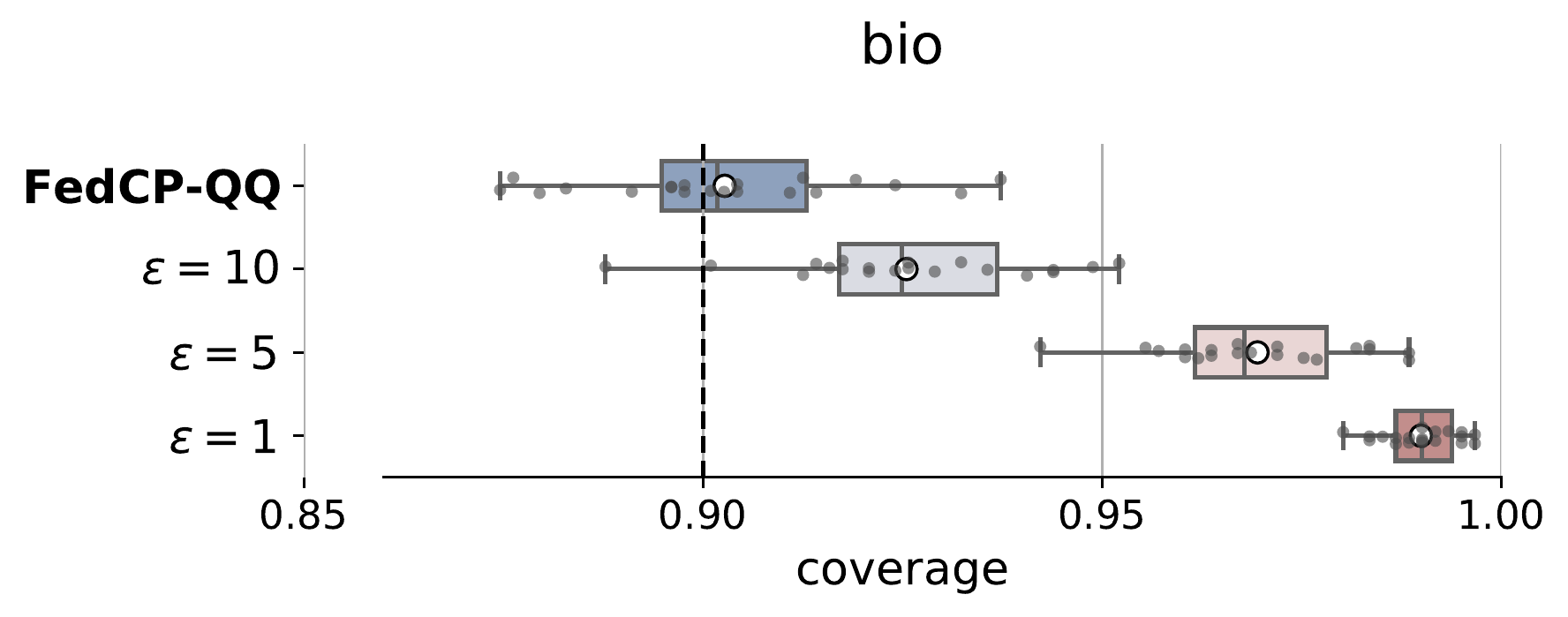}
	\includegraphics[width=.495\linewidth]{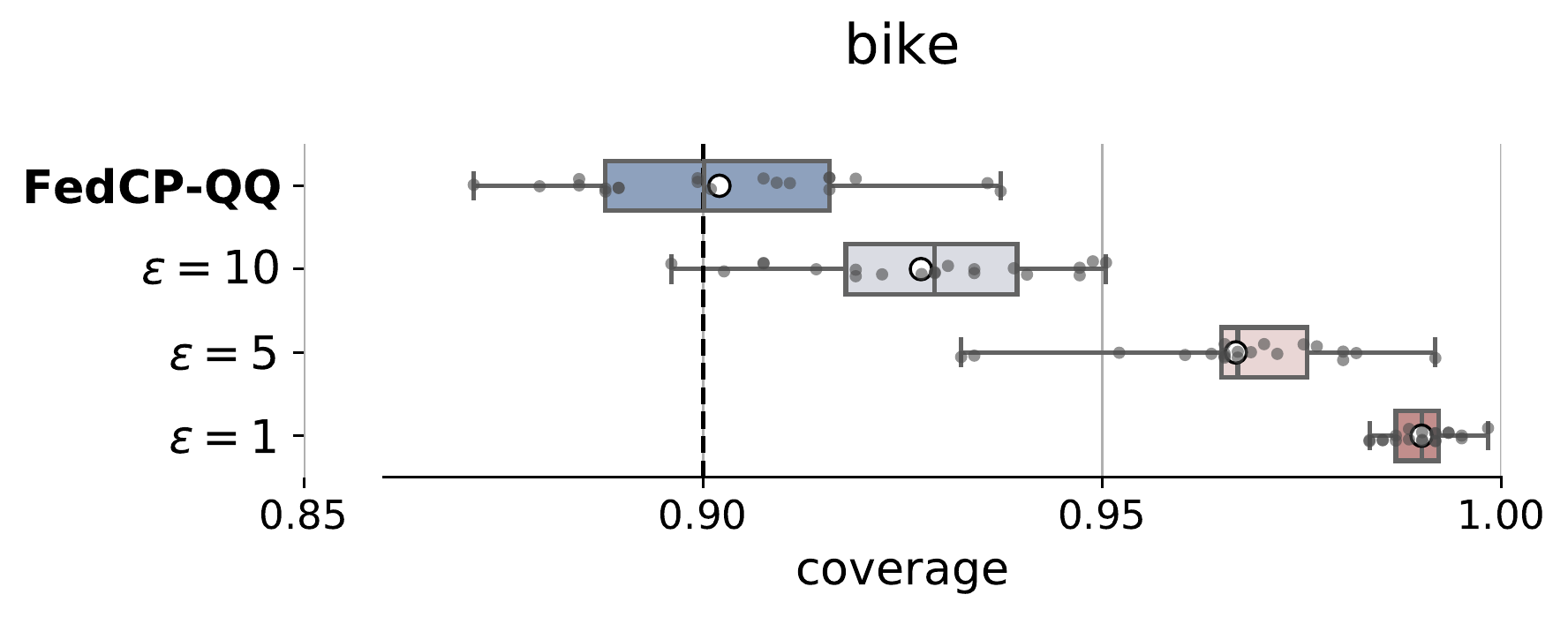}
	\caption{Empirical coverages of prediction intervals ($\alpha= 0.1$) constructed by \method~and its private version \methodDP~for $\varepsilon=10, 5, 1$. On top, coverages for the bio data set, and, on the bottom for the bike data sets. The white circle represents the mean.}
	\label{fig:real_xp_bike_private}
\end{figure}

\section{Proofs}

\subsection{Proof of Theorem \ref{them:main}}
\label{app:proof-them:main}
The proof of our results heavily relies on order statistics. We refer to \citet{david2004order} for an in-depth presentation. We begin by recalling the following important result.

\begin{lemma}
\label{lem.cdf-densite.stat-ordre}
Let $X_1, \ldots, X_n$ be some i.i.d. sample drawn from a continuous distribution with c.d.f. $F_X$ and density $f_X$. If we denote by $X_{(1)} \leq \cdots \leq X_{(n)}$ the corresponding ordered sample, for every $\ell \in \llbracket 1, n \rrbracket$, the c.d.f. and density of $X_{(\ell)}$ are respectively given by
\begin{align*}
    F_{X_{(\ell)}}(x) & = \sum^n_{i=\ell} \binom{n}{i} F_X(x)^i \bigl[ 1 - F_X(x) \bigr]^{n-i} \; ,\\
    f_{X_{(\ell)}}(x) & = \dfrac{n!}{(\ell-1)!(n-\ell)!} f_X(x) F_X(x)^{\ell-1} \bigl[ 1 - F_X(x) \bigr]^{n-\ell} \; .
\end{align*}
\end{lemma}


We can now prove Theorem \ref{them:main}.

First, remark that if, conditionally to $\fh$, $(X^{(1)}_1, Y^{(1)}_1), \ldots, (X^{(m)}_n, Y^{(m)}_n), (X, Y)$ are i.i.d., then, conditionally to $\fh$, the associated scores $S_1^{(1)} , \ldots , S_n^{(m)}, S$ are i.i.d. We denote by $F_S$ their c.d.f. (given $\fh$), and make the proof conditionally to $\fh$.

We know that $F^{-1}_S$ is non-decreasing and that if $U \sim U_{[0, 1]}$, $F^{-1}_S(U)$ has the same distribution as $S$ (given $\fh$). Therefore, if $U_1^{(1)} , \ldots , U_n^{(m)} , U$ are independent with a uniform distribution over $[0,1]$, and independent from the data, and if
\[
U_{(\ell,k)} \triangleq \Qh_{(k)} \left( \Qh_{(\ell)} \left( \{ U_i^{(1)} \, , \, i=1,\ldots, n \} \right)  , \ldots , \Qh_{(\ell)} \left( \{ U_i^{(m)} \, , \, i=1,\ldots, n \} \right)   \right)
\]
denotes the corresponding QQ estimator, then $F_S^{-1} (U_{(\ell,k)})$ has the same distribution as $\Qh_{(\ell, k)}$ (given $\fh$). We obtain that 
\begin{align}
\label{eq.pr.them:main.1}
    \IP\left(Y \in \Chat(X) \mid \fh \right) = \IP\left(S \leq \Qh_{(\ell, k)} \mid \fh \right) &= \IP\left(F^{-1}_S(U) \leq F^{-1}_S(U_{(\ell, k)}) \mid \fh \right) \geq \IP\left(U \leq U_{(\ell, k)} \mid \fh \right) \; . 
\end{align}
Furthermore, if $F_S$ is continuous, $F^{-1}_S$ is increasing, and 
\begin{equation} 
\IP\left(F^{-1}_S(U) \leq F^{-1}_S(U_{(\ell, k)}) \mid \fh \right) = \IP\left(U \leq U_{(\ell, k)} \mid \fh \right)
\, . 
 \label{eq.pr.them:main.2}
 \end{equation}
Therefore, it remains to treat the uniform case. 
By Lemma~\ref{lem.cdf-densite.stat-ordre}, we have
\begin{align*}
        F_{U_{(\ell, k)}}(t)
        &=\sum^m_{j=k} \binom{m}{j} F_{U_{(\ell)}}(t)^j \bigl[ 1 - F_{U_{(\ell)}}(t) \bigr]^{m-j} \\
        &= \sum^m_{j=k}\binom{m}{j} \left[\sum^n_{i=\ell} \binom{n}{i} t^i (1 - t)^{n-i}\right]^j \left[1 - \sum^n_{i=\ell} \binom{n}{i} t^i (1 - t)^{n-i}\right]^{m-j} \\
        &= \sum^m_{j=k}\binom{m}{j} \left[\sum^n_{i=\ell} \binom{n}{i} t^i (1 - t)^{n-i}\right]^j \left[\sum^{\ell-1}_{i=0} \binom{n}{i} t^i (1 - t)^{n-i}\right]^{m-j}
        \\
        \text{since} \qquad 
        1 
        &= \sum^n_{i=0} \binom{n}{i} t^i (1 - t)^{n-i} 
        = \sum^{\ell-1}_{i=0} \binom{n}{i} t^i (1 - t)^{n-i}
        + \sum^n_{i=\ell} \binom{n}{i} t^i (1 - t)^{n-i} \, ,
\end{align*}
hence we get that 
\begin{align*} 
F_{U_{(\ell, k)}}(t)
&= \sum^m_{j=k}\binom{m}{j} \sum^{n}_{i_1=\ell} \cdots \sum^{n}_{i_j=\ell} \sum^{\ell-1}_{i_{j+1}=0} \cdots \sum^{\ell-1}_{i_m=0} \binom{n}{i_1}\cdots\binom{n}{i_m} t^{i_1 + \cdots + i_m} (1-t)^{mn-(i_1 + \cdots + i_m)} 
\; . 
\end{align*}
As a consequence, we obtain
\begin{align*}
    \IP\left(U_{(\ell, k)} \leq U \right) &= \IE \bigl[ F_{U_{(\ell,k)}} (U) \bigr] \\
    &= \int^{1}_{0} F_{U_{(\ell, k)}}(t)  \mathrm{d}t \\
    &= \int^{1}_{0} \sum^m_{j=k}\binom{m}{j} \sum^{n}_{i_1=\ell} \cdots \sum^{n}_{i_j=\ell} \sum^{\ell-1}_{i_{j+1}=0} \cdots \sum^{\ell-1}_{i_m=0} \binom{n}{i_1}\cdots\binom{n}{i_m} t^{i_1 + \cdots + i_m} (1-t)^{mn-(i_1 + \cdots + i_m)}  \mathrm{d}t  \\
    &= \sum^{m}_{j=k} \binom{m}{j} \sum^{n}_{i_1=\ell} \cdots \sum^{n}_{i_j=\ell} \sum^{\ell-1}_{i_{j+1}=0} \cdots \sum^{\ell-1}_{i_m=0} \binom{n}{i_1} \cdots \binom{n}{i_m} \mathrm{B}\left(i_1 + \cdots + i_m + 1, m n - (i_1 + \cdots + i_m) + 1\right) \; ,
\end{align*}
where 
\[ 
\mathrm{B} : (a,b) \in (0,+\infty)^2 \mapsto \int_0^1 t^{a-1} (1-t)^{b-1} \mathrm{d}t
\]
denotes the Beta function. 
The identity $\dbinom{a}{b} = \dfrac{1}{(a+1)\mathrm{B}(b+1, a-b+1)}$, with $a=mn$ and $b=(i_1 + \cdots + i_m)$, implies that $\IP\left(U_{(\ell, k)} \leq U \right) = 1 - M_{n,k}$, hence 
\begin{equation}
\label{eq.pr.them:main.3}
    \IP\left(U \leq U_{(\ell, k)} \right) = M_{n,k} \; . 
\end{equation}
By Eq.~\eqref{eq.pr.them:main.1}, we obtain that 
\[ 
\IP \left( Y \in \Chat_{\ell,k}(X) \,\vert\, \fh \right) \geq M_{n,k} 
\]
almost surely, hence Eq.~\eqref{eq:main_equa} by integrating this inequality. When $F_S$ is continuous, Eq. \eqref{eq.pr.them:main.2} and~\eqref{eq.pr.them:main.3} show that
\[ 
\IP \left( Y \in \Chat_{\ell,k}(X) \,\vert\, \fh \right) = M_{n,k}\;,
\]
hence the result. 
\qed


\subsection{Proof of Theorem~\ref{them:cond}}
\label{app.pr.them:cond}
First, let us remark that $\sum_{j=1}^{m} \sum_{i=1}^{n} \One{S^{(j)}_i \leq \Qh_{(\ell, k)}}$ is almost surely greater or equal to $\ell \cdot k$ by definition of $\Qh_{(\ell, k)}$.
Now, following the proof of \citet[Theorem~1]{bian2022training}, 
by definition of the \method~method, we have
\begin{align*}
\bigl\{ Y \in \Chat_{k,\ell} (X) \bigr\}  
&= \left\{ S \leq \widehat{Q}_{(\ell, k)} \right\} \\
&\supseteq \left\{ \sum^m_{j = 1}\sum_{i = 1}^{n} \One{S^{(j)}_i < S } < \sum_{j=1}^{m} \sum_{i=1}^{n} \One{S^{(j)}_i \leq \Qh_{(\ell, k)}} \right\} \\
&\supseteq\left\{ \sum^m_{j = 1}\sum_{i = 1}^{n} \One{S^{(j)}_i < S } < \ell \cdot k\right\} \\
&=\left\{ \sum^m_{j = 1}\sum_{i = 1}^{n} \One{S^{(j)}_i \geq S } \geq mn - \ell \cdot k\right\} \\
&= \left\{ \bar{F}_{mn}(S) \geq \dfrac{mn - \ell \cdot k}{mn} \right\} \; ,
\end{align*}
where $\bar{F}_{mn}(S)$ is the right-tail empirical c.d.f of the $\{ S_i^{(j)} \}_{i,j=1}^{n,m}$ at $S$. Note that this is a random variable in both the data set and $S$. We now have
\begin{align*}
\alpha_P(\mathcal{D}_{mn}) 
&= \IP\left(Y \notin \Chat_{\ell,k}(X) \,\big\vert\, \fh , \mathcal{D}_{mn} \right) 
\\
&\leq \IP \left( \bar{F}_{mn}(S) < 1 - \frac{\ell \cdot k}{mn} \,\bigg\vert\, \fh , \mathcal{D}_{mn} \right) \\
&= \IP \left( \bar{F}_{mn}(S) + \bar{F}_S(S) - \bar{F}_S(S) < 1 - \frac{\ell \cdot k}{mn} \,\bigg\vert\, \fh , \mathcal{D}_{mn} \right) \\
& \leq \IP \left( \bar{F}_S(S) \leq 1 - \frac{\ell \cdot k}{mn} + \sup_{s\in\IR} \left\{ \bar{F}_S(s) - \bar{F}_{mn}(s)\right\} \,\bigg\vert\, \fh ,  \mathcal{D}_{mn} \right) \; .
\end{align*}
Fixing any $\Delta>0$, let us consider the event
\begin{equation*}
\left\{\sup_{s\in\IR} \bigl\{ \bar{F}_S(s) - \bar{F}_{mn}(s)\bigr\} \leq \Delta \right\}\; .
\end{equation*}
Note that it depends of the data $\mathcal{D}_{mn}$. On this event, we have
\[\alpha_P(\mathcal{D}_{mn}) \leq \IP \left( \bar{F}_S(S) \leq 1 - \frac{\ell \cdot k}{mn} + \Delta \,\bigg\vert\, \fh ,  \mathcal{D}_{mn} \right) 
\leq 1 - \frac{\ell \cdot k}{mn} + \Delta  \]
since $\bar{F}_S(S)$ is a valid p-value \citep[Lemma~1]{bian2022training}. As a consequence,
\[ 
\IP\left( \alpha_P(\mathcal{D}_{mn}) > 1 - \frac{\ell \cdot k}{mn} + \Delta \right) 
\leq \IP\left(\sup_{s\in\IR} \bigl\{ \bar{F}_S(s) - \bar{F}_{mn}(s)\bigr\} > \Delta\right) 
\; .
\]
Applying the Dworetzky-Kiefer-Wolfowitz inequality \citep{dvoretzky1956asymptotic, massart1990tight}, the last term is upper-bounded by $\delta \in (0, 0.5]$ when we choose $\Delta =\sqrt{\frac{\log(1/\delta)}{2 mn}}$. Finally, for $\ell \cdot k \geq (1-\alpha) \cdot mn$, we have
\[
\IP\left( \alpha_P(\mathcal{D}_{mn}) \leq \alpha + \sqrt{\frac{\log(1/\delta)}{2 mn}} \right) 
\geq \IP\left( \alpha_P(\mathcal{D}_{mn}) \leq 1 - \frac{\ell \cdot k}{mn} + \sqrt{\frac{\log(1/\delta)}{2 mn}} \right) \geq 1-\delta \; .
\]
\qed


\subsection{Proof of Proposition \ref{prop:hetero}}
\label{app.pr.prop:hetero}
All the proof is made conditionally to the predictor $\fh$, which means that we prove below that 
\begin{equation}
\label{eq.pr.prop:hetero.res-cond-Dtr-explicite}
\IP\bigl( Y \in \Chat_{\ell^\star, k^\star}(X) \, \vert \, \fh \, \bigr) 
\geq 1-\alpha - \IE \biggl[ 
\dTV \Bigl( 
\PoisBin \bigl( p^* (S) \bigr)
 \, , \, 
\Bin\bigl( m , \tilde{p}^*(\widetilde{S}) \bigr)
\Bigr) 
\, \Big\vert \, \fh \, \biggr]   
\, . 
\end{equation}
The result follows by taking an expectation. 
In the remainder of the proof, for simplicity, we write $\IP(\cdot)$ and $\IE[\cdot]$ instead of $\IP(\cdot \, \vert \, \fh \, )$ and $\IE[\cdot \, \vert \, \fh \,]$, respectively. 

First, for every $k \in \llbracket 1, m \rrbracket$ and $\ell \in \llbracket 1, n \rrbracket$, by definition of $\Chat_{\ell, k}$, we have
\begin{align}
\label{eq.pr.prop:hetero.1}
\IP\bigl(Y \notin \Chat_{\ell, k}(X) \bigr)
= \IP\left(\Qh_{(\ell, k)} < S \right)
= \IP\biggl( \underbrace{ \sum_{j=1}^{m}\One{\Qh_{(\ell)}(\calS^{(j)}) < S} }_{\triangleq W} \geq k \biggr)
\, .
\end{align}
Similarly,
\begin{equation}
\label{eq.pr.prop:hetero.2}
\IP\left(\Qh_{(\ell, k)} \left( \widetilde{\calS}^{(1)} , \ldots, \widetilde{\calS}^{(m)} \right) < \widetilde{S}  \right)
= \IP\biggl( \underbrace{ \sum_{j=1}^{m}\One{\Qh_{(\ell)}(\widetilde{\calS}^{(j)}) < \widetilde{S}} }_{\triangleq \widetilde{W}} \geq k \biggr) \; .
\end{equation}
Given $S$ (and $\fh$\,), the random variables $\OneSm{\Qh_{(\ell)}(\calS^{(j)}) < S}$, $j=1, \ldots, m$, are independent Bernoulli random variables with respective parameters $p_j(S,\ell) \triangleq \IP(S^{(j)}_{(\ell)} \leq S \vert S )$, so their sum $W$ follows the $\PoisBin ( p(S,\ell) ) $ distribution, where $p(S,\ell) \triangleq (p_1(S,\ell), \ldots, p_m(S,\ell))$. 
Given $\tilde{S}$ (and $\fh$\,), the random variables $\OneSm{\Qh_{(\ell)}(\widetilde{\calS}^{(j)}) < \widetilde{S}}$, $j=1 , \ldots, m$, are i.i.d. Bernoulli random variables with common parameter $\tilde{p}(\widetilde{S},\ell) \triangleq \IP(\widetilde{S}^{(1)}_{(\ell)} \leq \widetilde{S} \vert \widetilde{S})$, so their sum $\widetilde{W}$ follows the $\Bin (m , \widetilde{p}(S,\ell)) $ distribution. 
As a consequence, we have 
\begin{align*}
\IP( W \geq k \, \vert \, S ) - \IP( \widetilde{W} \geq k \, \vert \, \widetilde{S} )
&= \PoisBin \bigl( p (S,\ell) \bigr) \bigl( [k,+\infty) \bigr)
- \Bin\bigl( m , \tilde{p}(\widetilde{S},\ell) \bigr) \bigl( [k,+\infty) \bigr)
\\
&\leq \dTV \Bigl( 
\PoisBin \bigl( p (S,\ell) \bigr) 
\, , \,
\Bin\bigl( m , \tilde{p}(\widetilde{S},\ell) \bigr) 
\Bigr) 
\, ,
\end{align*}
by definition of the total-variation (TV) distance 
$\dTV (\mu, \nu) = \sup_{A \text{ measurable}} \bigl\{ \mu(A) - \nu(A) \bigr\}$ for any probability distributions $\mu$ and $\nu$. \\
Taking an expectation and using Eq. \eqref{eq.pr.prop:hetero.1} and~\eqref{eq.pr.prop:hetero.2}, we get that 
\begin{align*}
\IP\bigl(Y \notin \Chat_{\ell, k}(X) \bigr)
&= 
\IP( \widetilde{W} \geq k ) + \IP( W \geq k ) - \IP( \widetilde{W} \geq k ) 
\\
&\leq 
\IP\left(\Qh_{(\ell, k)} \left( \widetilde{\calS}^{(1)} , \ldots, \widetilde{\calS}^{(m)} \right) < \widetilde{S}  \right) 
+ \IE \biggl[ \dTV \Bigl( 
\PoisBin \bigl( p (S,\ell) \bigr) 
\, , \, 
\Bin\bigl( m , \tilde{p}(\widetilde{S},\ell) \bigr) 
\Bigr)  \biggr]
\\
&\leq 1- M_{\ell,k} 
+ \IE \biggl[ \dTV \Bigl( 
\PoisBin \bigl( p (S,\ell) \bigr) 
\, , \, 
\Bin\bigl( m , \tilde{p}(\widetilde{S},\ell) \bigr) 
\Bigr)  \biggr] \; ,
\end{align*}
by Theorem~\ref{them:main}, which applies here since $\widetilde{S}_1^{(1)}, \ldots, \widetilde{S}_n^{(m)}, \widetilde{S}$ are i.i.d., conditionally to $\fh$.
Therefore, 
\begin{align} 
\label{eq.pr.prop:hetero.pre-concl}
\IP\bigl(Y \in \Chat_{\ell, k}(X) \bigr)
\geq M_{\ell,k} 
- \IE \biggl[ \dTV \Bigl( 
\PoisBin \bigl( p (S,\ell) \bigr) 
\, , \, 
\Bin\bigl( m , \tilde{p}(\widetilde{S},\ell) \bigr) 
\Bigr)  \biggr]
\, , 
\end{align}
which implies the result by taking $(\ell,k) = (\ell^* , k^*)$ since $ M_{\ell^*,k^*} \geq 1-\alpha$. 
\qed

\begin{remark}
\label{rk.prop:hetero.choix-Stilde}
In Proposition~\ref{prop:hetero}, let us emphasize that the auxiliary random variables $\{\widetilde{S}^{(j)}_i\}^{n, m}_{i, j=1}, \widetilde{S}$ can be dependent on the scores $\{S^{(j)}_i\}^{n, m}_{i, j=1}, S$, as long as they satisfy the only assumption required: $\{\widetilde{S}^{(j)}_i\}^{n, m}_{i, j=1}, \widetilde{S}$ must be i.i.d. given $\fh$.  
One also can choose the common distribution of the $\{\widetilde{S}^{(j)}_i\}^{n, m}_{i, j=1}, \widetilde{S}$. Here, the best choice is the one that maximizes the right-hand side of Eq.~\eqref{eq.pr.prop:hetero.res-cond-Dtr-explicite}.  
We conjecture that a good choice is to take $\widetilde{S}=S$, and to define the $\widetilde{S}^{(j)}_i$ as independent copies of~$S$ (given $\fh$). 
\end{remark}

\medbreak
Finally, let us recall a result from \citet[Theorem~1]{ehm1991binomial} which can be useful to control the right-hand side of Eq.~\eqref{eq.pr.prop:hetero.res-cond-Dtr-explicite}.
\begin{theorem}
\label{thme:TVpoibin}
Let $m \geq 1$ be an integer, $p_1, \ldots, p_m \in [0,1]$ and $\tilde{p} = \frac{1}{m}\sum^m_{j=1} p_j$. Let $\Bin(m, \tilde{p})$ denote the binomial distribution and $\PoisBin (m, (p_1, \cdots, p_m))$ denote the Poisson-binomial distribution. The following inequalities hold true: 
\begin{align*}
C \bigl[ 1-\tilde{p}^{m+1}-(1-\tilde{p})^{m+1} \bigr] \cdot \left[ 1 - \dfrac{\sum^m_{i=1} p_i (1-p_i)}{m \tilde{p}(1-\tilde{p})} 
 \right]
&\leq \dTV \Bigl( 
\PoisBin (p_1, \cdots, p_m) 
\, , \, 
\Bin (m, \tilde{p}) 
\Bigr) 
\\
&\leq \dfrac{m}{m+1} \bigl[ 1-\tilde{p}^{m+1}-(1-\tilde{p})^{m+1} \bigr] 
\cdot 
\left[ 1 - \dfrac{\sum^m_{i=1} p_i (1-p_i)}{m \tilde{p}(1-\tilde{p})} 
 \right]
\, ,
\end{align*}
where $\dTV(\cdot, \cdot)$ is the total-variation distance, and $C$ is a universal constant. 

\end{theorem}


\subsection{Proof of Theorem \ref{thme:private}}
\label{app:privacy-thme}

The privacy guarantee is a direct consequence of the fact that Algorithm \ref{alg:DP-ordered} is $\varepsilon$-DP (exponential mechanism). Indeed, each agent calls this algorithm only one time during \methodDP, making it $\varepsilon$-DP with respect to its local data set ($\varepsilon$-LDP).

It remains to prove that the desired coverage is achieved. To do so, recall the following utility lemma related to the output of Algorithm \ref{alg:DP-ordered} \citep{Angelopoulos_2022}.

\begin{lemma} \label{lemma:utility-dp}\emph{(Utility of Algorithm \ref{alg:DP-ordered}).} For any $\delta \in (0,1)$, scores $S_1,\ldots,S_n$ and $q \in [0.5, 1)$, the output of Algorithm \ref{alg:DP-ordered}, denoted $\widehat{Q}_1^\varepsilon$, satisfies:
\begin{equation}
    \IP\Bigg(\frac{|\{i:\bar{S}_i \leq \widehat{Q}_1^\varepsilon \}|}{n} \geq q - \frac{2\log{(B/\delta)}}{n\varepsilon} \Bigg| S_1,\ldots,S_n \Bigg) \geq 1- \delta
    \, .
\end{equation}
\end{lemma}

\begin{proof}
    The proof, provided by \citet[Lemma~1]{Angelopoulos_2022}, is a direct application of the utility guarantee of the general exponential mechanism \citep[see][Corollary $3.12$]{dwork2014algorithmic}. Note that \citet[Lemma~1]{Angelopoulos_2022} state the above result on average over $S_1,\ldots,S_n$, but their proof is actually valid conditionally to $S_1,\ldots,S_n$ since the original result by \citet[Corollary $3.12$]{dwork2014algorithmic} is valid conditionally to $S_1,\ldots,S_n$.
\end{proof}

We can now prove our main result. Let us first define the event $E=\{\widehat{Q}^\varepsilon \geq \widehat{Q}_{(\ell_\gamma, k_\gamma)}\}$, i.e., when the private estimator $\widehat{Q}^\varepsilon = \widehat{Q}^\varepsilon_{(k_{\gamma})}$ returned by \methodDP~is greater than the non-private estimator $\widehat{Q}_{(\ell_\gamma, k_\gamma)}$ that would be returned by \method~(Algorithm~\ref{alg:CP-QQ}) with coverage $\frac{1-\alpha}{1-\gamma \alpha}$. Denoting by $\calS^{(1\ldots m)}$ the full data set containing all local data sets of scores $\calS^{(1)},\ldots,\calS^{(m)}$, we have:
\begin{align}
    \IP\bigl(Y \in \Chat_\varepsilon(X)\bigr) 
    = \IP\bigl(S \leq \widehat{Q}^\varepsilon \bigr) 
    = \IE \Bigl[\IP\bigl(S \leq \widehat{Q}^\varepsilon \,\big\vert\, \calS^{(1\ldots m)} \bigr) \Bigr]
    & \geq \IE \Bigl[\IP\bigl(S  \leq \widehat{Q}^\varepsilon \text{ and } E \,\big\vert\, \calS^{(1\ldots m)} \bigr) \Bigr]\nonumber\\
    & \geq \IE \Bigl[\IP\bigl(S  \leq \widehat{Q}_{(\ell_\gamma, k_\gamma)} \text{ and } E \,\big\vert\, \calS^{(1\ldots m)}  \bigr) \Bigr]
\nonumber\\
	& = \IE \Bigl[ \IP\bigl( S  \leq \widehat{Q}_{(\ell_\gamma, k_\gamma)} \,\big\vert\, \calS^{(1\ldots m)} \bigr) \cdot \IP \bigl( E \,\vert\, \calS^{(1\ldots m)} \bigr) \Bigr] \; , \label{eq:priv1}
\end{align}
where the last equality is obtained by the fact that knowing $\calS^{(1\ldots m)}$, the random variable $\widehat{Q}_{(\ell_\gamma, k_\gamma)}$ is deterministic, hence the events $E = \{\widehat{Q}^\varepsilon \geq \widehat{Q}_{(\ell_\gamma, k_\gamma)}\}$ and $\{S\leq \widehat{Q}_{(\ell_\gamma, k_\gamma)}\}$ are independent.

We first show that $\IP(E\,\vert\,\calS^{(1\ldots m)}) \geq 1-\gamma \alpha$. Notice that a sufficient condition for the event $E$ to be satisfied is that each agent $j$ outputs a value $\widehat{Q}_j^\varepsilon$ greater that the $\ell_\gamma$-th ordered score $ S_{(\ell_\gamma)}^{(j)}$ of the local data set $ S^{(j)}$. Indeed, in that case the $k_\gamma$-th ordered value of $\widehat{Q}_1^\varepsilon, \ldots, \widehat{Q}_m^\varepsilon$, i.e., $\widehat{Q}^\varepsilon$, is necessarily bigger than the $k_\gamma$-th ordered value of $S_{(\ell_\gamma)}^{(1)}, \ldots, S_{(\ell_\gamma)}^{(m)}$, i.e., $\widehat{Q}_{(\ell_\gamma, k_\gamma)}$. In the end, we have $E\supset \overset{m}{\underset{j=1}{\cap}}\{\widehat{Q}_j^\varepsilon \geq  S_{(\ell_\gamma)}^{(j)}\}$, which allows us to obtain a lower bound for $\IP(E\,\vert\,\calS^{(1\ldots m)})$:
\begin{align} \notag 
    \IP(E\,\vert\,\calS^{(1\ldots m)}) 
    &\geq \IP\Bigl( \bigcap_{j=1}^m \bigl\{\widehat{Q}_j^\varepsilon \geq  S_{(\ell_\gamma)}^{(j)} \bigr\} \,\big\vert\, \calS^{(1\ldots m)}\Bigr) 
    \\
    &= \prod_{j=1}^m\IP\bigl(\widehat{Q}_j^\varepsilon \geq  S_{(\ell_\gamma)}^{(j)} \,\big\vert\, \calS^{(1\ldots m)} \bigr) 
    = \prod_{j=1}^m\IP\bigl(\widehat{Q}_j^\varepsilon \geq  S_{(\ell_\gamma)}^{(j)} \,\big\vert\, \calS^{(j)}\bigr) 
    \geq \prod_{j=1}^m \IP\bigl(\widehat{Q}_j^\varepsilon \geq \bar{ S}_{(\ell_\gamma)}^{(j)} \,\big\vert\, \calS^{(j)}\bigr)
    \; , \label{eq:PES1}
\end{align}
where the first equality comes from the fact that the events $\{\widehat{Q}_j^\varepsilon \geq  S_{(\ell_\gamma)}^{(j)}\}$ are independent given~$\calS^{(1\ldots m)}$. 
The last inequality comes from the fact that, for all $j=1,\ldots,m$, the discretized score $\bar{ S}_{(\ell_\gamma)}^{(j)}$ is larger than (or equal to) the non-discretized score $ S_{(\ell_\gamma)}^{(j)}$. 
Moreover, for every $j \in \{1, \ldots , m\}$, we have:
\begin{align*}
    \IP\Bigl(\widehat{Q}_j^\varepsilon \geq \bar{ S}_{(\ell_\gamma)}^{(j)} \,\big\vert\, \calS^{(j)}\Bigr) 
    & = \IP\Bigl( \bigl\lvert \{i : \bar{S}^{(j)}_i \leq \widehat{Q}_j^\varepsilon\} \bigr\rvert \geq \ell_\gamma \,\big\vert\, \calS^{(j)} \Bigr) \\
    & = \IP\Bigl( \bigl\lvert \{i : \bar{S}^{(j)}_i \leq \widehat{Q}_j^\varepsilon\}\bigr\rvert \geq \ell_\gamma + \ell_{\text{ cor}} - \ell_{\text{ cor}} \,\big\vert\, \calS^{(j)} \Bigr) \\
    & \geq \IP\left(\frac{ \bigl\lvert \{i : \bar{S}^{(j)}_i \leq \widehat{Q}_j^\varepsilon\}\bigr\rvert}{n} \geq \frac{\ell_\gamma + \ell_{\text{ cor}}}{n} -  \frac{2}{n\varepsilon}\log \left(\frac{B}{1-(1-\gamma \alpha)^\frac{1}{m}}\right) \,\bigg\vert\, \calS^{(j)} \right) 
    \\
    & \geq \IP\left(\frac{ \bigl\lvert \{i : \bar{S}^{(j)}_i \leq \widehat{Q}_j^\varepsilon\} \bigr\rvert }{n} \geq \max\left\{ \frac{\ell_\gamma + \ell_{\text{ cor}}}{n} , \frac{1}{2} \right\} -  \frac{2}{n\varepsilon}\log \left( \frac{B}{1-(1-\gamma \alpha)^\frac{1}{m}}\right) \,\bigg\vert\, \calS^{(j)} \right) 
    \\
    & \geq (1-\gamma \alpha)^\frac{1}{m}\;,
\end{align*}
where the last inequality is obtained by applying Lemma~\ref{lemma:utility-dp} with $\{S_1, \ldots , S_n\} = \calS^{(j)}$, $q=\max\{ \frac{\ell_\gamma + \ell_{\text{ cor}}}{n} , \frac{1}{2} \}$ and $\delta = 1-(1-\gamma \alpha)^\frac{1}{m}$.

Plugging this result into Eq.~\eqref{eq:PES1}, we get $\IP(E \,\vert\, \calS^{(1\ldots m)}) \geq 1-\gamma \alpha$, which can then be plugged into Eq.~\eqref{eq:priv1} and leads to 
\begin{align*}
	\IP\bigl( Y \in \Chat_\varepsilon(X) \bigr) 
    \geq \IE \Bigl[ \IP\bigr( S \leq \widehat{Q}_{(\ell_\gamma, k_\gamma)} \,\big\vert\, \calS^{(1\ldots m)} \bigr) \cdot \IP \bigl(E \,\vert\, \calS^{(1\ldots m)} \bigr) \Bigr] 
	&\geq \IE \Bigl[ \IP\bigl(S \leq \widehat{Q}_{(\ell_\gamma, k_\gamma)} \,\big\vert\, \calS^{(1\ldots m)} \bigr) \Bigr] \cdot (1-\gamma\alpha)  \\
	& = \IP\bigl( S \leq \widehat{Q}_{(\ell_\gamma, k_\gamma)} \bigr) \cdot (1-\gamma\alpha)  \\
	& \geq 1-\alpha \, ,
\end{align*}
where the last inequality comes from the fact that  $\widehat{Q}_{(\ell_\gamma, k_\gamma)}$ is the output of \method{} (Algorithm~\ref{alg:CP-QQ}) with coverage $\frac{1-\alpha}{1-\gamma \alpha}$.
\qed


\subsection{Proof of Proposition \ref{prop:maxQQ}}
\label{app.pr.prop:maxQQ}
We start by proving the following lemma. 
\begin{lemma} 
\label{le.pr.prop:maxQQ}
The following equality holds true for every integer $k \geq 1$:
	\begin{equation*}
	\sum^{k-1}_{j=0} \dfrac{\Gamma(j + 1/n)}{\Gamma(j + 1)} = \dfrac{n \cdot k \cdot \Gamma(k + 1/n)}{\Gamma(k + 1)} \; .
	\end{equation*}
\end{lemma}
\begin{proof}
	Throughout the proof, we use that for any $x>0$,
	$\Gamma(x) = \dfrac{\Gamma(x + 1)}{x}\;,$ according to \citet{davis1959leonhard}. \\
	
	\noindent
	We proceed by induction on $k$.
	First, for $k = 1$, 
	\begin{align*}
	\sum^{k-1}_{j=0} \dfrac{\Gamma(j + 1/n)}{\Gamma(j + 1)} &= \dfrac{\Gamma(1/n)}{\Gamma(1)} = \Gamma(1/n) 
 = n \cdot \Gamma(1/n + 1) 
 = \frac{n \Gamma(1/n + 1)}{\Gamma(2)} \; .
	\end{align*}
	Then, assume that the result holds true for some $k \geq 1$, that is, 
	\begin{align*}
	\sum^{k-1}_{j=0} \dfrac{\Gamma(j + 1/n)}{\Gamma(j + 1)} &= \dfrac{n \cdot k \cdot \Gamma(k + 1/n)}{\Gamma(k + 1)} \; ,
	\end{align*}
 and let us prove that it holds true for $k+1$: 
\begin{align*}
	\sum^{k}_{j=0} \dfrac{\Gamma(j + 1/n)}{\Gamma(j + 1)} &= \sum^{k-1}_{j=0} \dfrac{\Gamma(j + 1/n)}{\Gamma(j + 1)} + \dfrac{\Gamma(k + 1/n)}{\Gamma(k + 1)} \\
	&= \dfrac{n \cdot k \cdot \Gamma(k + 1/n)}{\Gamma(k + 1)} + \dfrac{\Gamma(k + 1/n)}{\Gamma(k + 1)} \\
	&= \dfrac{(n \cdot k + 1)\Gamma(k + 1/n)}{\Gamma(k + 1)} \\
	&= \dfrac{(k + 1)}{\Gamma(k + 2)} \cdot (n \cdot k + 1)\Gamma(k + 1/n) \\
	&= \dfrac{(k + 1)}{\Gamma(k + 2)} \cdot (n \cdot k + 1)\dfrac{\Gamma(k + 1/n + 1)}{k + 1/n} \\
	&= \dfrac{(k + 1)}{\Gamma(k + 2)} \cdot (n \cdot k + 1)\dfrac{\Gamma(k + 1/n + 1)}{(n\cdot k + 1)/n} \\
	&= \dfrac{n \cdot (k + 1) \cdot \Gamma(k + 1 + 1/n)}{\Gamma(k + 2)} \; .
	\end{align*}
\end{proof}


We can now prove Proposition~\ref{prop:maxQQ}. 
Let us assume that $\{ U^{(j)}_i \}_{i,j=1}^{m,n}, U$ are i.i.d. uniform on $[0, 1]$ and use the notation of the proof of Theorem \ref{them:main}. We have $M_{n,k} = \IP ( U \leq U_{(n, k)} )$ by Theorem~\ref{them:main} and by Lemma~\ref{lem.cdf-densite.stat-ordre}, for every $t \in [0,1]$:
	\begin{align*}
	F_{U_{(n, k)}}(t)
	&=\sum^m_{j=k} \binom{m}{j} F_{U_{(n)}}(t)^j \bigl[ 1 - F_{U_{(n)}}(t) \bigr]^{m-j}
	= \sum^m_{j=k}\binom{m}{j} \left[t^n\right]^j \left[1 - t^n\right]^{m-j} \; .
	\end{align*}
Therefore, 	
\begin{align*}
	1 - M_{n, k} &= \IP\left(U_{(n, k)} \leq U \right) 
	= \int^{1}_{0} F_{U_{(n, k)}}(t) \mathrm{d}t  \\
	&= \int^{1}_{0} \sum^m_{j=k}\binom{m}{j} ( t^n )^j ( 1 - t^n )^{m-j} \mathrm{d}t  \\
	&= \dfrac{1}{n} \sum^m_{j=k}\binom{m}{j} \int^{1}_{0} u^j ( 1 - u )^{m-j} u^{1/n - 1} \mathrm{d}u  \qquad\qquad \text{( change of variable $u = t^n$)}  \\
	&= \dfrac{1}{n} \sum^m_{j=k}\binom{m}{j} \int^{1}_{0} u^{j + 1/n - 1} ( 1 - u )^{m-j} \mathrm{d}u  \\
	&= \dfrac{1}{n} \sum^m_{j=k}\binom{m}{j} \mathrm{B}(j + 1/n, m-j+1) \; ,
 \end{align*} 
where 
\[ 
\mathrm{B} : (a,b) \in (0,+\infty)^2 \mapsto \int_0^1 u^{a-1} (1-u)^{b-1} \mathrm{d}u
= \frac{\Gamma(a) \Gamma(b)}{\Gamma(a+b)}
\]
denotes the Beta function. We obtain that
\begin{align*}
    1 - M_{n, k} 
    &= \dfrac{1}{n} \sum^m_{j=k} \dfrac{\Gamma(j + 1/n)}{\Gamma(j + 1)} \cdot \dfrac{\Gamma(m + 1)}{\Gamma(m + 1/n + 1)} \\
    &= \dfrac{1}{n} \cdot \dfrac{\Gamma(m + 1)}{\Gamma(m + 1/n + 1)} \sum^m_{j=k} \dfrac{\Gamma(j + 1/n)}{\Gamma(j + 1)} \\
    &= \dfrac{1}{n} \cdot \dfrac{\Gamma(m + 1)}{\Gamma(m + 1/n + 1)} \left(\sum^m_{j=0} \dfrac{\Gamma(j + 1/n)}{\Gamma(j + 1)} - \sum^{k-1}_{j=0} \dfrac{\Gamma(j + 1/n)}{\Gamma(j + 1)} \right) \; . 
\end{align*}
Using Lemma~\ref{le.pr.prop:maxQQ}, we get that 
\begin{align*}
    1 - M_{n, k} 
    &= \dfrac{1}{n} \cdot \dfrac{\Gamma(m + 1)}{\Gamma(m + 1/n + 1)} \left(\dfrac{n (m + 1)\Gamma(m + 1/n + 1)}{\Gamma(m + 2)} - \dfrac{n \cdot k \cdot \Gamma(k + 1/n)}{\Gamma(k + 1)}\right) \qquad \\
    &= \Gamma(m + 1) \left(\dfrac{m + 1}{\Gamma(m + 2)} - \dfrac{k \cdot \Gamma(k + 1/n)}{\Gamma(k + 1) \Gamma(m + 1/n + 1)}\right) \\
    &= 1 - \dfrac{ \Gamma(k + 1/n)}{\Gamma(k)} \cdot \dfrac{\Gamma(m+1)}{\Gamma(m+1/n+1)} \; ,
\end{align*}
 which proves the first formula. 
	
Now, using Stirling's formula, when $k,m \to +\infty$, we have
\begin{align*}
    \dfrac{ \Gamma(k + 1/n)}{\Gamma(k)} \cdot \dfrac{\Gamma(m+1)}{\Gamma(m+1/n+1)} \sim \dfrac{ \Gamma(k)k^{1/n}}{\Gamma(k)} \cdot \dfrac{\Gamma(m) m }{\Gamma(m) m^{1/n+1}} = \dfrac{k^{1/n} }{m^{1/n}} \; .
\end{align*}
By setting $k = k_m \geq m (1-\alpha)^n$, we obtain the second result.
\qed

%% file: fast_comput.tex

Let us recall that the right-hand side of Equation \eqref{eq:main_equa} is 
\begin{align*}
& M_{\ell,k} = 1 - \displaystyle \dfrac{1}{m n + 1} \sum^{m}_{j=k} \binom{m}{j} \sum^{n}_{i_1=\ell} \cdots \sum^{n}_{i_j=\ell} \sum^{\ell-1}_{i_{j+1}=0} \cdots \sum^{\ell-1}_{i_m=0} \dfrac{\binom{n}{i_1} \cdots \binom{n}{i_m}}{\binom{m n}{i_1 + \cdots + i_m}} 
= 1 - \displaystyle \dfrac{1}{m n + 1} \sum^{m}_{j=k} \binom{m}{j} \sum^{n}_{I_{1,j}=\ell} \sum^{\ell-1}_{I^c_{1,j}=0}  \dfrac{\binom{n}{i_1} \cdots \binom{n}{i_m}}{\binom{m n}{i_1 + \cdots + i_m}} \; ,
\end{align*}
where $I_{1,j} = \{i_1, \ldots, i_j\}$ and $I^c_{1,j} =\{i_{j+1},\ldots,i_m\}$. The time complexity of its brute-force computation is too high. In this section, we therefore provide an efficient algorithm to compute it. In the first step, we rewrite the summations to bring out the mass function of a multivariate hypergeometric distribution:
\begin{align}
    \label{eq:sum_fast}
    & \sum^{n}_{I_{1,j}=\ell} \sum^{\ell-1}_{I^c_{1,j}=0} \dfrac{\binom{n}{i_1} \cdots \binom{n}{i_m}}{\binom{m n}{i_1 + \cdots + i_m}} = \sum_{r \in \tilde{R}_j} \sum^{n}_{I_{1,j}=\ell} \sum^{\ell-1}_{I^c_{1,j}=0} \underbrace{\dfrac{\binom{n}{i_1} \cdots \binom{n}{i_m}}{\binom{m n}{i_1 + \cdots + i_m}} \Ind{i_1 + \cdots + i_m = r}}_{\text{mass of a multivariate hypergeometric distribution}} \; \; ,
\end{align}
with $\tilde{R}_j = \{j\ell, \ldots ,jn + (m-j)(\ell-1)\}$. The summation in $I_{1,j}$, and $I^c_{1,j}$ therefore computes rectangular probabilities and can be rewritten as follows
\begin{align*}
    p_r(\mathbf{a}, \mathbf{b}) 
    &\triangleq \IP(a_1 \leq H_1 \leq b_1, \cdots, a_m \leq H_m \leq b_m) \; ,
    \\
    \text{where } 
    (a_i,b_i) &= \begin{cases} 
    (\ell,n) \qquad &\text{if } i \in \{ 1 , \ldots , j \} \\
    (0,\ell-1) \qquad &\text{if } i \in \{ j+1, \ldots , m\}  \; ,
    \end{cases}
\end{align*}
and $(H_1, \ldots, H_m)$ follows a multivariate hypergeometric distribution with parameters $(\{n, \ldots, n\}, r)$. 
By a direct application of Bayes’ theorem we obtain \citep[Equations~(2) and (5)]{lebrun2013efficient}:
\begin{align*}
    p_r(\mathbf{a}, \mathbf{b}) &= \IP \left( \sum^m_{i=1} T_i = r \right) \dfrac{\prod^m_{i=1} \IP(a_i \leq W_i \leq b_i)}{\IP(\sum^m_{i=1} W_i = r)} \; ,
\end{align*}
where for any $t \in (0, 1)$ and for all $1 \leq i \leq m$, the random variables $W_i$ follow a binomial distribution $\mathcal{B}(n, t)$ and $T_i = (W_i \mid a_i \leq W_i \leq b_i)$ follows a truncated binomial distribution.

As there exists efficient algorithms to compute both $\IP(a_i \leq W_i \leq b_i)$ and $\IP(\sum^m_{i=1} W_i = r)$, the only difficulty remains the evaluation of $\IP(\sum^m_{i=1} T_i = r)$. One approach is to multiply the generating probability functions of the $T_i$ and then extract the coefficient of degree $r$. This algorithm has a time complexity of $\mathcal{O}(m r \log(r))$ if the multiplications are done using an FFT based algorithm. This strategy still remains costly for large values of $m$ or $r$, and more advanced algorithms have been proposed by \citet{lebrun2013efficient}.

Note finally that since $M_{\ell,k}$ is a non-decreasing function of both $\ell$ and $k$, one can find $(\ell^*,k^*)$ without computing $M_{\ell,k}$ for all values of~$(\ell,k)$. Figures \ref{fig:heat_map} and~\ref{fig:rev2-fig3} illustrate that $M_{\ell,k}$ actually needs to be computed for only a few values of~$(\ell,k)$.

\begin{figure}
	\centering
	\includegraphics[width=.495\linewidth]{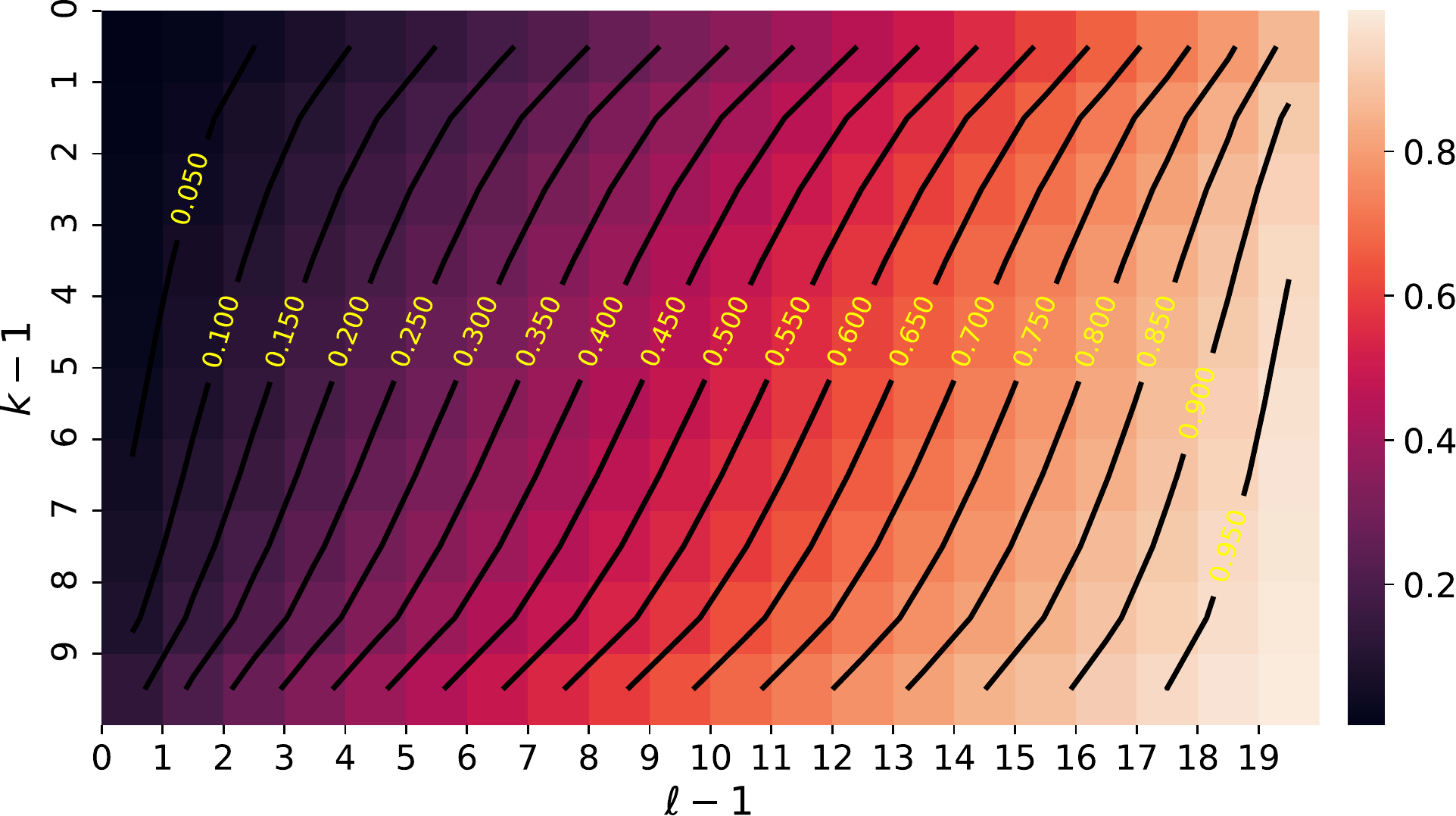}
	\vspace{-1em}
	\caption{Heat-map representation of $(M_{\ell,k})_{1 \leq \ell \leq n , 1 \leq k \leq m}$ for $m=10$ and $n=20$. }
	\label{fig:heat_map}
\end{figure}

\begin{figure}
	\centering
	\includegraphics[width=.495\linewidth]{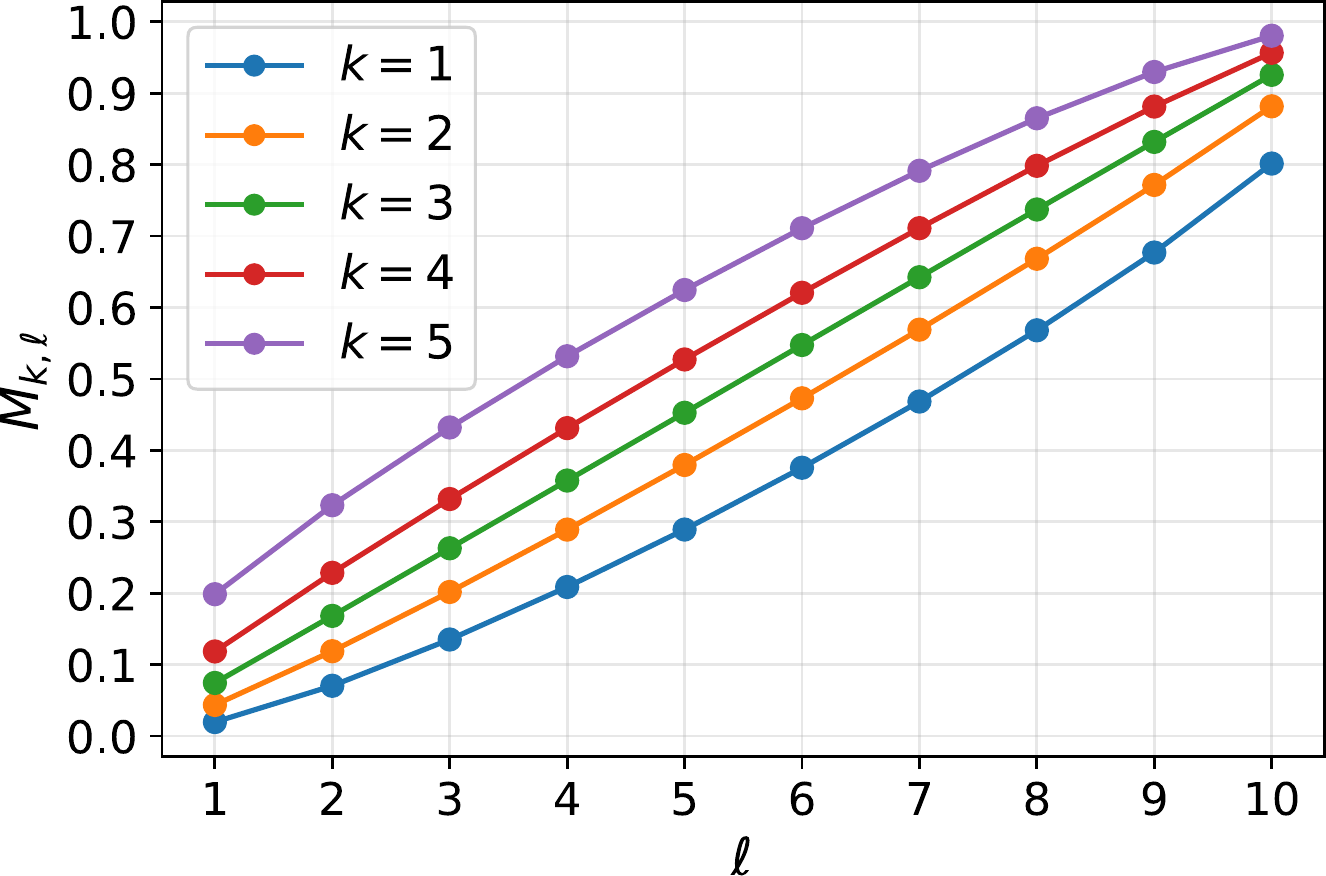}
 \includegraphics[width=.495\linewidth]{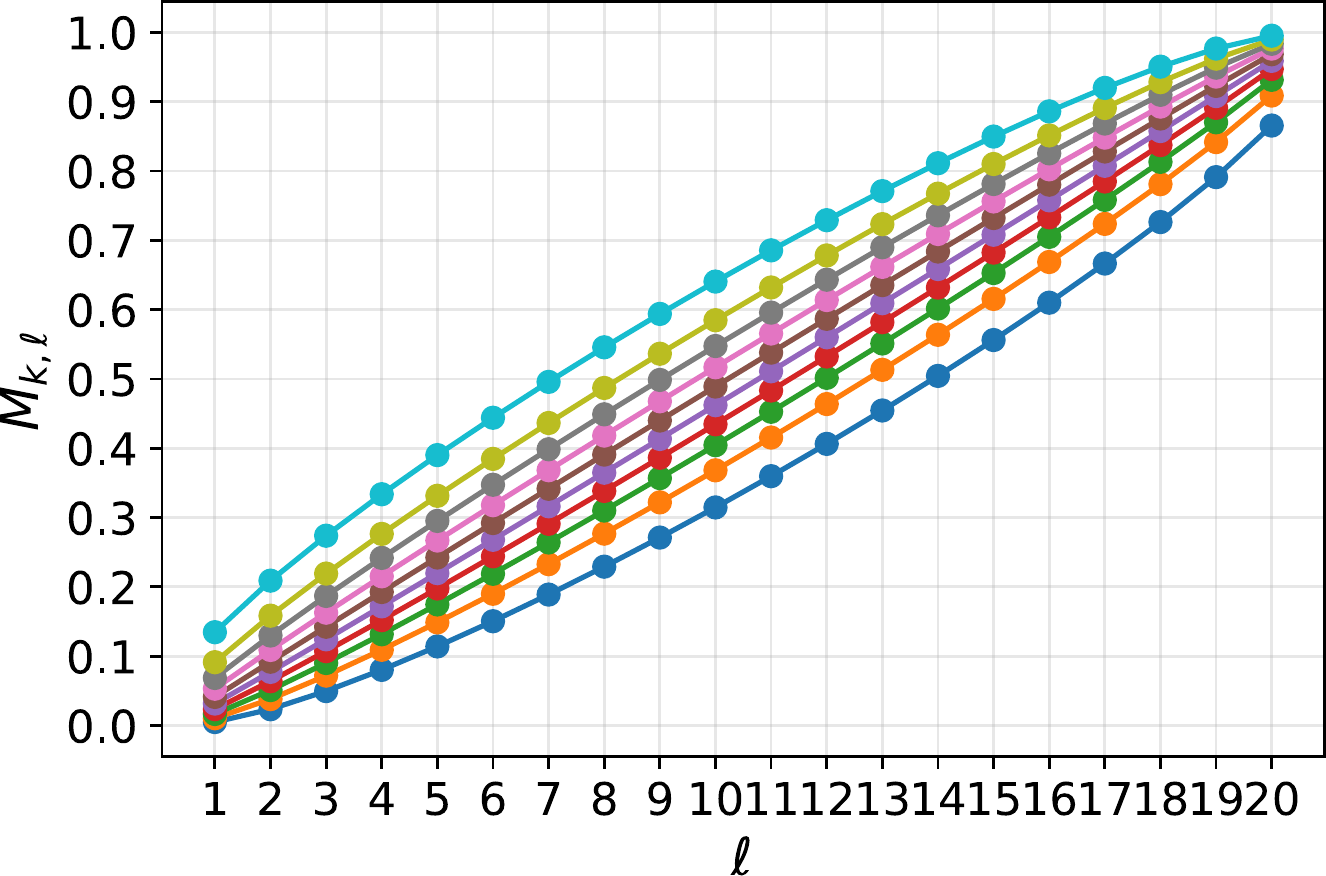}
	\vspace{-1em}
	\caption{Values of $(M_{\ell,k})_{1 \leq \ell \leq n , 1 \leq k \leq m}$ when $(m, n) = (5, 10)$ (Left panel) and $(10, 20)$ (Right panel). One color per value of $k$.}
	\label{fig:rev2-fig3}
\end{figure}